%% file: main.tex
\documentclass[11pt]{article}
\usepackage{enumerate}
\usepackage[OT1]{fontenc}
\usepackage[usenames]{color}
\usepackage{smile}
\usepackage{booktabs}
\usepackage[colorlinks,
linkcolor=red,
anchorcolor=blue,
citecolor=blue
]{hyperref}
\usepackage{mathrsfs}
\usepackage{fullpage}
\usepackage{hyperref}
\usepackage[protrusion=true, expansion=true]{microtype}
\usepackage{float}
\usepackage{subfigure}
\usepackage{amsfonts,amsmath,amssymb,amsthm,url,xspace,mathtools,authblk}
\usepackage{tikz}
\usepackage{verbatim}
\usetikzlibrary{arrows,shapes}
\usepackage[bottom]{footmisc}
\usepackage{enumitem}
\usepackage{lscape,soul,diagbox}
\usepackage{caption}
\usepackage{algorithmicx,algpseudocode}
\usepackage{smile,math}
\usepackage{extarrows}

\ifx\counterwithout\undefined\usepackage{chngcntr}\fi
\counterwithout{equation}{section}

\allowdisplaybreaks[1]
\mathtoolsset{showonlyrefs=true}

\usepackage{xargs}
\usepackage[colorinlistoftodos,prependcaption,textsize=tiny]{todonotes}
\newcommandx{\unsure}[2][1=]{\todo[linecolor=red,backgroundcolor=red!25,bordercolor=red,#1]{#2}}
\newcommandx{\change}[2][1=]{\todo[linecolor=blue,backgroundcolor=blue!25,bordercolor=blue,#1]{#2}}
\newcommandx{\info}[2][1=]{\todo[linecolor=OliveGreen,backgroundcolor=OliveGreen!25,bordercolor=OliveGreen,#1]{#2}}
\newcommandx{\improvement}[2][1=]{\todo[linecolor=Plum,backgroundcolor=Plum!25,bordercolor=Plum,#1]{#2}}

\usepackage{alphalph}
\begin{document}
\date{}

\title{Online Statistical Inference of Constrained Stochastic Optimization via Random Scaling}

\author[1]{Xinchen Du}
\author[2]{Wanrong Zhu}
\author[1]{Wei Biao Wu}
\author[3]{Sen Na}

\affil[1]{Department of Statistics, The University of Chicago}
\affil[2]{Department of Statistics, University of California, Irvine}
\affil[3]{School of Industrial and Systems Engineering, Georgia Institute of Technology}

\maketitle

\begin{abstract}

Constrained stochastic nonlinear optimization problems have attracted significant attention for their ability to model complex real-world scenarios in physics, economics, and biology. As datasets continue to grow, online inference methods have become crucial for enabling real-time decision-making without the need to store historical data. 
In this work, we develop an online inference procedure for constrained stochastic optimization by leveraging a method called Adaptive Inexact Stochastic Sequential Quadratic Programming (AI-SSQP). As a generalization of (sketched) Newton methods to constrained problems, AI-SSQP approximates the objective with a quadratic model and the constraints with a linear model at each step, then applies a randomized sketching solver to inexactly solve the resulting subproblem, along with an adaptive random stepsize to update the primal-dual iterates.
Building on this design, we first establish the asymptotic normality guarantee of \textit{averaged} AI-SSQP and observe that the averaged iterates exhibit better statistical efficiency than the last iterates, in terms of a smaller limiting covariance matrix. Furthermore, instead of estimating the limiting covariance matrix directly,~we~study~a~new~\mbox{online}~\mbox{inference}~\mbox{procedure}~called~\textit{\mbox{random}~\mbox{scaling}}. 
Specifically, we construct a test statistic by appropriately rescaling the averaged iterates, such~that~the limiting distribution of the test statistic is free of any unknown parameters.~Compared to \mbox{existing}~\mbox{online} inference procedures, our approach offers two key advantages: (i) it enables the construction~of~\textit{asymptotically valid} and \textit{statistically efficient} confidence intervals, while existing procedures based on the last iterates are less efficient and rely on a plug-in covariance estimator that is \mbox{inconsistent};~and~(ii) it is \textit{matrix-free}, i.e., the computation involves only primal-dual iterates themselves without any~matrix inversions, making its computational cost match that of advanced first-order methods for~unconstrained problems.~We validate our theoretical findings through numerical experiments~on~nonlinearly constrained regression problems and demonstrate the superior performance of random scaling over existing inference procedures.

\end{abstract}

\input{sec1}

\input{sec2}

\input{sec3}

\input{sec4}

\input{sec5}

\input{sec6}


\bibliographystyle{my-plainnat}
\bibliography{ref}

\appendix
\numberwithin{equation}{section}
\numberwithin{theorem}{section}

\input{appendix}

\end{document}

%% file: sec1.tex
\section{Introduction}
\label{sec:1}

We consider the following constrained stochastic optimization problem:
\begin{equation}\label{pro:1}
\min_{\bx\in \mR^d}\;\;  f(\bx) =  \mE_{\xi \sim \P}[F(\bx; \xi)], \quad  \text{s.t.} \;\;  c(\bx) = \0,
\end{equation}
where $f:\mR^d\rightarrow\mR$ is the  objective function, $F(\bx, \xi)$ is a noisy measurement of $f(\bx)$ with randomness $\xi$ following distribution $\P$, and $c:\mR^d\rightarrow \mR^m$ denotes the equality constraints. 
Problem \eqref{pro:1} arises in various applications across science, economics, and engineering, such as portfolio allocation \citep{Fan2012Vast, Du2022High}, computer vision \citep{Roy2018Geometry},~\mbox{PDE-constrained}~\mbox{optimization} \citep{Rees2010Optimal, Kouri2014Inexact}, constrained deep neural networks \citep{Chen2018Constraint},~\mbox{physics-inspired}~\mbox{neural} networks \citep{Cuomo2022Scientific}, natural language processing~\citep{Nandwani2019primal}, and constrained maximum likelihood estimation  \citep{Geyer1991Constrained, Chatterjee2016Constrained}.~In~practice,~Problem \eqref{pro:1} can be interpreted as a constrained model parameter estimation problem, where  $\xi$ denotes the data sample, $f(\bx)$ denotes the expected loss, and the (primal) solution $\bx^{\star} = \argmin_{c(\bx) = \boldsymbol{0}} f(\bx)$ denotes the true model parameter. Let $\mL(\bx, \blambda) = f(\bx) + \blambda^{\top}c(\bx)$ be the Lagrangian function, where $\blambda\in\mR^m$ is the dual vector associated with equality constraints; and let $\blambda^{\star}$ denote the dual solution. In this paper, our goal is to conduct online statistical inference for the primal-dual solution $(\bx^{\star}, \blambda^{\star})$.

Using the classic offline method via sample average approximation (SAA), one can find~the~\mbox{solution} to the following  problem:
\begin{equation}\label{pro:2}
\min_{\bx\in \mR^d}\;\; \hat{f}_t(\bx) = \frac{1}{t}\sum_{i=1}^{t}F(\bx;\xi_i),\quad \text{s.t.} \;\;  c(\bx) = \0,
\end{equation} 
where $\xi_1, \xi_2, \cdots, \xi_t$ are observed i.i.d data from $\mathcal{P}$. It is well known that the minimizer $(\hat{\bx}_t, \hat{\blambda}_t)$ of~\eqref{pro:2}, also called \textit{constrained $M$-estimator}, exhibits $\sqrt{t}$-consistency and asymptotic normality, allowing us to perform statistical inference. In particular, under proper regularity conditions,~we~have~that~as~$t\rightarrow\infty$ \citep[Chapter 5]{Shapiro2014Lectures},
\begin{equation}\label{nnequ:1}
\sqrt{t} \cdot ( \hat{\bx}_t - \tx, \hat{\blambda}_t - \tlambda )\stackrel{d}{\longrightarrow}\mN\rbr{\0,\; \Xi^{\star}},
\end{equation}
where the limiting covariance is given by $\Xi^\star = (\nabla^2\mL(\tx,\tlambda))^{-1}\text{cov}(\nabla\mL(\tx,\tlambda;\xi))(\nabla^2\mL(\tx,\tlambda))^{-1}$ (see Section \ref{sec:3.2} for details). 
However, offline methods often have to deal with large batches of samples, leading to significant computational and memory cost. In recent years, a growing body of literature has focused on developing and analyzing \textit{online} methods for solving the above constrained stochastic optimization problems and performing the underlying inference~tasks.~One~\mbox{simple}~\mbox{online}~\mbox{method} is projection-based Stochastic Gradient Descent (SGD). \cite{Davis2019Stochastic} proved almost-sure convergence for projection-based SGD, while \cite{Duchi2021Asymptotic} and \cite{Davis2024Asymptotic} established asymptotic normality guarantees for its iterates.
However, the projection operator, which plays a key role in these methods, can be computationally expensive in practice, especially for nonlinear equality constraints in \eqref{pro:1}.
To circumvent this difficulty, researchers have developed alternative constrained optimization methods, including penalty methods, barrier methods,~augmented Lagrangian methods, and sequential quadratic programming methods \citep{Bertsekas1997Nonlinear, Bertsekas2014Constrained, Nocedal2006Numerical}.\;\;

More recently, various stochastic sequential quadratic programming (SSQP) methods, which generalize stochastic Newton methods for unconstrained problems, have been proposed to efficiently solve constrained problems in an online fashion; see \cite{Berahas2021Sequential,Berahas2024Stochastic, Na2022adaptive, Na2023Inequality, Fang2024Fully, Curtis2024Stochastic} and references therein.  
For the purpose of statistical~inference, \cite{Na2025Statistical} considered an Adaptive Inexact SSQP (AI-SSQP) scheme, where at~each step the method employs an iterative sketching solver to inexactly solve the SQP quadratic subproblem and selects a suitable random stepsize (inspired by stochastic line search) to update the primal-dual iterate $(\bx_t, \blambda_t)$. The authors demonstrated that AI-SSQP iterates exhibit the following asymptotic normality property (see Section \ref{sec:3.2} for details) \citep[Theorem 5.6]{Na2025Statistical}:
\begin{equation}\label{equ:CLT_last_iterate}
\sqrt{1/\bar{\alpha}_t} \cdot (\boldsymbol{x}_t - \boldsymbol{x}^{\star}, \boldsymbol{\lambda}_t - \boldsymbol{\lambda}^{\star}) \stackrel{d}{\longrightarrow} \mathcal{N}(0, \tilde{\Xi}^{\star}),
\end{equation}
where $\baralpha_t$ is the random stepsize and the limiting covariance $\Tilde{\Xi}^{\star}$ (given by \eqref{equ:Lyapunov}) is determined by the underlying sketching distribution.~AI-SSQP methods offer advantages over \mbox{projection-based} methods~by linearizing nonlinear constraints, thereby eliminating the need for costly projection operators. Additionally, the sketching solver introduces a trade-off between computational and statistical efficiency. In particular, when the sketching solver is suppressed and the quadratic subproblem is solved exactly, the AI-SSQP iterates achieve primal-dual \textit{asymptotic minimax efficiency}, with the limiting covariance $\Tilde{\Xi}^{\star}$ reducing to the optimal covariance $\Xi^\star$ achieved by the constrained $M$-estimator $(\hat{\bx}_t, \hat{\blambda}_t)$~in~\eqref{nnequ:1}.~The primal component of this covariance also matches~that~\mbox{obtained}~by~\mbox{projection-based} methods~in~\cite{Duchi2021Asymptotic, Davis2024Asymptotic}.

To facilitate online inference via asymptotic normality, a suitable covariance matrix estimator is typically required. \cite{Jiang2025Online} analyzed a \textit{batch-means} covariance estimator originally developed for SGD methods \citep{Chen2020Statistical, Zhu2021Online}, and demonstrated that the same estimator applies to projected SGD methods with the same convergence rate.~\mbox{However},~due~to~the~nonlinearity and nonconvexity of the problem, the convergence analysis of the batch-means \mbox{estimator}~was~conducted only under light-tailed gradient noise, a stronger condition than that required for asymptotic normality of projected SGD. \cite{Kuang2025Online} designed a \textit{batch-free} covariance \mbox{estimator}~for~stochastic Newton methods and established an improved convergence rate under strong convexity. Their analysis also does not apply to nonlinear problem \eqref{pro:1}.~In fact, \cite{Na2025Statistical} introduced a \textit{plug-in} covariance estimator for $\tilde{\Xi}^\star$ in \eqref{equ:CLT_last_iterate} within the AI-SSQP framework. However, due to the difficulty of estimating the sketching components of $\tilde{\Xi}^\star$, their estimator neglects all such components and instead estimates $\Xi^\star$. This leads to two main limitations.
First, the resulting covariance \mbox{estimator}~is~generally \mbox{inconsistent} due to the approximation error incurred~by~the~\mbox{sketching}~solver~\citep[Theorem 5.10]{Na2025Statistical}. It is consistent only when the sketching solver is suppressed; and the~bias~\mbox{significantly} undermines the validity of statistical inference. Second, the method requires inverting the estimated Hessian matrix, which incurs a computational cost of $O((d+m)^3)$. This~cost~often exceeds that~of~solving the Newton subproblems themselves, contradicting the spirit of utilizing \mbox{sketching}~solvers.$\quad\quad\quad$

To overcome the above challenges, we leverage a technique called \textit{random scaling} to bypass the need for direct covariance matrix estimation. Specifically, instead of focusing on the asymptotic~normality of the \textit{last} iterate of AI-SSQP, we consider the averaged iterate $(\bar{\bx}_t, \bar{\blambda}_t) \coloneqq \frac{1}{t} \sum_{i = 0}^{t-1} (\bx_i, \blambda_i)$~and establish its asymptotic normality:
\begin{equation*}
\sqrt{t} \cdot (\bar{\bx}_t - \bx^{\star}, \bar{\blambda}_t - \blambda^{\star}) \stackrel{d}{\longrightarrow} \mathcal{N}(0, \barXi^{\star}),
\end{equation*}
where the limiting covariance matrix $\barXi^{\star}$ differs from $\tilde{\Xi}^\star$ of the last iterate in \eqref{equ:CLT_last_iterate}. Similar to $\tilde{\Xi}^\star$,~when the sketching solver is degraded, $\barXi^{\star}$ reduces to the minimax optimal covariance $\Xi^\star$. However,~when~the sketching solver is deployed, we show that the averaged iterate enjoys better statistical efficiency~with $\barXi^\star \preceq \tilde{\Xi}^\star$. 
Furthermore, to perform statistical inference, instead of~estimating~$\barXi^{\star}$, we exploit a~\textit{random scaling} matrix that can be updated on-the-fly:
\begin{equation*}
V_t = \frac{1}{t^2}\sum_{i = 1}^t i^2 \begin{pmatrix}
\bar{\bx}_i - \bar{\bx}_t \\
\bar{\blambda}_i - \bar{\blambda}_t
\end{pmatrix}\begin{pmatrix}
\bar{\bx}_i - \bar{\bx}_t \\
\bar{\blambda}_i - \bar{\blambda}_t
\end{pmatrix}^{\top},
\end{equation*}
and studentize $(\barx_t, \barlambda_t)$ using $V_t$. In particular, we prove that the resulting test statistic~is~\textit{asymptotically pivotal}: for any $\boldsymbol{w} \in \mR^{d+m}$,
\begin{equation*}
\frac{\sqrt{t} \cdot \boldsymbol{w}^{\top} (\bar{\bx}_t - \bx^{\star}, \bar{\blambda}_t - \blambda^{\star})}{\sqrt{\boldsymbol{w}^{\top} V_t \boldsymbol{w}}} \stackrel{d}{\longrightarrow} \frac{W_1(1)}{\sqrt{\int_{0}^1  \left(W_1(r) - rW_1(1) \right)^2 dr}}
\end{equation*}
where $W_1(\cdot)$ denotes the standard one-dimensional Brownian motion.~Our pivotal test statistic~offers four key advantages. 

\noindent \textbf{(i)} Since both the averaged iterate $(\bar{\bx}_t, \bar{\blambda}_t)$ and the random scaling matrix $V_t$ can be updated~recursively, our inference method is well-suited for online computation. 

\vskip3pt
\noindent \textbf{(ii)} Our method enables the construction of asymptotically valid confidence intervals for the true parameters $(\bx^{\star}, \blambda^{\star})$, resolving the inconsistency of the plug-in estimator in \cite{Na2025Statistical}. 

\vskip3pt
\noindent \textbf{(iii)} Our test statistic based on second-order methods is matrix-free (i.e., no matrix inversion).~Therefore, the memory and computational complexities of our inference procedure leveraging second-order methods match those of first-order methods, i.e., $O((d+m)^2)$. 

\vskip3pt
\noindent \textbf{(iv)} Our inference method is applicable to a broad class of constrained nonlinear problems. In~contrast, existing inference literature based on SGD or stochastic Newton methods focuses on unconstrained strongly convex problems \citep{Chen2020Statistical, Zhu2021Online, Wei2023Weighted, Chen2024Online, Kuang2025Online} (a detailed literature review is deferred later), while \mbox{projection-based}~methods such as \cite{Davis2024Asymptotic, Jiang2025Online} rely on sub-Gaussian gradient noise~assumptions.$\quad\;$

We demonstrate the promising empirical performance of our random scaling method through nonlinearly constrained regression problems.

\vspace{0.2cm}

\noindent \textbf{Related work on unconstrained online inference.}
We present a brief overview of the unconstrained online statistical inference problem, which serves as the foundation for our study~of~the~constrained inference problem.

Early works established almost-sure convergence of SGD under various assumptions \citep{Robbins1951stochastic, Kiefer1952Stochastic, Robbins1971convergence}. Later, \cite{Polyak1992Acceleration} proved the asymptotic normality of averaged SGD iterates, with the limiting~covariance matrix given by $\Sigma^{\star} = (\nabla^2 f(\bx^\star))^{-1}\text{cov}(\nabla F(\bx^{\star}; \xi))(\nabla^2 f(\bx^\star))^{-1}$.~Since $\Sigma^{\star}$ is generally unknown,~recent work by \cite{Fang2018Online} developed a bootstrap procedure to construct~\mbox{confidence}~\mbox{intervals}~by~approximating the limiting distribution using empirical bootstrap samples. To avoid high computational cost of bootstrap, later studies focused on covariance matrix estimation. 
\cite{Chen2020Statistical} designed an~online plug-in estimator for $\Sigma^{\star}$ by separately estimating $(\nabla^2 f(\bx^\star))^{-1}$ and $\text{cov}(\nabla F(\bx^{\star}; \xi))$~\mbox{using}~sample averages. 
Although this estimator performs well, it requires computing the Hessian inverse, which~is intrusive and unnecessary for SGD updates themselves.
As an alternative, \cite{Zhu2021Online} proposed an online batch-means covariance estimator using only SGD iterates. While this approach is more computationally efficient, it suffers from a slow convergence rate. In practice, estimating the full covariance matrix can be difficult and is often unnecessary, especially when the primary goal is to construct confidence intervals.
Consequently, recent work has focused on developing asymptotically pivotal statistics for online inference in SGD, thereby avoiding the need for accurate covariance matrix estimation. One such approach is the random scaling method \citep{Lee2022Fast},~which~self-normalizes the estimation error of averaged SGD iterates using a random scaling matrix. This~technique has demonstrated promising performance across a wide range of optimization algorithms, including ROOT-SGD \citep{Luo2022Covariance}, stochastic approximation under Markovian data \citep{Li2023Online}, weighted-averaged SGD \citep{Wei2023Weighted}, and the Kiefer–Wolfowitz algorithm \citep{Chen2024Online}. We extend the above literature by applying random scaling technique to AI-SSQP methods for constrained inference problems.

Another line of related literature focuses on online inference using second-order methods. \cite{Leluc2023Asymptotic} established asymptotic normality for the last iterate of conditioned~SGD~and~observed that the limiting covariance achieves optimality when the conditioning matrix is the Hessian. \cite{Bercu2020Efficient} designed an online Newton method for logistic regression, and \cite{Boyer2023asymptotic, Cenac2025efficient} extended that approach to more general \mbox{regression}~problems. Although the aforementioned works provided asymptotic normality guarantees comparable to those for first-order methods, they do not offer a clear approach for performing online inference based~on the normality property. 
Moreover, these methods are restricted to regression and/or convex problems, making them inapplicable to the nonlinear problems considered in our work.~For~example,~the Hessian inverse in regression problems often exhibits a nice structure as a sum of rank-one matrices, enabling efficient Hessian inverse updates via the Sherman–Morrison formula. However, this~computational advantage, relying on specific problem structures, does not extend to constrained regression problems. 
To our knowledge, this paper presents the first asymptotically valid inference~procedure~for nonlinear and nonconvex problems using second-order methods, with a computational cost matching that of state-of-the-art first-order methods.

\subsection{Organization and notation}\label{sec:1.3}

We begin by reviewing the AI-SSQP method in Section \ref{sec:2}. Section \ref{sec:3} presents the assumptions~and establishes the asymptotic normality of the averaged AI-SSQP iterates. In Section \ref{sec:4}, we introduce the pivotal test statistic and develop its theoretical properties. Numerical experiments and conclusions are provided in Sections \ref{sec:5} and \ref{sec:6}, respectively. All proofs are deferred to the appendices.

We let $\|\cdot\|$ denote the $\ell_2$ norm for vectors and spectral norm for matrices, $\|\cdot\|_{F}$ denote~the~Frobenius norm for matrices, and $\text{Tr}(\cdot)$ denote the trace of a matrix. For two sequences $\{a_t, b_t\}$, $a_t = O(b_t)$ (also written as $a_t \lesssim b_t$) implies that $|a_t| \leq c |b_t|$ for sufficient large $t$, where $c$ is~a~positive~constant. Similarly, $a_t = o(b_t)$ implies that $|a_t/b_t| \to 0$ as $t \to \infty$.~We let $I$ \mbox{denote} the~\mbox{identity}~matrix, $\boldsymbol{0}$ denote the zero vector or matrix, and $\boldsymbol{e}_i$ denote the vector with 1 at the $i$-th entry and 0 elsewhere. We let $\mathbf{1}_{\{\cdot\}}$ denote the indicator function and $\lfloor \cdot \rfloor$ denote~the~floor~\mbox{function}, which rounds~down~to the nearest integer. For a sequence of \mbox{compatible}~\mbox{matrices} $\{A_i\}_i$, $\prod_{k=i}^{j}A_k = A_jA_{j-1}\cdots A_i$~when $j\geq i$ and $I$ when $j<i$. For two matrices $A$ and $B$, $A \succeq B$ means that $A - B$ is positive~semidefinite. 
We reserve the notation $G(\bx)$ to denote the constraints Jacobian, that is,~$G(\bx) = \nabla c(\bx) = (\nabla c_1(\bx),\ldots,\nabla c_m(\bx))^{\top}\in \mR^{m\times d}$. Let $\mL(\bx,\blambda) = f(\bx) + \blambda^{\top}c(\bx)$ denote the Lagrangian function of \eqref{pro:1}, where $\blambda\in\mR^m$ is the dual vector associated with equality constraints. For simplicity, we let  $f_t = f(\bx_t)$, $c_t = c(\bx_t)$, $\nabla\mL_t = \nabla\mL(\bx_t,\blambda_t)$ (similar for $\nabla f_t$, $G_t$ etc.); and $f^\star = f(\bx^\star)$, $c^\star = c(\bx^\star)$, $\nabla\mL^\star = \nabla\mL(\tx,\tlambda)$ (similar for $\nabla f^\star$, $G^\star$ etc.).
The bar notation~$\bar{(\cdot)}$ denotes random quantities at each step, except for the averaged iterate $(\bar{\bx}_t, \bar{\blambda}_t) = \sum_{i = 0}^{t-1} (\bx_i, \blambda_i)/t$. For example, $\bnabla_{\bx}\mL_t$ is the random estimate of $\nabla_{\bx}\mL$ at the step $t$.

%% file: sec2.tex
\section{Adaptive Inexact Stochastic Sequential Quadratic Programming}\label{sec:2}

In this section, we review the AI-SSQP method for solving Problem \eqref{pro:1}. For constrained \mbox{problems},~it~is known that under certain constraint qualifications (cf. Assumption \ref{ass:1}), the first-order necessary condition for $(\tx, \tlambda)$ to be a local solution of \eqref{pro:1} is the KKT condition $\nabla\mL^\star = \nabla\mL(\tx, \tlambda) = \0$ \citep{Nocedal2006Numerical}. As such, AI-SSQP applies Newton’s method to the equation $\nabla \mL(\bx, \blambda) = \0$ in three steps: estimating the objective gradient and Hessian, (inexactly) solving Newton’s system,~and updating the primal-dual iterate. We refer to \cite{Na2025Statistical} for further details.  

\vskip0.2cm
\noindent\textbf{$\bullet$ Step 1: Estimate objective gradient and Hessian.} Given the current iterate $(\bx_t, \blambda_t)$, we draw a sample $\xi_t \sim \P$ and compute the stochastic objective gradient and Hessian estimates as 
\begin{equation*}
\barg_t = \nabla F(\bx_t; \xi_t)\quad\quad \text{ and }\quad\quad \barH_t = \nabla^2 F(\bx_t; \xi_t).
\end{equation*}
Recalling the Jacobian matrix of the constraints $G_t = \nabla c_t$, we then compute the stochastic estimates of the Lagrangian gradient and Hessian with respect to $\bx$ as
\begin{equation*}
\bnabla_{\bx}\mL_t = \barg_t + G_t^{\top}\blambda_t \quad\quad \text{ and }\quad\quad \bnabla_{\bx}^2\mL_t = \barH_t + \sum_{j=1}^m[\blambda_t]_j\nabla^2c_j(\bx_t).
\end{equation*}
Next, we define the regularized averaged Hessian  $B_t$:
\begin{equation}
B_t = \frac{1}{t}\sum_{i=0}^{t-1}\bnabla_{\bx}^2\mL_i + \Delta_t,
\end{equation}
where $\Delta_t = \Delta(\bx_t, \blambda_t)$ is a Hessian regularization term chosen such that $B_t$ is positive definite in the null space $\{\bx \in \mR^d: G_t \bx = \boldsymbol{0}\}$. 
This regularization, combined with the constraint qualification condition (LICQ, cf. Assumption \ref{ass:1}), ensures that the SQP subproblem \eqref{equ:QP} admits a unique solution. For convex problems, we may simply set $\Delta_t = \boldsymbol{0}$ for all $t$; for general constrained~\mbox{problems},~we~may~choose $\Delta_t = \delta_t I$ for a large enough $\delta_t > 0$ \citep{Nocedal2006Numerical}. In addition, we highlight that~$B_t$~is~an average over the samples $\{\xi_i\}_{i=0}^{t-1}$, suggesting that~$B_t$~and~$\Delta_t$~are deterministic given $(\bx_t, \blambda_t)$. 
The averaging facilitates the convergence of the Hessian matrix (i.e., the law of large numbers for martingales), while excluding $\xi_t$ reduces the conditional bias of~the~stochastic~Newton direction~(see~\eqref{equ:Newton}).$\quad\;$

\vskip0.2cm
\noindent\textbf{$\bullet$ Step 2: Inexactly solve the SQP subproblem.} With the quantities $\barg_t, G_t, \bnabla_{\bx}\mL_t, B_t$, we~then formulate the SQP subproblem by performing a quadratic approximation to the objective and a linear approximation to the constraints. Specifically, we aim to solve the following quadratic~program:
\begin{equation}\label{equ:QP}
\min_{\tDelta\bx_t} \;\; \frac{1}{2}\tDelta\bx_t^{\top}B_t\;\tDelta\bx_t + \barg_t^{\top}\tDelta\bx_t, \quad \text{s.t.}\quad\; c_t +  G_t\tDelta\bx_t = \0.
\end{equation}
It can be shown that the above quadratic program is equivalent to the linear system
\begin{equation}\label{equ:Newton}
\underbrace{\begin{pmatrix}
B_t & G_t^{\top}\\
G_t & \0
\end{pmatrix}}_{K_t}\underbrace{\begin{pmatrix}
\tDelta\bx_t\\
\tDelta\blambda_t
\end{pmatrix}}_{\tbz_t} = -\underbrace{\begin{pmatrix}
\bnabla_{\bx}\mL_t\\
c_t
\end{pmatrix}}_{\bnabla\mL_t}.
\end{equation}
Solving the above Newton system by computing $K_t^{-1}$ to derive the exact stochastic Newton direction $\tbz_t \coloneqq (\tDelta\bx_t, \tDelta\blambda_t)$ incurs a computational complexity of $O((d+m)^3)$, which is less competitive~compared to first-order methods. To reduce the computational overhead, we employ a sketching solver~to derive an approximate solution.

In particular, for each $t$, we perform $\tau$ sketching steps. At each sketching step $j$, we independently generate a random sketching matrix/vector $S_{t,j} \in \mR^{(d+m) \times q}$ for some $q \geq 1$ from a distribution $S$, and aim to solve the following sketched Newton system: 
\begin{equation*}
S_{t,j}^{\top}K_t\bz = -S_{t,j}^{\top}\bnabla\mL_t.
\end{equation*}
However, this sketched system generally has multiple solutions, including the exact Newton direction $\Tilde{\bz}_t$. We prefer the solution that is closest to the current solution approximation $\bz_{t, j}$, that is ($\bz_{t, 0} = \boldsymbol{0}$),
\begin{equation}\label{equ:stochastic_proj}
\bz_{t,j+1} = \arg\min_{\bz}\|\bz - \bz_{t,j}\|^2\quad\quad  \text{s.t.} \quad\; S_{t,j}^{\top}K_t\bz = -S_{t,j}^{\top}\bnabla\mL_t.
\end{equation}
The closed form solution to \eqref{equ:stochastic_proj} is given by:
\begin{equation}\label{equ:pseduo}
\bz_{t,j+1} = \bz_{t,j} - K_tS_{t,j}(S_{t,j}^{\top}K_t^2S_{t,j})^\dagger S_{t,j}^{\top}(K_t\bz_{t,j} +\bnabla\mL_t),\quad\quad 0\leq j\leq \tau-1,
\end{equation}
where $(\cdot)^\dagger$ denotes the Moore–Penrose pseudoinverse. Finally, we define
\begin{equation}\label{def:zttau}
(\barDelta\bx_t, \barDelta\blambda_t) \coloneqq \bz_{t,\tau}
\end{equation}
as our approximate Newton direction.

\vskip0.2cm
\noindent\textbf{$\bullet$ Step 3: Adaptive stepsize}. With the approximate Newton direction $(\barDelta\bx_t, \barDelta\blambda_t) = \bz_{t,\tau}$, we finally update the iterate $(\bx_t, \blambda_t)$ with an adaptive random stepsize $\baralpha_t$:
\begin{equation}\label{equ:update}
(\bx_{t+1}, \blambda_{t+1}) = (\bx_t, \blambda_t) + \baralpha_t\cdot (\barDelta\bx_t, \barDelta\blambda_t).
\end{equation}
The AI-SSQP method allows any adaptive stepsize selection scheme as long as the safeguard~condition holds:
\begin{equation}\label{equ:sandwich}
0< \beta_t\leq \baralpha_t \leq \eta_t\quad  \text{ with }\quad  \eta_t = \beta_t + \chi_t.
\end{equation}
Here, $\{\beta_t, \eta_t\}$ are sequences serving as lower and upper bounds, respectively, and $\chi_t$~represents the adaptivity gap between them. See \cite{Berahas2021Sequential, Berahas2024Stochastic, Curtis2024Stochastic, Na2025Statistical} for line-search-based stepsize selection schemes that adhere to the safeguard~condition \eqref{equ:sandwich}.

To conclude this section, we present a remark discussing the computational complexity of~the AI-SSQP method, as well as its sources of randomness from both data and computation, in~comparison~to first-order methods. We refer the reader to \cite[Remark 4.5]{Na2025Statistical} for more discussions.

\begin{remark}\label{rem:1}

The AI-SSQP method relates to (unconstrained) sketched Newton methods~\citep{Pilanci2016Iterative, Pilanci2017Newton, Gower2015Randomized, Gower2019RSN, Lacotte2020Optimal, Lacotte2021Adaptive, Hong2023Constrained}. It is well known that solving the Newton system \eqref{equ:Newton} is typically~the~most~computationally expensive step in second-order methods. Randomized sketching solvers can significantly reduce this cost, particularly when equipped with suitable sketching matrices (e.g., sparse sketches) \citep{Strohmer2008Randomized, Gower2015Randomized, Luo2016Efficient, Doikov2018Randomized, Derezinski2020Debiasing, Derezinski2021Newton}.

For first-order methods such as SGD, the per-iteration computational flops is $O(d + m)$. However, performing valid online inference is more expensive than merely achieving convergence, as it requires matrix updates to capture distributional information, resulting in the per-iteration computational cost of at least $O((d + m)^2)$.
In comparison, \cite{Na2025Statistical} showed that using sparse~sketching vectors -- e.g., the randomized Kaczmarz method in \cite{Strohmer2008Randomized} samples~$S \sim \text{Uniform}(\{\be_i\}_{i=1}^{d+m})$ -- yields $O(d + m)$ flops per sketching step in \eqref{equ:pseduo}, with $\tau = O(d + m)$ sketching~steps in total to ensure the convergence of AI-SSQP.
This suggests that the per-iteration cost of AI-SSQP is $O((d + m)^2)$. In our work, we further demonstrate that one can directly leverage AI-SSQP~iterates to construct valid confidence intervals, thereby achieving the same $O((d + m)^2)$ cost as that required for convergence.
This stands in contrast to the $O((d + m)^3)$ cost incurred by the plug-in~covariance~matrix estimator used in \cite{Na2025Statistical}. Therefore, our work demonstrates that AI-SSQP matches the computational complexity of SGD for inference tasks while achieving faster convergence through second-order updates.

We highlight that AI-SSQP introduces additional sources of randomness beyond random sampling. Specifically, the sketching solver is randomized at each iteration; even with a fixed sample~$\xi_t$,~the~approximate Newton direction can vary depending on the realized sketching matrices. Furthermore,~the adaptive stepsize $\baralpha_t$ may be determined by the stochastic direction $(\barDelta\bx_t, \barDelta\blambda_t)$ and is also random. These additional layers of randomness make the uncertainty quantification analysis for AI-SSQP~more challenging than that for SGD.

\end{remark}

%% file: sec3.tex
\section{Asymptotic Normality of Averaged AI-SSQP}\label{sec:3}

In this section, we first introduce the key assumptions required for the analysis of AI-SSQP,~and~then establish the asymptotic normality of its averaged iterate.~Our assumptions are identical to~(in~fact, slightly weaker than) those in \cite{Na2025Statistical}, which are also standard in the existing literature. Note that the prior work investigated the asymptotic normality of the last iterate,~while~we~complement the study by establishing the asymptotic normality of the averaged~\mbox{iterate}.~We further~demonstrate that the averaged iterate exhibits better statistical efficiency, as suggested by the smaller limiting covariance matrix.

\subsection{Assumptions and preliminary results}\label{sec:3.1}

The first assumption is about constraint qualification condition and the Lipschitz continuity of~the~Lagrangian Hessian.~These assumptions are common in constrained optimization literature \citep{Nocedal2006Numerical, Bertsekas2014Constrained}.

\begin{assumption}\label{ass:1}

We assume that all the iterates $\{(\bx_t, \blambda_t)\}_t$ are contained in a closed, bounded,~convex set $\mX\times \Lambda$, such that $f(\bx)$ and $c(\bx)$ are twice continuously differentiable over $\mX$, and the~Lagrangian Hessian $\nabla^2\mL$ is $\Upsilon_{L}$-Lipschitz continuous over $\mX\times \Lambda$, that is,
\begin{equation}\label{Lip:mL}
\|\nabla^2\mL(\bx, \blambda) - \nabla^2\mL(\bx', \blambda')\| \leq\Upsilon_L \|(\bx-\bx', \blambda-\blambda')\|, \quad \forall\; (\bx, \blambda), (\bx', \blambda')\in \mX\times \Lambda.
\end{equation}
In addition, we assume that the Jacobian of the constraints $G_t$ has full row rank satisfying~$G_t G_t^{\top} \succeq \gamma_{G} I$ for a constant $\gamma_{G} > 0$.
The regularization term $\Delta_t$ ensures that $B_t$ satisfies $\|B_t\|\leq \Upsilon_B$ and $\bx^{\top}B_t\bx\geq \gamma_{RH}\|\bx\|^2$, $\forall \bx\in\text{Kernel}(G_t)$ for some constants $\Upsilon_B, \gamma_{RH}>0$.

\end{assumption}

As mentioned before, the full-rank condition of $G_t$ (known as the linear independence constraint qualification, LICQ) together with the conditions on $B_t$ ensure that \eqref{equ:Newton} has a unique solution~\cite[Lemma 16.1]{Nocedal2006Numerical}. The twice continuous differentiability of $f(\bx)$ and $c(\bx)$ over the~set $\mX$, along with the Lipschitz continuity of the Hessian matrix $\nabla^2 \mL$, are also standard requirements for analyzing Newton’s and SQP methods.

The next assumption imposes the moment conditions on the stochastic estimates of the objective gradient and Hessian. 

\begin{assumption}\label{ass:2}
We assume the gradient and Hessian estimates are unbiased: $\mE[\nabla F(\bx_t; \xi_t) \mid \bx_t] = \\ \nabla f_t$ and $\mE[\nabla^2F(\bx_t; \xi_t) \mid \bx_t] = \nabla^2f_t$, $\forall t\geq 0$; and we assume the following moment~conditions:
\begin{equation}\label{equ:BM}
\mE[\|\nabla F(\bx_t; \xi_t) - \nabla f_t\|^{2+\delta} \mid \bx_t]\leq \Upsilon_m\quad\quad \text{ and }\quad\quad \mE[\sup_{\bx\in\mX}\|\nabla^2F(\bx;\xi)\|^2] \leq \Upsilon_m
\end{equation}
for some constants $\delta, \Upsilon_m>0$.
\end{assumption}

We highlight that our inference procedure only requires a bounded $(2+\delta)$-moment for the~gradient estimates. This is weaker than existing unconstrained and constrained inference procedures that rely on covariance matrix estimation, which require \textit{at least} a bounded fourth-order moment \citep{Chen2020Statistical, Zhu2021Online, Na2025Statistical, Davis2024Asymptotic, Jiang2025Online, Kuang2025Online}. In particular, for constrained inference problems, \cite{Na2025Statistical} imposed a $(2+\delta)$-moment condition to establish the asymptotic normality of AI-SSQP methods,~and~further~strengthened it to a fourth-order moment condition for inference. Additionally, \cite{Davis2024Asymptotic} imposed a fourth-order moment condition to establish the asymptotic normality of projected SGD methods,~and further strengthened it to a sub-Gaussian tail condition for valid inference. 

The condition on the Hessian estimate is also standard and is required even for SGD methods \cite[Assumption 3.2(2) and Lemma 3.1]{Chen2020Statistical}. It ensures the Lipschitz continuity~of~the~mapping  $\bx \rightarrow \mE[\nabla f(\bx; \xi) \nabla f(\bx; \xi)^{\top}]$. We note that some works may directly impose the condition on~this mapping as an alternative assumption \cite[Assumption 4]{Leluc2023Asymptotic} \cite[Assumption J]{Davis2024Asymptotic}. In fact, the Hessian condition has been shown to~hold~in~\mbox{various}~\mbox{problems},~such~as~least-squares and logistic regressions \citep{Chen2020Statistical, Na2025Statistical}.

The next assumption characterizes the sketching distribution, which ensures the linear convergence of the sketching solver.

\begin{assumption}\label{ass:3}
We assume the sketching matrix $S\in\mR^{(d+m)\times q}$ satisfies $\mE[\|S\|\|S^\dagger\|] \leq \Upsilon_S$ and 
\begin{equation}\label{eq:sketching_assumption_2}
\mE[K_tS(S^{\top}K_t^2S)^\dagger S^{\top}K_t \mid \bx_t, \blambda_t]\succeq \gamma_{S} I \quad \quad \text{ for any } t\geq 0
\end{equation}
for some constants $\gamma_{S}, \Upsilon_S >0$.
\end{assumption}

We note that the two expectation terms above are taken over the randomness of $S$. The first~term bounds the expected condition number of the sketching matrix, which is trivially satisfied when~$q=1$, i.e., when employing sketching vectors. The second term lower bounds the expected~projection~matrix, which is also standard and satisfied by various sketching methods \citep{Gower2015Randomized}.~For example, for randomized Kaczmarz method with $S \sim \text{Uniform}(\{\be_i\}_{i=1}^{d+m})$, \cite{Na2025Statistical} has shown that $\gamma_{S} \geq 1/\{(d+m)\kappa^2(K_t)\}$ where $\kappa(K_t)$ denotes the~\mbox{condition}~number~of~$K_t$ that~is~uniformly upper bounded under Assumption \ref{ass:1}.

With the above assumptions, we now review the global almost-sure convergence guarantee~of~AI-SSQP established in \cite{Na2025Statistical}. In nonlinear optimization, global convergence refers~to the guarantee that the algorithm will converge to a stationary point from any initialization.

\begin{theorem}[\cite{Na2025Statistical}, Theorem 4.8]\label{thm:1}
Consider the AI-SSQP updates in \eqref{equ:update}~under Assumptions \ref{ass:1}, \ref{ass:2}, \ref{ass:3}. There exists a threshold $\tau^\star$ such that for any sketching steps $\tau\geq \tau^\star$~and~any stepsize control sequences $\{\beta_t, \eta_t = \beta_t+\chi_t\}$ satisfying
\begin{equation}\label{cond:1}
\sum_{t=0}^{\infty} \beta_t = \infty, \quad\quad \sum_{t=0}^{\infty} \beta_t^2 <\infty, \quad\quad \sum_{t=0}^{\infty}\chi_t  <\infty,
\end{equation}
we have $\|\nabla\mL_t\|\rightarrow 0$ and $\|(\bx_{t+1}-\bx_t, \blambda_{t+1}-\blambda_t)\|\rightarrow0$ as $t\rightarrow\infty$ almost surely.
\end{theorem}

The threshold $\tau^\star$ has an explicit form in \cite{Na2025Statistical}, depending on the parameters~in the presented assumptions. More specifically, $\tau^\star=O(1/\log\{1/(1-\gamma_{S})\}) = O(1/\gamma_{S})$, suggesting~that $\tau^\star = O(d+m)$ when applying randomized Kaczmarz method, where $\gamma_{S} = O(1/(d+m))$.

Theorem \ref{thm:1} shows that AI-SSQP converges almost surely to a stationary point, which serves~as the foundation for exploring the local asymptotic behavior of its iterates.
To segue into the~local analysis, we assume from now on that the iterates $(\bx_t, \blambda_t)$ converge to a \textit{regular} local solution~$(\tx, \tlambda)$,~that is, $G^\star = \nabla c^\star$ has full row rank and $\nabla_{\bx}^2\mL^\star$ is positive definite on $\text{Kernel}(G^\star)$. Note that these regularity conditions are necessary even for offline $M$-estimators \citep{Shapiro2014Lectures, Duchi2021Asymptotic, Davis2024Asymptotic, Na2025Statistical}.

\subsection{Asymptotic normality}\label{sec:3.2}

In this subsection, we examine the asymptotic normality property of the averaged AI-SSQP iterate
\begin{equation*}
\bar{\bx}_t = \frac{1}{t}\sum_{i = 0}^{t-1}\bx_i \quad \text{ and }\quad \barlambda_t =  \frac{1}{t}\sum_{i = 0}^{t-1}\blambda_i.
\end{equation*}
We aim to show that, compared to the last iterate $(\bx_t, \blambda_t)$, the averaged iterate enjoys better~efficiency. In particular, similar to the limiting covariance $\tilde{\Xi}^\star$ of the last iterate (cf. \eqref{equ:CLT_last_iterate}), the limiting~covariance $\bar{\Xi}^\star$ of the averaged iterate reduces to the \textit{minimax optimal} covariance $\Xi^\star$, achieved by the constrained $M$-estimators, when the sketching solver is suppressed. \mbox{However},~when~the~\mbox{sketching}~solver~is~deployed, we arrive at a computational-statistical trade-off: some efficiency is lost with $\Xi^\star \preceq \bar{\Xi}^\star \preceq \tilde{\Xi}^\star$; nevertheless, the loss is tolerable and decays exponentially in terms of the number of sketching steps: $\|\bar{\Xi}^\star - \Xi^\star\| \leq \|\tilde{\Xi}^\star - \Xi^\star\| \leq O((1 - \gamma_S)^\tau)$.

We begin by laying out the asymptotic optimality of constrained $M$-estimator $(\hat{\bx}_t, \hat{\blambda}_t)$, which generates $t$ samples and approximates the population problem \eqref{pro:1} by solving the empirical \mbox{problem}~\eqref{pro:2}. Here, $\hat{\blambda}_t$ is the associated dual solution of \eqref{pro:2}. \citep[Chapter 5]{Shapiro2014Lectures} has shown that
\begin{equation}\label{equ:M:est}
\sqrt{t} \cdot ( \hat{\bx}_t - \tx, \hat{\blambda}_t - \tlambda )\stackrel{d}{\longrightarrow}\mN\rbr{\0,\; \Xi^{\star}}
\end{equation} 
with the limiting covariance $\Xi^{\star}$ given by
\begin{equation}\label{equ:Omega}
\Xi^{\star} = {\underbrace{\begin{pmatrix}
\nabla_{\bx}^2\mL^\star & (G^{\star})^{\top}\\
G^{\star} & \0
\end{pmatrix}}_{K^\star = \nabla^2\mL^\star}}^{-1}\underbrace{\begin{pmatrix}
\cov(\nabla F(\tx; \xi)) & \0\\
\0 & \0
\end{pmatrix}}_{\text{cov}(\nabla\mL(\tx,\tlambda;\xi))}\begin{pmatrix}
\nabla_{\bx}^2\mL^\star & (G^{\star})^{\top}\\
G^{\star} & \0
\end{pmatrix}^{-1}.
\end{equation}
Furthermore, \cite{Duchi2021Asymptotic, Davis2024Asymptotic} have shown that the martingale covariance~of $\Xi^{\star}$ for the $\bx$ component is locally asymptotically minimax optimal in the sense of H\'ajek and Le~Cam. Their results can be further generalized to the optimality of the full covariance $\Xi^\star$.

For the AI-SSQP method, due to the presence of sketching steps in \eqref{equ:pseduo}, we have to define~the~product of the projection matrices (projecting onto $\text{Span}(K^\star S)$):
\begin{equation}\label{equ:tC}
\tC^\star \coloneqq \prod_{j=1}^{\tau}(I - K^\star S_j(S_j^{\top}(K^\star)^2S_j)^\dagger S_j^{\top}K^\star ), \quad C^{\star} \coloneqq \mE[\Tilde{C}^{\star}] = (I - \mE[K^\star S(S^{\top}(K^\star)^2S)^\dagger S^{\top}K^\star])^\tau,
\end{equation} 
where $S_1,\ldots,S_\tau \stackrel{iid}{\sim} S$. 
The limiting covariance of AI-SSQP is adjusted from $\Xi^{\star}$ by these projection matrices to account for the effects of random sketching. 

\begin{theorem}[Asymptotic normality of averaged AI-SSQP]\label{thm:4}

\hskip-0.1cm Under Assumptions \ref{ass:1}, \ref{ass:2}, \ref{ass:3}~and~suppose the sketching step $\tau\geq \tau^\star$ and the stepsize control sequences $\beta_t = c_\beta/ (t+1)^\beta$ and $\chi_t = c_\chi/(t+1)^\chi$ satisfy $c_\beta, c_\chi >0$, $\beta\in(0.5, 1)$, and $\chi >\beta+0.5$. Then, we have
\begin{equation}\label{equ:CLT_averaged_iterate}
\sqrt{t} \cdot (\bar{\bx}_t - \bx^{\star}, \bar{\blambda}_t - \blambda^{\star}) \stackrel{d}{\longrightarrow} \N(\0, \bar{\Xi}^\star),
\end{equation}
where
\begin{equation}\label{equ:barXi}
\bar{\Xi}^\star = (I - C^{\star})^{-1}\mathbb{E}[(I - \widetilde{C}^{\star})\Xi^{\star}(I-\widetilde{C}^{\star})^{\top}](I- C^{\star})^{-1}.
\end{equation}	
\end{theorem}

We mention that under Assumption \ref{ass:3}, $I - C^{\star}$ is positive definite and hence invertible, since $K^\star S(S^{\top}(K^\star)^2S)^\dagger S^{\top}K^\star$ is a projection matrix with a positive lower bound away from zero.~Note~that the conditions on the stepsize sequences $\{\beta_t, \chi_t\}$ imply the condition \eqref{cond:1} required for the global~convergence.~Compared to the conditions of the normality of the last~\mbox{AI-SSQP} iterate~\cite[Lemma 5.12]{Na2025Statistical} where $\beta \hskip-1pt\in \hskip-1pt (0.5, 1]$ and $\chi\hskip-1pt >\hskip-1pt\max\{1, 1.5\beta\}$, we~\mbox{exclude}~the~case~\mbox{$\beta = 1$}.~We~\mbox{justify}~this restriction from two aspects. 
First, for the last \mbox{iterate},~the~\mbox{optimal}~\mbox{covariance}~is~attained when~$c_\beta = 1$, $\beta = 1$, and $\chi > 1.5$; while for the averaged iterate, the optimal covariance is attained for any $c_\beta > 0$, $\beta \in (0.5, 1)$, and $\chi > 1.5$. In practice, a smaller $\beta$ (i.e., a larger stepsize) is often preferred for faster convergence. Thus, our condition enables faster convergence \mbox{without}~\mbox{compromising}~the~\mbox{optimality} of the AI-SSQP method. 
Second, our condition generalizes the one used to analyze the \mbox{(projection-based)} averaged SGD method, where a deterministic stepsize is adopted ($\chi = \infty$) with $\beta \in (0.5, 1)$ \citep{Polyak1992Acceleration, Chen2020Statistical, Zhu2021Online, Duchi2021Asymptotic, Davis2024Asymptotic}.$\quad\quad$

Next, we examine in the following proposition the relationship between $\bar{\Xi}^{\star}$ and $\Xi^\star$.

\begin{proposition}\label{prop:2}
Under the conditions of Theorem \ref{thm:4}, we have 

\noindent (a): Without the sketching solver (i.e., solving \eqref{equ:Newton} exactly), $\bar{\Xi}^{\star} = \Xi^\star$.

\noindent (b): With the sketching solver, $\bar{\Xi}^{\star} \succeq \Xi^{\star}$. Furthermore, their difference can be bounded by
\begin{equation*}
\|\bar{\Xi}^{\star} - \Xi^{\star}\| \leq \frac{1+(1-\gamma_S)^\tau}{\{1-(1-\gamma_S)^\tau\}^2}\cdot(1-\gamma_S)^\tau \|\Xi^{\star}\|.
\end{equation*}

\end{proposition}

From the above proposition, we see that the statistical efficiency of AI-SSQP is indeed affected~by the underlying sketching solver. Without sketching, the method achieves optimal statistical efficiency. With sketching, the randomized solver introduces uncertainty when solving the Newton systems, which in turn improves the method’s computational efficiency. However, this comes at the cost~of~degraded statistical efficiency, leading to $\bar{\Xi}^{\star} \succeq \Xi^{\star}$. Fortunately, regardless of the sketching distribution used, the degradation remains tolerable and decays exponentially with the number~of~sketching~steps, resulting in a favorable computational–statistical trade-off.

In the next proposition, we compare the statistical efficiency between the averaged and last~iterates of AI-SSQP. Note that the computational efficiency is the same for both. In particular, \cite{Na2025Statistical} has shown for $\beta \in (0.5, 1]$ and $\chi > \max\{1, 1.5\beta\}$ that
\begin{equation}\label{nequ:1}
\sqrt{1/\beta_t}\cdot (\bx_t - \bx^{\star}, \blambda_t - \blambda^{\star}) \stackrel{d}{\longrightarrow} \N(\0, \tilde{\Xi}^\star),
\end{equation}
where the limiting covariance $\tilde{\Xi}^\star$ solves the following Lyapunov equation:
\begin{equation}\label{equ:Lyapunov}
\rbr{\cbr{1 - \frac{\b1_{\{\beta=1\}}}{2c_\beta}} I - C^{\star}} \tilde{\Xi}^\star + \tilde{\Xi}^\star\rbr{\cbr{1 - \frac{\b1_{\{\beta=1\}}}{2c_\beta}} I - C^{\star}} = \mathbb{E}[ (I - \widetilde{C}^{\star})\Xi^{\star}(I - \widetilde{C}^{\star})^{\top}].
\end{equation}
To compare under the same scaling as \eqref{equ:CLT_averaged_iterate}, we let $\beta_t=1/(t+1)$, corresponding to $c_\beta = 1$ and~$\beta=1$.

\begin{proposition}\label{prop:3}
Consider \eqref{equ:Lyapunov} with $c_\beta = 1$ and $\beta=1$. If $(1-\gamma_S)^\tau<0.5$, then we have $\bar{\Xi}^{\star}\preceq \tilde{\Xi}^{\star}$.
\end{proposition}

This proposition shows that the averaged iterate exhibits better statistical efficiency than the last iterate, regardless of the sketching distribution used.

%% file: sec4.tex
\section{Online Statistical Inference}\label{sec:4}

In this section, we leverage the established asymptotic normality of the averaged AI-SSQP method~to perform valid online statistical inference for $(\tx, \tlambda)$. 

A direct approach is to estimate the limiting covariance to normalize the estimation error.~For~unconstrained strongly convex problems, \cite{Chen2020Statistical} and \cite{Zhu2021Online} proposed plug-in~and batch-means covariance estimators, respectively, for averaged SGD methods, both of which are~consistent and hence fulfill the desired goal. \cite{Kuang2025Online} further enhanced the covariance~estimation by developing a batch-free estimator for stochastic Newton methods, and demonstrated its improved convergence rate. 
For constrained nonconvex problems, \cite{Na2025Statistical} proposed a computationally expensive plug-in covariance estimator for $\tilde{\Xi}^\star$. Due to the challenges of \mbox{estimating}~sketching components in \eqref{equ:Lyapunov}, the authors simply omitted all sketching-induced projection matrices, resulting in a biased estimator. Subsequently, \cite{Jiang2025Online} designed a batch-means covariance~estimator for projected SGD methods to reduce computational cost. However, their analysis relies on~sub-Gaussian gradient noise, a stronger condition than what is needed for establishing asymptotic~normality.

The aforementioned literature motivates us to develop a valid inference procedure for the~constrained problem \eqref{pro:1} that (i) matches the computational cost of first-order methods and (ii) does~not impose stronger moment conditions on the gradient noise. To this end, we leverage the random scaling technique to studentize the estimation error, rendering the limiting distribution free~of~any~unknown quantities. As a result, we can bypass the need for covariance matrix estimation.

As the first step, we extend the normality in Theorem \ref{thm:4} to the Functional Central Limit~Theorem (FCLT).

\begin{theorem}\label{thm:FCLT}

Under the conditions of Theorem \ref{thm:4}, we have 
\begin{equation*}
\frac{1}{\sqrt{t}} \sum_{i = 0}^{\lfloor rt \rfloor -1} ( \boldsymbol{x}_{i} - \boldsymbol{x}^{\star}, \boldsymbol{\lambda}_{i} - \boldsymbol{\lambda}^{\star}) \Longrightarrow  (\bar{\Xi}^{\star})^{1/2} W_{d+m}(r),\quad\quad r\in[0, 1],
\end{equation*}
where $W_{d+m}(\cdot)$ is the standard $(d+m)$-dimensional Brownian motion and $\bar{\Xi}^{\star}$ is defined in \eqref{equ:barXi}.

\end{theorem}

Here, we use ``$\Rightarrow$" to denote convergence in distribution in a function space, which~distinguishes it from pointwise notation $\stackrel{d}{\rightarrow}$. Specifically, $X_t(r)\Rightarrow X(r)$, $r\in[0, 1]$ if for any bounded and continuous functional $f: C[0, 1]\rightarrow \mR$, we have $\mE[f(X_t)]\rightarrow\mE[f(X)]$ as $t\rightarrow\infty$, where $C[0, 1]$ denotes the space of continuous function on $[0,1]$.

Theorem \ref{thm:4} is a special case of Theorem \ref{thm:FCLT} with $r=1$. The statement of Theorem~\ref{thm:FCLT}~resembles the FCLT established for unconstrained SGD methods \citep{Lee2022Fast, Wei2023Weighted, Luo2022Covariance, Chen2024Online}, differing mainly in the sketching-dependent scaling matrix $\bar{\Xi}^\star$.~However, two technical challenges arise in our constrained second-order method.
First, the update rule~\eqref{equ:update} is more complex than that of standard SGD, involving not only data randomness but also computational randomness (i.e., sketching) and adaptive stepsizes. We show in the proof that the randomness from adaptive stepsizes contributes only to higher-order errors, provided that the adaptivity gap satisfies $\chi_t = o(\beta_t/\sqrt{t})$, while the randomness from sketching is captured by the scaling matrix.
Second, the Lagrangian function $\mathcal{L}(\bx, \blambda) = f(\bx) + \blambda^{\top}c(\bx)$ exhibits only a saddle-point structure. In contrast to the global convexity structure leveraged in prior works, we have to develop a stopping-time~technique to localize our analysis (see Lemma \ref{lem:C1}). 
In particular, we demonstrate that within a neighborhood~of $(\tx, \tlambda)$, the AI-SSQP iterates converge at a desirable rate, which~together~with~the~global~\mbox{almost-sure} convergence guarantee ensures that the additional randomness does not degrade the limiting behavior of the partial sum process.

With Theorem \ref{thm:FCLT}, we now introduce a studentized, pivotal test statistic for online inference.~We first define the \textit{random scaling} matrix as
\begin{equation}\label{def:V_t}
V_t \coloneqq \frac{1}{t^2}\sum_{i = 1}^t i^2 \begin{pmatrix}
\bar{\bx}_i - \bar{\bx}_t \\
\bar{\blambda}_i - \bar{\blambda}_t
\end{pmatrix}\begin{pmatrix}
\bar{\bx}_i - \bar{\bx}_t \\
\bar{\blambda}_i - \bar{\blambda}_t
\end{pmatrix}^{\top}.
\end{equation}
Note that the matrix $V_t$ is \textit{not} intended to estimate the covariance $\bar{\Xi}^{\star}$. Instead, by the continuous~mapping theorem \cite[Theorem 3.4.3]{Whitt2002Stochastic}, we are able to show the following limiting~distribution~for the test statistic.

\begin{theorem}\label{thm:Random_Scaling}

Under the conditions of Theorem \ref{thm:4} and assuming $\cov(\nabla F(\tx; \xi))\succ 0$, for~any~vector $\boldsymbol{w} = (\boldsymbol{w}_{\bx}, \boldsymbol{w}_{\blambda}) \in \mR^{d+m}$ with $\bw\notin \text{Span}((G^\star)^\top)\otimes\0_m$, we have
\begin{equation}\label{eq:RS_converge_in_distribution}
\frac{\sqrt{t} \ \boldsymbol{w}^{\top} (\bar{\bx}_t - \bx^{\star}, \bar{\blambda}_t - \blambda^{\star})}{\sqrt{\boldsymbol{w}^{\top} V_t \boldsymbol{w}}} \stackrel{d}\longrightarrow \frac{W_1(1)}{\sqrt{\int_{0}^1  \left(W_1(r) - rW_1(1) \right)^2 dr}},
\end{equation}
where $W_1(\cdot)$ is the standard one-dimensional Brownian motion.
\end{theorem}

We should mention that the condition $\bw \notin \text{Span}((G^\star)^\top) \otimes \0_m$ is imposed to ensure that~the~inference direction is not aligned with the normal direction of the constraint function, along which~the variance is zero --- the projection of the model parameter $\tx$ onto any \mbox{normal}~\mbox{directions}~$\text{Span}((G^\star)^\top)=\text{Span}((\nabla c^\star)^\top)$ is always zero. In particular, due to the presence of the constraints $c(\bx) = \0$,~the covariance matrix $\Xi^\star$ in \eqref{equ:Omega} is singular. We can show that for any $\bw = (\bw_\bx, \bw_\blambda)$ with $\bw_\bx \in \text{Span}((G^\star)^\top)$ and $\bw_\blambda = \0$, we have $\bw^\top \Xi^\star \bw = 0$, and vice versa. This indicates that only inference along~the~tangential direction $\text{Kernel}(G^\star)$ is needed and non-trivial.

Theorem \ref{thm:Random_Scaling} shows that our test statistic is asymptotically pivotal, meaning its limiting distribution is free of any unknown parameters. In fact, the distribution in \eqref{eq:RS_converge_in_distribution} appears widely~in~econometrics and statistics literature, such as in cointegration analysis \citep{Johansen1991Estimation, Abadir1997Two} and robust inference \citep{Kiefer2000Simple, Abadir2002Simple}. For reference,~we~report~its quantiles from \cite{Abadir1997Two} in Table \ref{table:fclt}.

\begin{table}[H]
\centering
\begin{tabular}{c|cccc}
\toprule
$p$ & 90\% & 95\% & 97.5\% & 99\% \\
\hline\\[-10pt]
Quantile($p$) & 3.875 & 5.323 & 6.747 & 8.613 \\
\bottomrule
\end{tabular}
\caption{Quantile table of the distribution $W_1(1)/\{\int_{0}^1  \left(W_1(r) - rW_1(1) \right)^2 dr\}^{1/2}$.}
\label{table:fclt}
\end{table}

Theorem \ref{thm:Random_Scaling} directly leads to the following corollary, which demonstrates the construction of asymptotically valid confidence intervals for $(\tx, \tlambda)$.

\begin{corollary}\label{cor:RS_CI}
Under the conditions of Theorem \ref{thm:Random_Scaling}, for any $p\in (0, 1)$, we have
\begin{equation*}
P\left( \boldsymbol{w}^{\top}(\bx^{\star}, \blambda^{\star}) \in \sbr{ \boldsymbol{w}^{\top} (\bar{\bx}_t, \bar{\blambda}_t) \pm  U_{ 1 - p/2} \sqrt{\boldsymbol{w}^{\top}V_{t}\boldsymbol{w}/t } } \right) \longrightarrow 1 - p \qquad \text{as} \qquad t \to \infty,
\end{equation*}
where $U_{1 - p/2}$ denotes the $(1 - p/2)\times100\%$ quantile of the limiting distribution in \eqref{eq:RS_converge_in_distribution}.
\end{corollary}

To conclude this section, we highlight that our entire inference procedure can be carried out~in an online, matrix-free manner, with a computational cost of $O((d+m)^2)$, matching that of unconstrained first-order methods \citep{Chen2020Statistical, Zhu2021Online}. Furthermore, in contrast~to~existing~constrained inference procedures \citep{Na2025Statistical, Jiang2025Online},~our~approach~avoids~computing the projection operators and matrix inversions.
Importantly, the random scaling matrix~$V_t$~can be computed in a recursive way. To see this, we rewrite $V_t$ as follows:
\begin{align*}
{V}_{t} &= \frac{1}{t^2}\sum_{i = 1}^t i^2 \begin{pmatrix}
\bar{\bx}_i - \bar{\bx}_t \\
\bar{\blambda}_i - \bar{\blambda}_t
\end{pmatrix}\begin{pmatrix}
\bar{\bx}_i - \bar{\bx}_t \\
\bar{\blambda}_i - \bar{\blambda}_t
\end{pmatrix}^{\top}\\
& = \frac{1}{t^2}\sum_{i = 1}^{t}i^2\begin{pmatrix}
\bar{\bx}_i \\
\bar{\blambda}_i
\end{pmatrix} \begin{pmatrix}
\bar{\bx}_i \\
\bar{\blambda}_i
\end{pmatrix}^\top
- \frac{1}{t^2}\begin{pmatrix}
\sum_{i = 1}^{t}i^2\bar{\bx}_i  \\
\sum_{i = 1}^{t}i^2\bar{\blambda}_i
\end{pmatrix}
\begin{pmatrix}
 \bar{\bx}_t \\
\bar{\blambda}_t
\end{pmatrix}^\top - \frac{1}{t^2}\begin{pmatrix}
\bar{\bx}_t \\
\bar{\blambda}_t
\end{pmatrix} \begin{pmatrix}
\sum_{i = 1}^{t}i^2\bar{\bx}_i  \\
\sum_{i = 1}^{t}i^2\bar{\blambda}_i
\end{pmatrix}^\top
+ \frac{\sum_{i = 1}^{t}i^2}{t^2} 
\begin{pmatrix}
\bar{\bx}_t \\
\bar{\blambda}_t
\end{pmatrix}\begin{pmatrix}
\bar{\bx}_t \\
\bar{\blambda}_t
\end{pmatrix}^\top.
\end{align*}
Each of the four terms above can be easily computed recursively. We formalize the online calculation in the following algorithm.

\begin{algorithm}
\caption{Online computation of the scaling matrix $V_{t}$}\label{alg:2}
\begin{algorithmic}[1]
\State \textbf{Initialize:} set initial values for \( \boldsymbol{s}_0 = (\bx_0, \blambda_0), \bar{\boldsymbol{s}}_1 = \boldsymbol{s}_0, P_1 = \bar{\boldsymbol{s}}_1\bar{\boldsymbol{s}}_1^{\top}, Q_1 = \bar{\boldsymbol{s}}_1, V_1 = 0\);
\For{\( t = 1, 2, \ldots \)}
\State Run AI-SSQP with the update \eqref{equ:update} to obtain $\boldsymbol{s}_{t} = (\bx_t, \blambda_t)$;
\State Compute $\bar{\boldsymbol{s}}_{t+1} = \frac{1}{t+1}\boldsymbol{s}_t + \frac{t}{t+1}\bar{\boldsymbol{s}}_{t}$, $P_{t+1} = (t+1)^2\bar{\boldsymbol{s}}_{t+1}\bar{\boldsymbol{s}}_{t+1}^{\top} + P_{t}$, $Q_{t+1} = (t+1)^2\bar{\boldsymbol{s}}_{t+1} + Q_{t}$;
\State Output $V_{t+1} = \frac{1}{(t+1)^2}P_{t+1} - \frac{1}{(t+1)^2}Q_{t+1}\bar{\boldsymbol{s}}_{t+1}^{\top} - \frac{1}{(t+1)^2}\bar{\boldsymbol{s}}_{t+1} Q_{t+1}^{\top}+ \frac{(t+2)(2t+3)}{6(t+1)}\bar{\boldsymbol{s}}_{t+1}\bar{\boldsymbol{s}}_{t+1}^{\top}$;
\EndFor
\end{algorithmic}
\end{algorithm}

%% file: sec5.tex
\section{Numerical Experiments}\label{sec:5}

In this section, we demonstrate the empirical performance of our proposed inference procedure on both constrained linear and logistic regression problems. We refer our method to as \texttt{AveRS}. We compare it with four other online inference procedures based on asymptotic normality with different covariance estimators. 
Two procedures perform inference using the averaged AI-SSQP~iterates:~one~employs the plug-in covariance estimator (adapted from \cite{Na2025Statistical}) (\texttt{AvePlugIn}) and one employs the batch-means covariance estimator \citep{Zhu2021Online} (\texttt{AveBM}).
The other two procedures perform inference using the last AI-SSQP iterates: one employs the plug-in covariance estimator \citep{Na2025Statistical} (\texttt{LastPlugIn}) and one employs the batch-free covariance estimator \citep{Kuang2025Online} (\texttt{LastBF}). We note that, although the batch-means and batch-free covariance estimators~are~not originally proposed for the AI-SSQP methods considered here, those estimation procedures can still be readily adapted for a reasonable comparison.

We evaluate the performance of all methods by reporting the primal-dual mean absolute error,~the coverage rate and length of~the confidence intervals, as well as the floating-point operations (flops)~per iteration required for inference.

\subsection{Experimental setup}

For the constrained linear regression problem, we consider the model $\xi_{b} = \xi_{a}^\top \boldsymbol{x}^\star + \varepsilon$, where~$\xi = (\xi_a, \xi_b) \in \mathbb{R}^{d+1}$ is the covariate-response pair;  $\varepsilon \sim \mathcal{N}(0, \sigma^2)$ is the Gaussian noise; and $\bx^{\star} \in \mathbb{R}^{d}$ is the model parameter. For this model, we consider the squared loss function $ F(\bx; \xi) = 0.5(\xi_b - \xi_a^\top \boldsymbol{x})^2$. We also enforce both linear constraint $\boldsymbol{A}\bx = \boldsymbol{b}$ and nonlinear constraint \mbox{$\|\bx\|^2 = R^2$}.~Thus,~the~constrained linear regression problem can be summarized as: \vskip-0.1cm
\begin{equation}\label{eq:constrainde_LS_problem}
\min_{\boldsymbol{x} \in \mR^{d}} \mathbb{E}\big[\frac{1}{2}(\xi_b - \xi_a^{\top} \boldsymbol{x})^2\big] \qquad \text{s.t.} \quad  A \boldsymbol{x} = \boldsymbol{b},\quad  \; \|\boldsymbol{x}\|^2 = R^2.
\end{equation}

For the constrained logistic regression problem, we consider the model $P(\xi_{b} | \xi_a) = \frac{\text{exp}(\xi_b \cdot \xi_a^\top \boldsymbol{x}^\star)}{1 + \text{exp}(\xi_b \cdot \xi_a^\top \boldsymbol{x}^\star)}$, where $(\xi_a, \xi_b)\in \mR^d\times \{-1,1\}$ is the covariate-response pair and $\boldsymbol{x}^{\star} \in \mathbb{R}^d$ is the model parameter.~For this model, we use the log loss function $F(\bx; \xi) = \log( 1 + \exp(-\xi_b \cdot \xi_a^\top \boldsymbol{x}))$. Similar~to~the~linear~regression problem, we consider both linear and nonlinear constraints, arriving at
\begin{equation}\label{eq:constrainde_logistic_problem}
\min_{\boldsymbol{x} \in \mR^d} \mathbb{E}\left[ \log \left( 1 + \text{exp} (-\xi_b \cdot \xi_a^{\top} \boldsymbol{x})\right) \right] \qquad \text{s.t.} \quad A \boldsymbol{x} = \boldsymbol{b}, \quad\; \|\boldsymbol{x}\|^2 = R^2.
\end{equation}

\noindent $\bullet$ \textbf{Model parameters setup.} 
For both regression problems, we vary the dimension $d \in \{5, 20, 40\}$ and let the true model parameter  $\bx^{\star} \in \mR^d$ be linearly spanned between $0$ and $1$. For the linear model, we let the noise variance $\sigma^2=1$. For each dimension, we follow the study \cite{Na2025Statistical} and generate the covariate $\xi_a\sim \mN(\0, 5I +  \Sigma_a)$ with three different types of $\Sigma_a$. (i) Identity matrix:~$\Sigma_a = I$; (ii) Toeplitz matrix: $[\Sigma_a]_{i, j} = r^{|i - j|}$ with $r \in \{0.4, 0.5, 0.6\}$;~and~(iii)~\mbox{Equi-correlation}~\mbox{matrix}:~$[\Sigma_a]_{i, i} = 1$ and $[\Sigma_a]_{i, j} = r$ for $i \neq j$, with $r \in \{0.1, 0.2, 0.3\}$. Given $\xi_a$, the response $\xi_b$ is generated~by~following the particular linear or logistic models. For constraints, we let $A \in \mR^{m \times d}$ with each entry~independently drawn from the standard normal distribution. We set $\bb = A\tx$, $R = \|\tx\|$,~and~$m = 1$~for~$d=5$ and $m=3$ otherwise.

\vskip2pt
\noindent$\bullet$ \textbf{Algorithm parameters setup.} 
All five inference procedures are based on the same AI-SSQP iteration sequence under the same algorithmic setup. In our experiment, we vary the~\mbox{sketching}~steps $\tau \in \{20, 40, \infty\}$, where $\tau = \infty$ corresponds to using the exact solver for \eqref{equ:Newton}.~For the \mbox{sketching}~solver,~we~implement the randomized Kaczmarz method \citep{Strohmer2008Randomized}; specifically, the~\mbox{sketching} vectors are drawn from $S \sim \text{Uniform}(\{\be_i\}_{i=1}^{d+m})$ (cf. Section \ref{sec:3.1}). 
We set $\beta_t = 1/(t+1)^{0.501}$, $\chi_t = \beta_t^2$, and choose the random stepsize $\bar{\alpha}_t \sim \text{Uniform}[\beta_t, \eta_t]$ with $\eta_t = \beta_t + \chi_t$.~For the plug-in and batch-free covariance estimators, there are no additional tuning parameters beyond those in the algorithm. For the batch-means covariance estimator, we follow the setup from \cite{Zhu2021Online}, setting the batch size sequence as $a_m  =  \lfloor m^{2/(1-\beta)}\rfloor$ with $\beta = 0.501$ (in their notation).~For~all~\mbox{inference}~\mbox{methods},~we initialize $(\bx_0, \blambda_0)$ as vectors of all ones and run $10^5$ iterations. The nominal coverage probability $1 - p$ is set to $95\%$, and we conduct statistical inference for the coordinate-wise average of the model parameter, $\sum_{i=1}^d \bx_i^\star / d$. To report the confidence interval lengths and coverage rates, we conduct~200 independent runs for each parameter configuration.

\subsection{Numerical results}

We first investigate the consistency of the AI-SSQP iterates. In particular, we compare the estimation error of the averaged iterate $\|\bar{\bx}_t - \bx^{\star}\|$ with that of the last iterate $\|\bx_t - \bx^{\star}\|$. 
We take $d = 20$ and use the Equi-correlation design with $r = 0.2$ as a representative example, and~present~the~\mbox{comparisons}~under varying sketching steps in Figure \ref{fig:error} (similar patterns are observed across all other settings).~From the figure, we observe that the error curves of the averaged iterate decay significantly faster than those of the last iterate.~This is because as shown~in Theorem~\ref{thm:4},~the averaged iterate enjoys~\mbox{$\sqrt{t}$-consistency} 
while the last iterate only enjoys $\sqrt{1/\beta_t} \approx O(t^{1/4})$-consistency (see \eqref{nequ:1}).~\mbox{Furthermore}, the larger~fluctuations in the error curves of the last iterate, compared to the smoother error decay of the averaged iterate, indicate higher variability and thus lower statistical efficiency of the last iterate.

\begin{figure}[!th]
\centering     
\subfigure[$\tau = 20$]{\label{A1}\includegraphics[width=0.328\textwidth]{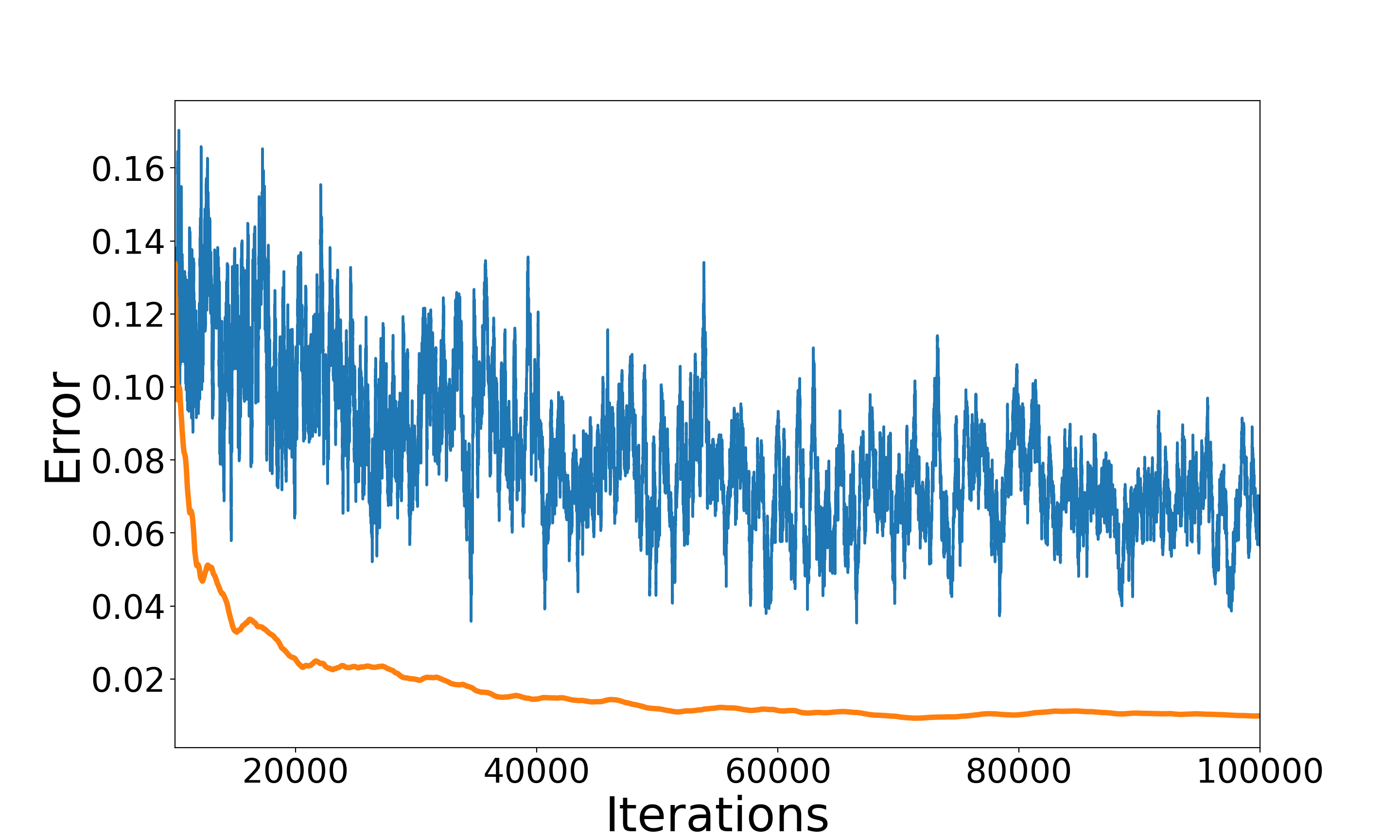}}
\subfigure[$\tau = 40$]{\label{A2}\includegraphics[width=0.328\textwidth]{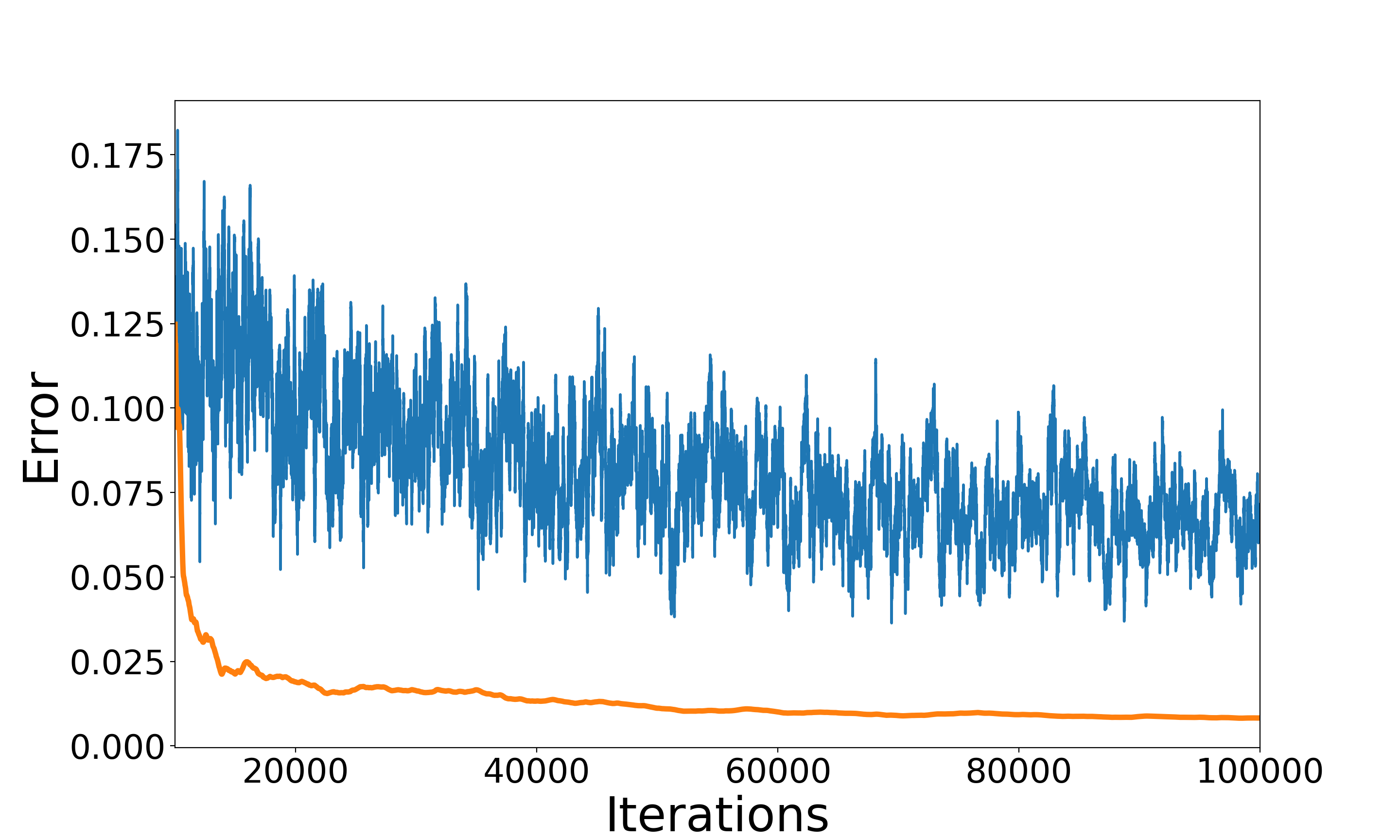}}
\subfigure[$\tau = \infty$]{\label{A3}\includegraphics[width=0.328\textwidth]{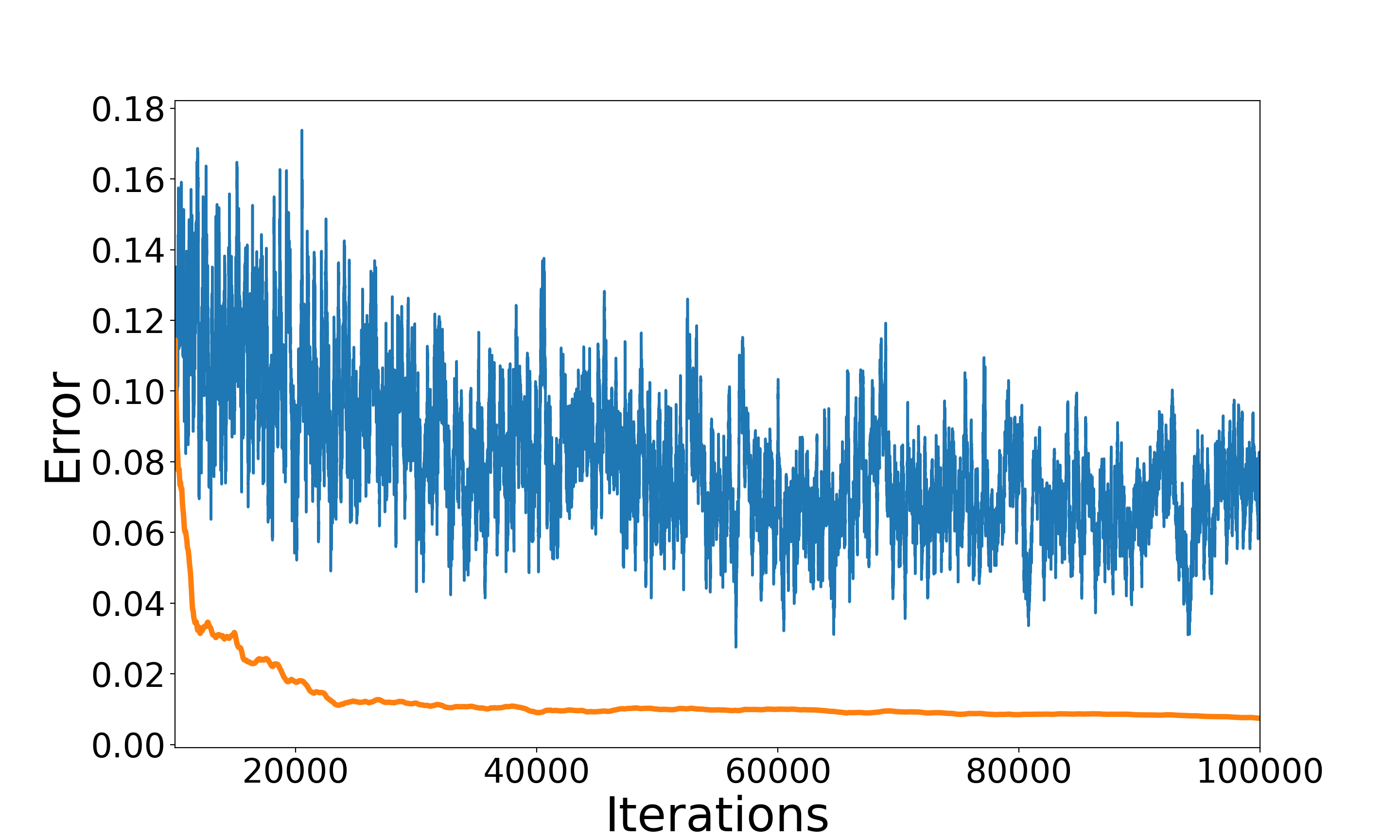}}
\centering{Constrained Linear Regression}

\subfigure[$\tau = 20$]{\label{A4}\includegraphics[width=0.328\textwidth]{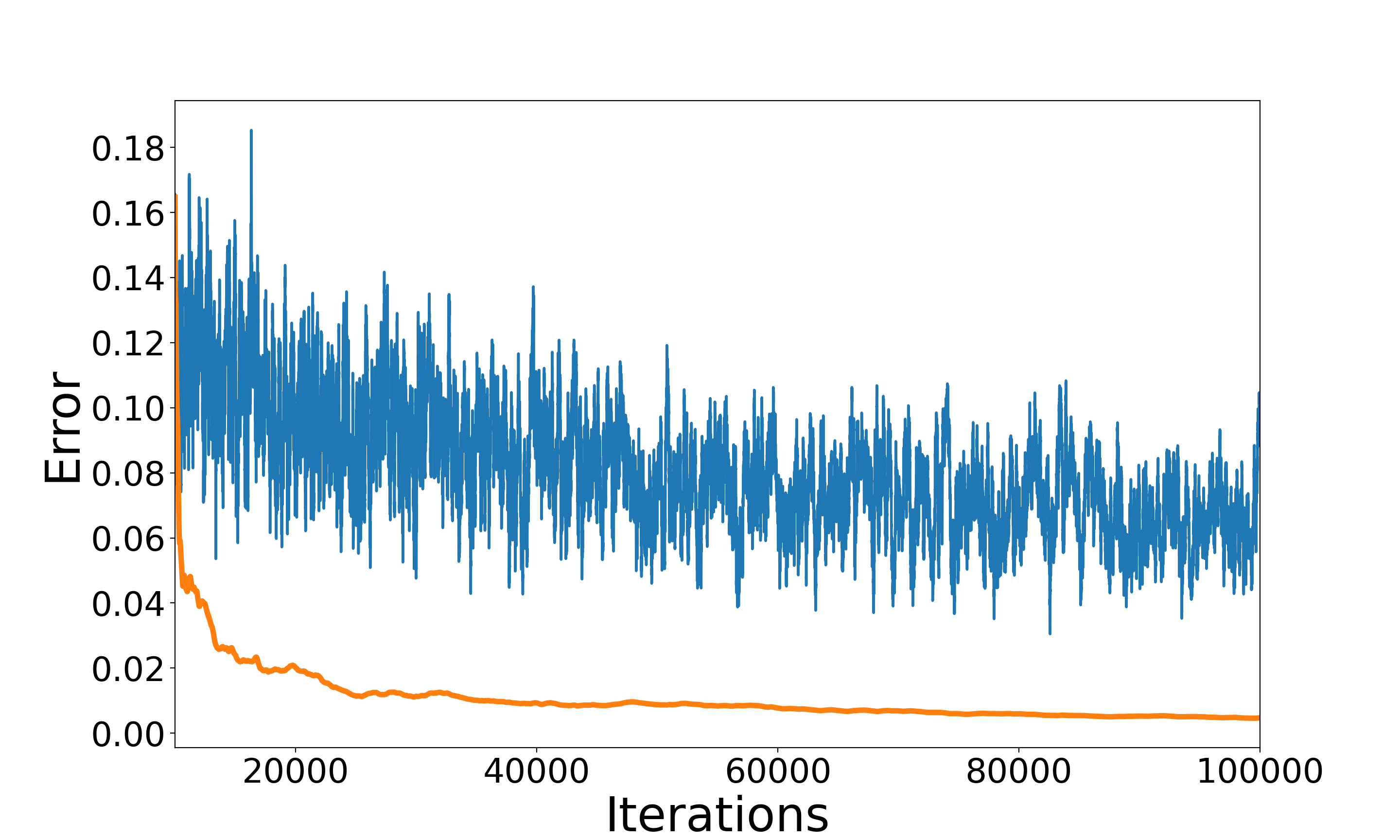}}
\subfigure[$\tau = 40$]{\label{A5}\includegraphics[width=0.328\textwidth]{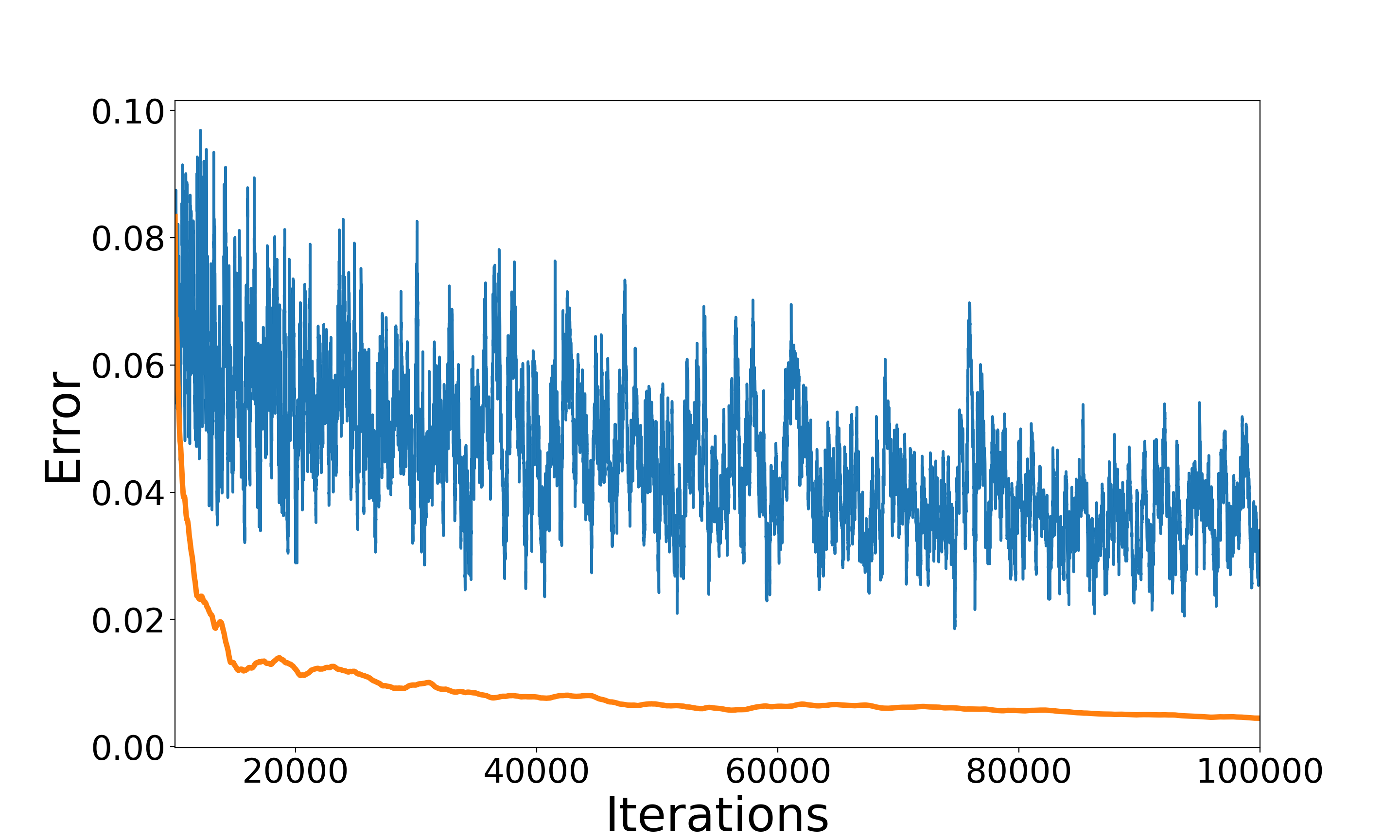}}
\subfigure[$\tau = \infty$]{\label{A6}\includegraphics[width=0.328\textwidth]{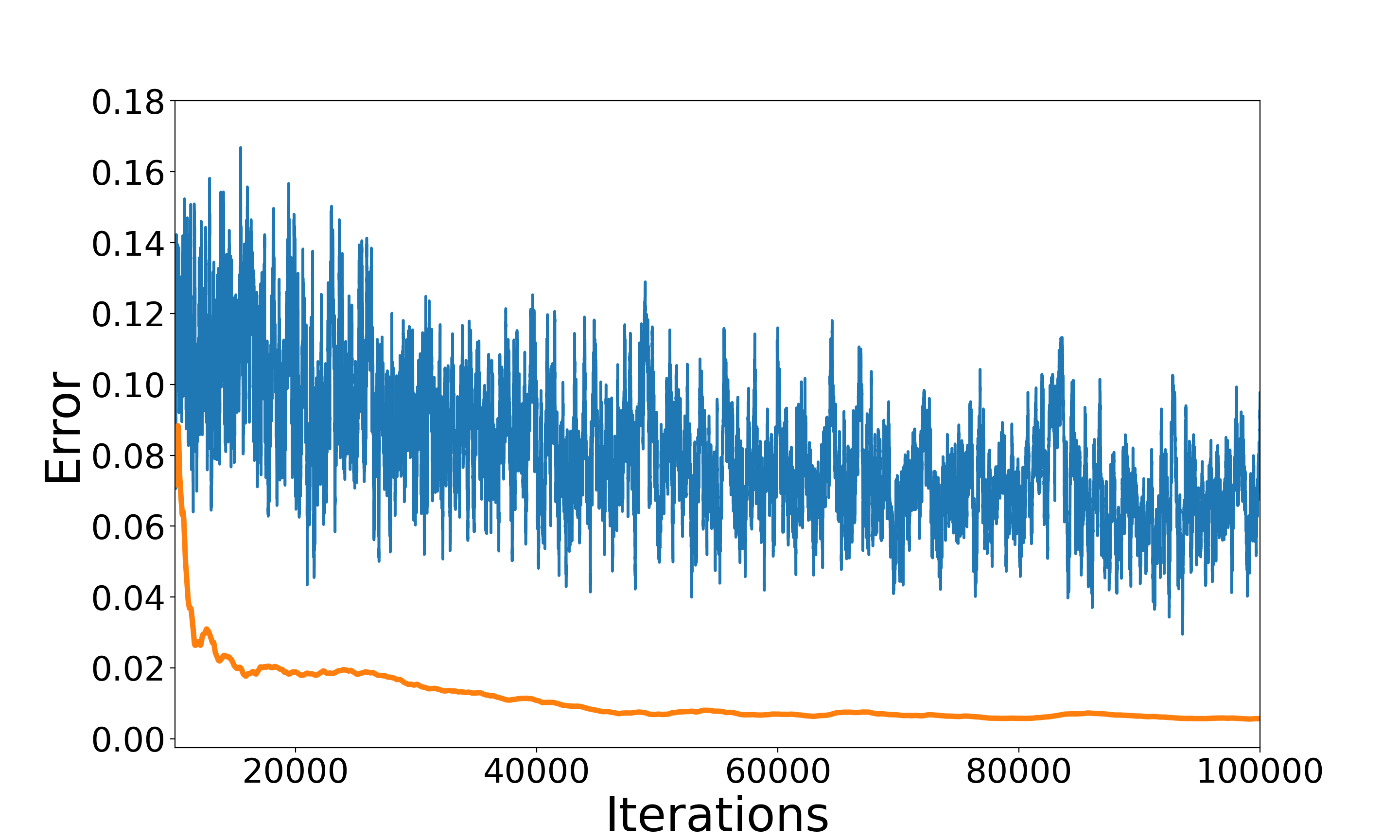}}
\centering{Constrained Logistic Regression}
\vskip5pt
\includegraphics[width=0.5\textwidth]{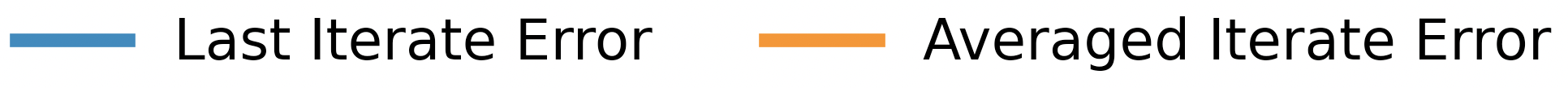}
\caption{\textit{
The averaged iterate error $\|\bar{\bx}_t - \bx^{\star}\|$ and the last iterate error $\|\bx_t - \bx^{\star}\|$  for constrained~linear and logistic regression with varying sketching steps $\tau$. We use $d=20$ and Equi-correlation design with $r=0.2$ as an illustrative example.
}}\label{fig:error}
\end{figure}

Next, we investigate the empirical performance of our random scaling inference procedure alongside four other inference methods that employ different covariance matrix estimators. Among~the~five methods, two are based on the last iterate (\texttt{LastPlugIn}, \texttt{LastBF}), and three are based on the averaged iterate (\texttt{AvePlugIn}, \texttt{AveBM}, \texttt{AveRS} (ours)).
We report the primal-dual mean \mbox{absolute}~\mbox{errors},~$\|(\bx_t-\tx,\blambda_t-\tlambda)\|$ for the last iterate and $\|(\barx_t-\tx, \barlambda_t-\tlambda)\|$ for the averaged iterate, along with~the~coverage rates and lengths of their constructed confidence intervals over 200 independent runs.~We~also~compare the flops per iteration for different inference procedures.
We summarize a subset of the results~in Table \ref{tab:2}, while the complete results and detailed discussions are provided in Appendix~\ref{appen:D}.

\vskip3pt
\noindent $\bullet$ \textbf{Mean absolute error.} 
From Table \ref{tab:2}, we observe that the averaged iterate error improves upon the last iterate error by 0.5-1 order of magnitude across all cases. As explained, this is expected~since the averaged iterate achieves $\sqrt{t}$-consistency for any stepsize control \mbox{sequence}~\mbox{$\beta_t = O(1/t^{\beta})$} with~$\beta \in (0.5,1)$, while the last iterate only achieves $O(1/\sqrt{\beta_t})$-consistency (which is $O(t^{1/4})$ in our experiments).
In fact, it is known that setting $\beta = 0.5$ (i.e., using a larger stepsize) achieves~the~\mbox{optimal}~non-asymptotic algorithmic convergence rate for stochastic methods, while setting $\beta = 1$ (i.e., using a smaller stepsize) achieves the optimal asymptotic statistical efficiency. In practice, however,~setting $\beta = 1$ often leads to significantly slower convergence, making it difficult to observe~optimal~asymptotic normality and to perform reliable inference.
In contrast, using the averaged iterate for inference allows setting $\beta \approx 0.5$ without sacrificing efficiency.

\vskip3pt
\noindent $\bullet$ \textbf{Asymptotic validity.} 
In terms of the averaged coverage rate (Ave Cov), Table \ref{tab:2} shows that our random scaling method consistently achieves promising coverage rates, generally close to 95\%~across the majority of scenarios. One exception arises in constrained linear regression when $d=40$ and the design covariance $\Sigma_a$ is Equi-correlation with $r=0.2$: here, applying the random scaling method to exact SSQP iterates ($\tau=\infty$) leads to an overcoverage of 99\%.~Nevertheless,~on~the~more~computationally efficient inexact, sketched SSQP iterates ($\tau=40$), the coverage rate improves to~96.5\%.$\quad\quad\;$

We compare \texttt{AveRS} with other inference procedures that leverage different covariance estimators. We observe that the plug-in covariance estimators, \texttt{LastPlugIn} and \texttt{AvePlugIn}, perform reasonably well for exact SSQP iterates ($\tau=\infty$) -- albeit with significantly higher computational costs -- but tend to fail for inexact SSQP iterates ($\tau=40$). This is because the plug-in estimators neglect all sketching-related components in the limiting covariance (cf.~\eqref{equ:barXi} and \eqref{equ:Lyapunov}),~\mbox{leading}~to~\mbox{inconsistency}~under~inexact iterations. For example, in both constrained linear and~logistic regression with $d=40$~and~identity design, \texttt{LastPlugIn} and \texttt{AvePlugIn} exhibit undercoverage (below 90\%) for inexact SSQP, whereas \texttt{AveRS} still maintains reasonable performance.
Furthermore, we note that \texttt{LastBF} is generally as competitive as \texttt{AveRS} in terms of coverage rate; however, it suffers from larger mean absolute errors and produces significantly wider confidence intervals. In addition, the existing analysis of the batch-free covariance estimator relies on the strong convexity of the problem, so the theoretical foundations for its performance in our setting~\mbox{remain}~\mbox{unclear}~\citep{Kuang2025Online}. Finally, we observe that the~batch-means \mbox{estimator}~\texttt{AveBM}~exhibits coverage rates far below 95\%. This~poor~\mbox{performance}~may~be attributed to the complexity of our constrained online inference tasks and the nonconvexity of the problem, as existing analyses for this estimator often assume sub-Gaussian noise and/or strongly convex objectives \citep{Zhu2021Online, Jiang2025Online}.

\vskip3pt
\noindent $\bullet$ \textbf{Statistical efficiency.}
In terms of the averaged confidence interval length (Ave Len), we observe that \texttt{AvePlugIn} and \texttt{AveBM} yield the shortest confidence intervals, closely followed by~\texttt{AveRS}.~In~contrast, \texttt{LastPlugIn} and \texttt{LastBF} yield much wider intervals -- by an order of magnitude -- than~the other three inference methods based on the averaged iterate. This pattern aligns with our observations for mean absolute error, further supporting the efficiency of averaged iterates. 
While~\texttt{AvePlugIn}~and \texttt{AveBM} benefit from the asymptotic normality of the averaged iterates, achieving minimax optimal efficiency, we emphasize that \texttt{AveRS} attains a comparable interval length while substantially improving their coverage rates since their covariance estimators are unreliable. This highlights the~strength of the random scaling method with a pivotal asymptotic distribution.

\vskip3pt
\noindent $\bullet$ \textbf{Computational efficiency.} 
In terms of flops per iteration, we observe that \texttt{AveRS}, along with \texttt{AveBM} and \texttt{LastBF}, incurs the lowest computational costs, while the methods with plug-in covariance estimators (\texttt{LastPlugIn} and \texttt{AvePlugIn}) are much more expensive. This is because the plug-in estimators require computing a matrix inverse with flops $O((d+m)^3)$. In addition, \mbox{AI-SSQP} methods are more computationally efficient than exact SSQP methods, as the latter must solve the Newton system that involves another costly matrix inversion. Overall, AI-SSQP~methods~combined with the random scaling technique result in matrix-free inference procedures, making them \mbox{particularly}~promising~for considered second-order methods and matching the cost of projection-based first-order methods.$\quad\quad$

\input{table2}

%% file: table2.tex
\begin{table}[thbp!] 
\vskip-0.3cm
\centering
\setlength{\tabcolsep}{6pt}         
\renewcommand{\arraystretch}{0.98}   
\resizebox{1.00\linewidth}{0.49\textheight}{\begin{tabular}{|c|cc|cc|ccc|ccc|c|}
\hline
&&&&&&&&&&& \\[-2.4ex]
\multirow{2}{*}{$d$}&\multicolumn{2}{c|}{\multirow{2}{*}{Design Cov}}&\multicolumn{2}{c|}{\multirow{2}{*}{Method}}&\multicolumn{3}{c|}{Constrained linear regression}& \multicolumn{3}{c|}{Constrained logistic regression} & \multirow{2}{*}{Flops/iter} \\
\cline{6-11}
&&&&&&&&&&& \\[-2.4ex]
& & & & & MAE $(10^{-2})$ &  \multicolumn{1}{|c|}{Ave Cov $(\%)$} & Ave Len $(10^{-2})$ & MAE $(10^{-2})$ &  \multicolumn{1}{|c|}{Ave Cov $(\%)$} & Ave Len $(10^{-2})$ & \\
\hline
\multirow{30}{*}{20}
&\parbox[t]{2mm}{\multirow{10}{*}{\;\;\;Identity}}& & \multirow{5}{*}{$\tau=\infty$} & 
\multicolumn{1}{|c|}{\texttt{LastPlugIn}} & 9.19 & \multicolumn{1}{|c|}{94.50} & 1.06 & 4.39 & \multicolumn{1}{|c|}{91.00} & 0.43 & \multirow{2}{*}{17632.74}\\
\cline{5-11}
&&&& \multicolumn{1}{|c|}{\texttt{AvePlugIn}} & 1.63 & \multicolumn{1}{|c|}{94.50} & 0.19 & 0.84 &\multicolumn{1}{|c|}{88.00} & 0.08 & \\
\cline{5-12}
&&&& \multicolumn{1}{|c|}{\texttt{LastBF}} & 9.19 &\multicolumn{1}{|c|}{94.50}&1.05& 4.39 &\multicolumn{1}{|c|}{93.50}&0.46& \multirow{3}{*}{14063.88} \\
\cline{5-11}
&&&& \multicolumn{1}{|c|}{\texttt{AveBM}} & \multirow{2}{*}{1.63} &\multicolumn{1}{|c|}{60.00}&0.09& \multirow{2}{*}{0.84} &\multicolumn{1}{|c|}{59.00}&0.05& \\
\cline{5-5}\cline{7-8} \cline{10-11}
&&&& \multicolumn{1}{|c|}{\texttt{AveRS}} &  &\multicolumn{1}{|c|}{97.50} & 0.29 &  & \multicolumn{1}{|c|}{95.00} & 0.14 & \\
\cline{4-12}
&&& \multirow{5}{*}{$\tau=40$} & 
\multicolumn{1}{|c|}{\texttt{LastPlugIn}}  & 7.30  & \multicolumn{1}{|c|}{91.00} & 1.07 & 4.19 & \multicolumn{1}{|c|}{91.00} & 0.43 & \multirow{2}{*}{4768.87}\\
\cline{5-11}
&&&& \multicolumn{1}{|c|}{\texttt{AvePlugIn}} & 2.03 & \multicolumn{1}{|c|}{93.50} & 0.19 & 0.95 & \multicolumn{1}{|c|}{81.00} & 0.08 &  \\
\cline{5-12}
&&&& \multicolumn{1}{|c|}{\texttt{LastBF}} & 7.30 &\multicolumn{1}{|c|}{92.50}&1.09& 4.19 & \multicolumn{1}{|c|}{95.00}& 0.47 & \multirow{3}{*}{\textbf{1200.00}}\\
\cline{5-11}
&&&& \multicolumn{1}{|c|}{\texttt{AveBM}} & \multirow{2}{*}{2.03} & \multicolumn{1}{|c|}{65.00} & 0.11 & \multirow{2}{*}{0.95} & \multicolumn{1}{|c|}{66.00} & 0.05 &   \\
\cline{5-5}\cline{7-8} \cline{10-11}
&&&& \multicolumn{1}{|c|}{\texttt{AveRS}}  &  & \multicolumn{1}{|c|}{\textbf{\red96.00}} & \textbf{\red0.36} &  & \multicolumn{1}{|c|}{\textbf{\red95.50}} & \textbf{\red0.17} &   \\
\cline{2-12}
& \multicolumn{11}{c|}{}\\[-2ex]
\cline{2-12}
&\parbox[t]{2mm}{\multirow{10}{*}{\;\;\;\shortstack{Toeplitz\\$r=0.5$}}}& & \multirow{5}{*}{$\tau=\infty$} & 
\multicolumn{1}{|c|}{\texttt{LastPlugIn}} & 8.99 & \multicolumn{1}{|c|}{96.50} & 1.46 & 4.16 & \multicolumn{1}{|c|}{91.50} & 0.47 & \multirow{2}{*}{17632.74} \\
\cline{5-11}
&&&& \multicolumn{1}{|c|}{\texttt{AvePlugIn}} & 1.67 & \multicolumn{1}{|c|}{93.50} & 0.26 & 0.78 & \multicolumn{1}{|c|}{\textbf{\red93.50}} & \textbf{\red0.08} &  \\
\cline{5-12}
&&&& \multicolumn{1}{|c|}{\texttt{LastBF}} & 8.99 & \multicolumn{1}{|c|}{96.50} &1.45 & 4.16 & \multicolumn{1}{|c|}{93.00} &0.50& \multirow{3}{*}{14063.88} \\
\cline{5-11}
&&&& \multicolumn{1}{|c|}{\texttt{AveBM}} & \multirow{2}{*}{1.67} & \multicolumn{1}{|c|}{63.00} &0.13 & \multirow{2}{*}{0.78} & \multicolumn{1}{|c|}{58.50}&0.05& \\
\cline{5-5}\cline{7-8} \cline{10-11}
&&&& \multicolumn{1}{|c|}{\texttt{AveRS}} &  & \multicolumn{1}{|c|}{95.50} & 0.42 &  & \multicolumn{1}{|c|}{98.00} & 0.15 & \\
\cline{4-12}
&&& \multirow{5}{*}{$\tau=40$} & 
\multicolumn{1}{|c|}{\texttt{LastPlugIn}}  & 7.31 & \multicolumn{1}{|c|}{99.00} & 1.46 & 4.01 & \multicolumn{1}{|c|}{92.00} & 0.47 & \multirow{2}{*}{4768.87} \\
\cline{5-11}
&&&& \multicolumn{1}{|c|}{\texttt{AvePlugIn}} & 2.03 & \multicolumn{1}{|c|}{93.00} & 0.26 & 0.92 & \multicolumn{1}{|c|}{86.50} & 0.08 &  \\
\cline{5-12}
&&&& \multicolumn{1}{|c|}{\texttt{LastBF}} & 7.31 & \multicolumn{1}{|c|}{96.00} &1.12& 4.01 & \multicolumn{1}{|c|}{\textbf{\red95.50}} &\textbf{\red0.53}& \multirow{3}{*}{\textbf{1200.00}}\\
\cline{5-11}
&&&& \multicolumn{1}{|c|}{\texttt{AveBM}} & \multirow{2}{*}{2.03} & \multicolumn{1}{|c|}{56.00} & 0.12 & \multirow{2}{*}{0.92} & \multicolumn{1}{|c|}{66.00} & 0.06 & \\
\cline{5-5}\cline{7-8} \cline{10-11}
&&&& \multicolumn{1}{|c|}{\texttt{AveRS}}  &  & \multicolumn{1}{|c|}{\textbf{\red94.50}} & \textbf{\red0.39} &  & \multicolumn{1}{|c|}{98.00} & 0.19 & \\
\cline{2-12}
& \multicolumn{11}{c|}{}\\[-2ex]
\cline{2-12}
&\parbox[t]{2mm}{\multirow{10}{*}{\;\;\shortstack{Equi-Corr\\$r=0.2$}}}& & \multirow{5}{*}{$\tau=\infty$} & 
\multicolumn{1}{|c|}{\texttt{LastPlugIn}} & 8.92 & \multicolumn{1}{|c|}{95.00} & 1.40 & 4.07 & \multicolumn{1}{|c|}{94.50} & 0.43 & \multirow{2}{*}{17632.74} \\
\cline{5-11}
&&&& \multicolumn{1}{|c|}{\texttt{AvePlugIn}} & 1.59 & \multicolumn{1}{|c|}{96.00} & 0.25 & 0.76 & \multicolumn{1}{|c|}{92.50} & 0.08 &  \\
\cline{5-12}
&&&& \multicolumn{1}{|c|}{\texttt{LastBF}} & 8.92 & \multicolumn{1}{|c|}{94.00} & 1.36 & 4.07 & \multicolumn{1}{|c|}{94.00} & 0.45 & \multirow{3}{*}{14063.88} \\
\cline{5-11}
&&&& \multicolumn{1}{|c|}{\texttt{AveBM}} & \multirow{2}{*}{1.59} & \multicolumn{1}{|c|}{60.50} & 0.12 & \multirow{2}{*}{0.76} & \multicolumn{1}{|c|}{64.00} & 0.04 &   \\
\cline{5-5}\cline{7-8} \cline{10-11}
&&&& \multicolumn{1}{|c|}{\texttt{AveRS}} &  & \multicolumn{1}{|c|}{98.00} & 0.39 &  & \multicolumn{1}{|c|}{97.50} & 0.14 &  \\
\cline{4-12}
&&& \multirow{5}{*}{$\tau=40$} & 
\multicolumn{1}{|c|}{\texttt{LastPlugIn}}  & 7.43 & \multicolumn{1}{|c|}{99.50} & 1.40 & 3.81 & \multicolumn{1}{|c|}{\textbf{\red94.00}} & \textbf{\red0.42} &   \multirow{2}{*}{4768.87} \\
\cline{5-11}
&&&& \multicolumn{1}{|c|}{\texttt{AvePlugIn}}  &  1.93 & \multicolumn{1}{|c|}{89.50} & 0.25 & 0.87 & \multicolumn{1}{|c|}{92.00} & 0.08 &  \\
\cline{5-12}
&&&& \multicolumn{1}{|c|}{\texttt{LastBF}} & 7.43 & \multicolumn{1}{|c|}{94.50} & 1.15 & 3.81 & \multicolumn{1}{|c|}{95.00} & 0.46 & \multirow{3}{*}{\textbf{1200.00}} \\
\cline{5-11}
&&&& \multicolumn{1}{|c|}{\texttt{AveBM}} & \multirow{2}{*}{1.93} & \multicolumn{1}{|c|}{57.00} & 0.13 & \multirow{2}{*}{0.87} & \multicolumn{1}{|c|}{70.00} & 0.05 &  \\
\cline{5-5}\cline{7-8} \cline{10-11}
&&&& \multicolumn{1}{|c|}{\texttt{AveRS}}  &  & \multicolumn{1}{|c|}{\textbf{\red94.50}} & \textbf{\red0.39} &  & \multicolumn{1}{|c|}{98.50} & 0.16 &  \\
\cline{1-12}
\multicolumn{12}{|c|}{}\\[-2ex]
\cline{1-12}
\multirow{30}{*}{40}&\parbox[t]{2mm}{\multirow{10}{*}{\;\;\;Identity}}& & \multirow{5}{*}{$\tau=\infty$} & 
\multicolumn{1}{|c|}{\texttt{LastPlugIn}} & 10.68 & \multicolumn{1}{|c|}{97.00} & 0.62 & 5.63 & \multicolumn{1}{|c|}{95.00} & 0.28 & \multirow{2}{*}{106068.43}\\
\cline{5-11}
&&&& \multicolumn{1}{|c|}{\texttt{AvePlugIn}} & 1.93 & \multicolumn{1}{|c|}{94.00} & 0.11 & 1.07 & \multicolumn{1}{|c|}{90.00} & 0.05 &  \\
\cline{5-12}
&&&& \multicolumn{1}{|c|}{\texttt{LastBF}} & 10.68 & \multicolumn{1}{|c|}{95.50} &0.61 & 5.63 & \multicolumn{1}{|c|}{95.50} & 0.29 & \multirow{3}{*}{85975.21} \\
\cline{5-11}
&&&& \multicolumn{1}{|c|}{\texttt{AveBM}} & \multirow{2}{*}{1.93} & \multicolumn{1}{|c|}{60.00} & 0.06 & \multirow{2}{*}{1.07} & \multicolumn{1}{|c|}{59.00} & 0.03 & \\
\cline{5-5}\cline{7-8} \cline{10-11}
&&&& \multicolumn{1}{|c|}{\texttt{AveRS}} &  & \multicolumn{1}{|c|}{\textbf{\red96.00}} & \textbf{\red0.18} &  & \multicolumn{1}{|c|}{\textbf{\red96.00}} & \textbf{\red0.09} &   \\
\cline{4-12}
&&& \multirow{5}{*}{$\tau=40$} & 
\multicolumn{1}{|c|}{\texttt{LastPlugIn}}  & 10.19& \multicolumn{1}{|c|}{86.00} & 0.62 & 5.16 & \multicolumn{1}{|c|}{89.50} & 0.28 & \multirow{2}{*}{22645.27}\\
\cline{5-11}
&&&& \multicolumn{1}{|c|}{\texttt{AvePlugIn}} & 2.60& \multicolumn{1}{|c|}{81.00} & 0.11 & 1.43 & \multicolumn{1}{|c|}{78.00} & 0.05 & \\
\cline{5-12}
&&&& \multicolumn{1}{|c|}{\texttt{LastBF}} & 10.19& \multicolumn{1}{|c|}{93.50} & 0.76 & 5.16 & \multicolumn{1}{|c|}{91.50} & 0.30 & \multirow{3}{*}{\textbf{2552.05}} \\
\cline{5-11}
&&&& \multicolumn{1}{|c|}{\texttt{AveBM}} & \multirow{2}{*}{2.60} & \multicolumn{1}{|c|}{64.00} & 0.08  & \multirow{2}{*}{1.43} & \multicolumn{1}{|c|}{65.00} & 0.04 & \\
\cline{5-5}\cline{7-8} \cline{10-11}
&&&& \multicolumn{1}{|c|}{\texttt{AveRS}} &  & \multicolumn{1}{|c|}{\textbf{\red97.00}} & \textbf{\red0.26} &  & \multicolumn{1}{|c|}{\textbf{\red96.50}} & \textbf{\red0.13} & \\
\cline{2-12}
& \multicolumn{11}{c|}{}\\[-2ex]
\cline{2-12}
&\parbox[t]{2mm}{\multirow{10}{*}{\;\;\;\shortstack{Toeplitz\\$r=0.5$}}}& & \multirow{5}{*}{$\tau=\infty$} & 
\multicolumn{1}{|c|}{\texttt{LastPlugIn}} & 10.73 & \multicolumn{1}{|c|}{94.00} & 0.59 & 5.28& \multicolumn{1}{|c|}{94.00} & 0.30 & \multirow{2}{*}{10608.43} \\
\cline{5-11}
&&&& \multicolumn{1}{|c|}{\texttt{AvePlugIn}} &1.96 & \multicolumn{1}{|c|}{95.50} & 0.11 & 0.99& \multicolumn{1}{|c|}{90.50} & 0.05 &  \\
\cline{5-12}
&&&& \multicolumn{1}{|c|}{\texttt{LastBF}} & 10.73 & \multicolumn{1}{|c|}{94.00} & 0.58 & 5.28 & \multicolumn{1}{|c|}{94.00} & 0.31 & \multirow{3}{*}{85975.21}  \\
\cline{5-11}
&&&& \multicolumn{1}{|c|}{\texttt{AveBM}} & \multirow{2}{*}{1.96} & \multicolumn{1}{|c|}{66.50} & 0.06 & \multirow{2}{*}{0.99} &\multicolumn{1}{|c|}{68.00}&0.03& \\
\cline{5-5}\cline{7-8} \cline{10-11}
&&&& \multicolumn{1}{|c|}{\texttt{AveRS}} &  & \multicolumn{1}{|c|}{\textbf{\red94.50}} & \textbf{\red0.18} &  &\multicolumn{1}{|c|}{\textbf{\red96.00}} & \textbf{\red0.10} & \\
\cline{4-12}
&&& \multirow{5}{*}{$\tau=40$} & 
\multicolumn{1}{|c|}{\texttt{LastPlugIn}}  &10.15 & \multicolumn{1}{|c|}{84.50} & 0.59 & 4.84 & \multicolumn{1}{|c|}{95.00} & 0.30 &  \multirow{2}{*}{22645.27} \\
\cline{5-11}
&&&& \multicolumn{1}{|c|}{\texttt{AvePlugIn}} & 2.67& \multicolumn{1}{|c|}{80.50} & 0.10 & 1.38 & \multicolumn{1}{|c|}{79.00} & 0.05 &  \\
\cline{5-12}
&&&& \multicolumn{1}{|c|}{\texttt{LastBF}} & 10.15& \multicolumn{1}{|c|}{95.50} & 0.77 & 4.84 & \multicolumn{1}{|c|}{92.50} &0.30& \multirow{3}{*}{\textbf{2552.05}}\\
\cline{5-11}
&&&& \multicolumn{1}{|c|}{\texttt{AveBM}} & \multirow{2}{*}{2.67} & \multicolumn{1}{|c|}{60.50} &0.07 & \multirow{2}{*}{1.38} & \multicolumn{1}{|c|}{62.00} &0.04&  \\
\cline{5-5}\cline{7-8} \cline{10-11}
&&&& \multicolumn{1}{|c|}{\texttt{AveRS}}  &  & \multicolumn{1}{|c|}{\textbf{\red97.00}} & \textbf{\red0.23} &  & \multicolumn{1}{|c|}{\textbf{\red95.50}} & \textbf{\red0.12} &  \\
\cline{2-12}
& \multicolumn{11}{c|}{}\\[-2ex]
\cline{2-12}
&\parbox[t]{2mm}{\multirow{10}{*}{\;\;\shortstack{Equi-Corr\\$r=0.2$}}}& & \multirow{5}{*}{$\tau=\infty$} & 
\multicolumn{1}{|c|}{\texttt{LastPlugIn}} & 10.95 & \multicolumn{1}{|c|}{94.50} & 0.66 & 4.85 & \multicolumn{1}{|c|}{95.00} & 0.27 &  \multirow{2}{*}{10608.44}\\
\cline{5-11}
&&&& \multicolumn{1}{|c|}{\texttt{AvePlugIn}} & 1.97 & \multicolumn{1}{|c|}{\textbf{\red95.50}} & \textbf{\red0.12} & 0.90 & \multicolumn{1}{|c|}{92.50} & 0.05 &   \\
\cline{5-12}
&&&& \multicolumn{1}{|c|}{\texttt{LastBF}} & 10.95 & \multicolumn{1}{|c|}{94.50} &0.65 & 4.85 & \multicolumn{1}{|c|}{95.00}&0.28& \multirow{3}{*}{85975.21} \\
\cline{5-11}
&&&& \multicolumn{1}{|c|}{\texttt{AveBM}} & \multirow{2}{*}{1.97} & \multicolumn{1}{|c|}{70.00} &0.06 & \multirow{2}{*}{0.90} &\multicolumn{1}{|c|}{69.00} &0.03& \\
\cline{5-5}\cline{7-8} \cline{10-11}
&&&& \multicolumn{1}{|c|}{\texttt{AveRS}} &  & \multicolumn{1}{|c|}{99.00} & 0.19 &  & \multicolumn{1}{|c|}{97.00} & 0.09 &  \\
\cline{4-12}
&&& \multirow{5}{*}{$\tau=40$} & 
\multicolumn{1}{|c|}{\texttt{LastPlugIn}}  & 10.31 & \multicolumn{1}{|c|}{92.00} & 0.66 & 4.52 & \multicolumn{1}{|c|}{94.50} & 0.27 & \multirow{2}{*}{22645.27}  \\
\cline{5-11}
&&&& \multicolumn{1}{|c|}{\texttt{AvePlugIn}} & 2.83 & \multicolumn{1}{|c|}{86.00} & 0.12 & 1.21 & \multicolumn{1}{|c|}{80.50} & 0.05 & \\
\cline{5-12}
&&&& \multicolumn{1}{|c|}{\texttt{LastBF}} & 10.31 & \multicolumn{1}{|c|}{94.50} &0.78 & 4.52 & \multicolumn{1}{|c|}{94.50} & 0.27& \multirow{3}{*}{\textbf{2552.05}}  \\
\cline{5-11}
&&&& \multicolumn{1}{|c|}{\texttt{AveBM}} & \multirow{2}{*}{2.83} & \multicolumn{1}{|c|}{65.00} &0.08 & \multirow{2}{*}{1.21} & \multicolumn{1}{|c|}{63.50} &0.03&  \\
\cline{5-5}\cline{7-8} \cline{10-11}
&&&& \multicolumn{1}{|c|}{\texttt{AveRS}}  &  & \multicolumn{1}{|c|}{\textbf{\red96.50}} & \textbf{\red0.25} &  & \multicolumn{1}{|c|}{\textbf{\red96.50}} & \textbf{\red0.11} & \\
\hline
\end{tabular}}
\vspace{-0.2cm}
\caption{\textit{
A subset of comparison results for different inference methods on constrained linear and logistic regression problems.~For each combination of dimension and design covariance, we highlight the coverage rate and confidence interval length of a method if no other method achieves a comparably good coverage rate with a shorter confidence interval. If the coverage rates and confidence lengths are comparable, we highlight the method with fewer flops per iteration.
}}\label{tab:2}
\vspace{-0.2cm}
\end{table}

%% file: sec6.tex
\section{Conclusion and Future Work}\label{sec:6}

In this paper, we developed an online inference procedure for constrained stochastic optimization problems by leveraging a numerical method called Adaptive Inexact Stochastic Sequential Quadratic Programming (AI-SSQP). At each step, the method employs a sketching solver to approximately solve the expensive quadratic subproblem and selects a proper adaptive random stepsize. We established~the asymptotic normality of the averaged SSQP iterate $(\bar{\bx}_t, \bar{\blambda}_t)$ and demonstrated that the averaged~iterate enjoys better statistical efficiency than the last iterate, as evidenced by a smaller limiting covariance matrix. This limiting covariance reduces to the minimax optimal covariance when sketching is degraded, and, when sketching is used, remains close to the optimal covariance within a radius that~decays exponentially fast in terms of the number of sketching steps.
Furthermore, we analyzed the~partial sum process of SSQP,~$\frac{1}{\sqrt{t}}\sum_{i = 0}^{\lfloor rt\rfloor-1} (\bx_i - \bx^{\star}, \blambda_i - \blambda^{\star})$, and proved that it converges in distribution to a standard Brownian motion. Based on this result, we proposed a test statistic by studentizing~the estimation error $\sqrt{t}\cdot (\bar{\bx}_t - \bx^{\star}, \bar{\blambda}_t - \blambda^{\star})$ using a random scaling matrix $V_t$. We showed that the resulting test statistic is \mbox{asymptotically}~\mbox{pivotal},~in~the sense that its limiting distribution is free of any~unknown parameters. This enables the construction of valid confidence intervals for the local solution $(\bx^{\star}, \blambda^{\star})$. Our random scaling method is fully online~and~\mbox{matrix-free},~\mbox{significantly}~\mbox{reducing}~both~the computational and memory costs that typically burden second-order methods. In particular, our method matches the \mbox{computational}~\mbox{efficiency}~of~state-of-the-art unconstrained first-order methods, making it particularly suitable for large-scale, streaming data settings.

For future research, it is of interest to design online inference procedures for inequality-constrained stochastic optimization problems. Many methods of solving inequality-constrained problems~asymptotically reduce to solving equality-constrained problems, suggesting that the random scaling analysis, along with the stopping-time localization technique, can still be applicable \citep{Na2023Inequality}.~In addition, recent works on online inference have extended the data sampling scheme~for~\mbox{first-order}~methods from i.i.d. data to Markovian data \citep{Li2023Online, Roy2023Online}. A \mbox{promising}~direction, therefore, is to adapt second-order methods to the Markovian sampling settings and establish the corresponding online inference theory. Finally, extending constrained online inference to high-dimensional settings, where the dimensionality grows with the sample size, is~also~an~important~avenue for future exploration.

%% file: appendix.tex
\section{Preparation Lemmas and Preliminaries}\label{appen:A}

We present some preparation results that will be used in our later proofs. Throughout the sections,~we use the following filtration to clearly identify the source of randomness:
\begin{equation*}
\mF_{t-0.5} =  \sigma(\{\xi_i, \{S_{i, j}\}_{j}, \bar{\alpha}_i\}_{i=0}^{t-1}\cup \xi_t),\quad\quad\quad \mF_t = \sigma(\{\xi_i, \{S_{i, j}\}_{j}, \bar{\alpha}_i\}_{i=0}^{t}).
\end{equation*}
For consistency, $\mF_{-1}$ denotes the trivial $\sigma$-algebra. We recall that $\Tilde{\bz}_{t} = (\Tilde{\Delta}\bx_t, \Tilde{\Delta}\blambda_t)$ is the~\mbox{exact}~stochastic Newton direction in \eqref{equ:Newton} and $\bz_{t, \tau} = (\barDelta\bx_t, \barDelta\blambda_t)$ is the sketched Newton direction in \eqref{def:zttau}.~For~the~sake of brevity, we define the following projection matrices:
\begin{equation}\label{def:C_t}
C_{t,j} \coloneqq I - K_tS_{t,j}(S_{t,j}^\top K_t^2S_{t,j})^\dagger S_{t,j}^\top K_t,\quad\quad \tC_t \coloneqq \prod_{j=0}^{\tau-1} C_{t,j}, \quad\quad C_t \coloneqq \mE[\tC_t\mid \mF_{t-1}].
\end{equation}
By the sketching update \eqref{equ:pseduo}, \eqref{equ:Newton}, and the setup of $\bz_{t,0} = \0$, we have
\begin{equation}\label{equ:z:recur}
\bz_{t,\tau} - \tbz_t = C_{t,\tau-1}(\bz_{t,\tau-1}-\tbz_t) =  \rbr{\prod_{j = 0}^{\tau-1}C_{t,j}}(\bz_{t,0}-\tbz_t) = -\tC_t \tbz_t\;\quad  \Longrightarrow \;\quad \bz_{t,\tau} = (I-\tC_t)\tbz_t.
\end{equation}
We let $\rho \coloneqq 1-\gamma_{S}$, where $\gamma_S$ is introduced in Assumption \ref{ass:3}. The first lemma provides guarantees~on the sketching solver.

\begin{lemma}\label{lem:1}

Under Assumptions \ref{ass:1}, \ref{ass:2}, \ref{ass:3}, we have for all $t\geq 0$,
\begin{enumerate}[topsep=2pt,label=(\alph*):]
\setlength\itemsep{0.0em}
\item $0\leq \rho<1$.
\item $\mE[\bz_{t,\tau} \mid \mF_{t-0.5}] =  (I-C_t)\tbz_t$ and $\mE[\|\bz_{t,\tau}\|^2 \mid \mF_{t-1}] \leq \Upsilon_z$ for some constant $\Upsilon_z>0$.
\item $\|C_t\|\leq \rho^\tau$ and $\|C_t-C^\star\| \leq 2\tau\Upsilon_S\|K_t - K^\star\|/\sigma_{\min}(K^\star)$, where $C^\star$ is defined in \eqref{equ:tC}.
\end{enumerate}	
\end{lemma}

\begin{proof}
(a) By Assumption \ref{ass:3}, we know $\gamma_S>0$; thus, $\rho<1$. Further, $\gamma_S \leq \mE[\|K_tS(S^{\top}K_t^2S)^\dagger S^{\top}K_t \| \mid \bx_t, \blambda_t]\leq 1$; thus, $\rho\geq0$.~(b) The first result follows directly from \eqref{equ:z:recur} and the independence between randomness of $\xi_t$ and $\{S_{t,j}\}_j$; the second result is from \mbox{\cite[(D.10)]{Na2025Statistical}}.~(c)~The~first result follows from the independence among $\{S_{t,j}\}_j$,~and~the second result is from \cite[Corollary 5.4]{Na2025Statistical}. 
\end{proof}

The next lemma provides bounds and convergence results for the Hessian matrix $K_t$ and the projection matrix $C_t$. 

\begin{lemma}\label{lem:K_tconvergence}

Under Assumptions \ref{ass:1}, \ref{ass:2}, \ref{ass:3}, and assuming $(\bx_t, \blambda_t)\stackrel{a.s.}{\longrightarrow}(\tx,\tlambda)$, we have
\begin{enumerate}[topsep=2pt,label=(\alph*):]
\setlength\itemsep{0.0em}
\item The Hessian modification $\Delta_t=0$ for large $t$, and $K_t \rightarrow K^{\star}$, $C_t \rightarrow C^{\star}$ as $t \rightarrow \infty$ almost surely.
\item There exists a constant $\Upsilon_K>0$ such that $\|K_t^{-1}\|\leq \Upsilon_K$, $\forall t\geq 0$.
\end{enumerate}
\end{lemma}

\begin{proof}
Recall $K_t$ is defined in \eqref{equ:Newton} and $K^\star$ in \eqref{equ:Omega}. The result (a) is from \cite[Lemma E.4]{Na2025Statistical} and Lemma \ref{lem:1}(c), while the result (b) is from \cite[Lemma 5.1]{Na2024Fast}.
\end{proof}

We then decompose the recursion of the AI-SSQP method into several components, including the martingale difference component, the stepsize adaptivity component, and other higher-order~components. These components will be analyzed separately when establishing asymptotic normality.

\begin{lemma}\label{lem:4.1}

Let $\bomega_t = (\boldsymbol{x}_t - \boldsymbol{x}^{\star}, \boldsymbol{\lambda}_t - \boldsymbol{\lambda}^{\star})$. The error recursion of AI-SSQP can be decomposed as
\begin{equation}\label{eq:lem_4.1_recursion}
\bomega_{t+1} =  \cbr{I - \beta_t(I-C^\star)} \bomega_{t} + \beta_t(\btheta_t + \bdelta_t) +  \rbr{\baralpha_t - \beta_t} \bz_{t,\tau},
\end{equation}
where 
\begin{subequations}\label{rec:def}
\begin{flalign}
\btheta_{t} & = -(I-C_t)K_t^{-1}(\bnabla\mL_t - \nabla\mL_t) + \cbr{\bz_{t,\tau} - (I - C_t)\Tilde{\bz}_{t}}, \label{rec:def:b}\\ 
\bdelta_{t} & = -(I - C_t)\cbr{(K^\star)^{-1}\bpsi_t + \cbr{K_t^{-1} - (K^\star)^{-1}}\nabla\mL_t} + (C_t-C^\star)\bomega_t,\label{rec:def:c}\\ 
\bpsi_t & = \nabla\mL_t - K^\star \bomega_t, \label{rec:def:psi} 
\end{flalign}
\end{subequations}	
and $C_t$ is defined in \eqref{def:C_t}, $C^\star$ in \eqref{equ:tC}, $K_t$ in \eqref{equ:Newton}, and $K^\star$ in \eqref{equ:Omega}. Furthermore, under Assumptions~\ref{ass:1}, \ref{ass:2}, \ref{ass:3}, and supposing $(\bx_t,\blambda_t)\rightarrow(\tx,\tlambda)$, we know $\btheta_t$ is a martingale difference~with~$\mE[\btheta_t\mid \mF_{t-1}] = \0$ and
\begin{equation}\label{nequ:2}
\mE[\btheta_t\btheta_t^{\top} \mid \mF_{t-1}] \stackrel{a.s.}{\longrightarrow} \mE[(I-\tC^\star)\Xi^{\star}(I-\tC^\star)^{\top}]\quad\; \text{ as }\; t\rightarrow\infty.
\end{equation}
\end{lemma}

\begin{proof}
	
From the update rule of AI-SSQP in \eqref{equ:update}, we have
\begin{align}\label{nequ:3}
\bomega_{t+1} & =\bomega_t + \baralpha_t\bz_{t,\tau} = \bomega_t + \beta_t\bz_{t,\tau} + \rbr{\baralpha_t -\beta_t} \bz_{t,\tau} \nonumber\\
& = \bomega_t +\beta_t(I - C_t)\tbz_t + \beta_t\cbr{\bz_{t,\tau} - (I - C_t)\tbz_t} + \rbr{\baralpha_t - \beta_t}\bz_{t,\tau} \nonumber\\
& \stackrel{\mathclap{\eqref{equ:Newton}}}{=} \bomega_t - \beta_t(I - C_t)K_t^{-1}\bnabla\mL_t + \beta_t\cbr{\bz_{t,\tau} - (I - C_t)\tbz_t} + \rbr{\baralpha_t - \beta_t} \bz_{t,\tau} \nonumber\\
& = \bomega_t - \beta_t(I - C_t)K_t^{-1}\nabla\mL_t - \beta_t(I - C_t)K_t^{-1}(\bnabla\mL_t - \nabla\mL_t) + \beta_t\cbr{\bz_{t,\tau}- (I - C_t)\tbz_t}  + \rbr{\baralpha_t - \beta_t} \bz_{t,\tau} \nonumber\\
& \stackrel{\mathclap{\eqref{rec:def:b}}}{=}\; \;
\bomega_t - \beta_t(I - C_t)K_t^{-1}\nabla\mL_t  + \beta_t\btheta_t+ \rbr{\baralpha_t - \beta_t} \bz_{t,\tau} \\
& = \bomega_t - \beta_t(I - C_t)(K^\star)^{-1}\nabla\mL_t - \beta_t(I - C_t)\cbr{K_t^{-1} - (K^\star)^{-1}}\nabla\mL_t + \beta_t\btheta_t+ \rbr{\baralpha_t - \beta_t} \bz_{t,\tau} \nonumber\\
& \stackrel{\mathclap{\eqref{rec:def:psi}}}{=}\;\;  \cbr{I - \beta_t(I - C_t)}\bomega_t - \beta_t(I - C_t)(K^\star)^{-1}\bpsi_t - \beta_t(I - C_t)\cbr{K_t^{-1} - (K^\star)^{-1}}\nabla\mL_t \nonumber\\
&\quad + \beta_t\btheta_t+ \rbr{\baralpha_t - \beta_t} \bz_{t,\tau} \nonumber\\
& \stackrel{\mathclap{\eqref{rec:def:c}}}{=}\;\; \cbr{I - \beta_t(I - C^\star)}\bomega_t + \beta_t(\btheta_t + \bdelta_t) +  \rbr{\baralpha_t - \beta_t} \bz_{t,\tau}. \nonumber
\end{align}
Moreover, by Assumption \ref{ass:2} and Lemma \ref{lem:1}(b), we have $\mE[\btheta_t\mid \mF_{t-1}] = \0$. The result \eqref{nequ:2}~is~established in \cite{Na2025Statistical} after (E.15). This completes the proof.
\end{proof}

\begin{lemma}[{\cite{Davis2024Asymptotic}, Lemma A.9}]\label{lem:phi_t_recursion_lemma}
Suppose $\{A_t\}_t$ is a non-negative sequence satisfying
\begin{equation*}
A_{t+1} \leq (1 - c_1 t^{-\beta}) A_t + c_2 t^{-2\beta}, \quad\quad \forall t\geq t_0,
\end{equation*}
for some constants $c_1, c_2>0$, $\beta\in(0.5,1)$, and positive integer $t_0$. Then, there exists a constant~$c_0>0$ such that $A_t \leq c_0 t^{-\beta}$ for all $t\geq 1$.
\end{lemma}

\begin{lemma}[Kronecker's lemma]\label{lem:2}

Suppose $\{A_t\}_t$ is a sequence such that $\sum_{t=0}^{\infty}A_t$ exists and is finite. For any divergent positive non-decreasing sequence $\{a_t\}_t$, we have $\frac{1}{a_T}\sum_{t=1}^{T}a_tA_t \rightarrow 0$~as~$T\rightarrow\infty$.

\end{lemma}

\section{Proofs of Main Theorems}

\subsection{Proof of Theorem \ref{thm:4}}\label{sec:B.2}

By Lemma \ref{lem:1}(a), (c), and Lemma \ref{lem:K_tconvergence}(a), we know $\|C^\star\|\leq \rho<1$. Thus, $I - C^{\star}$ is a positive~definite and, thus, invertible matrix. Let us define two matrices
\begin{equation}\label{eq:A_B_recursion_matrices}
B_i^t = \beta_i \sum_{k=i}^{t}\prod_{j= i+1}^{k}(I - \beta_j (I - C^{\star}))\qquad \text{  and  } \qquad A_i^t = B_i^t - (I - C^{\star})^{-1}.
\end{equation}
Recall that $\bomega_k = (\boldsymbol{x}_k - \boldsymbol{x}^{\star}, \boldsymbol{\lambda}_k - \boldsymbol{\lambda}^{\star})$. We apply Lemma \ref{lem:4.1} recursively and obtain
\begin{equation*}
\bomega_{k+1} = \prod_{j=0}^{k}\rbr{I - \beta_j(I-C^\star)}\bomega_0 + \sum_{i=0}^{k}\prod_{j=i+1}^{k}\rbr{I - \beta_j(I-C^\star)}\beta_i\rbr{\btheta_i + \bdelta_i+\frac{\baralpha_i - \beta_i}{\beta_i}\bz_{i,\tau}}.
\end{equation*}
Therefore, we obtain
\begin{align}\label{eq:polyakrecursion}
& \frac{1}{\sqrt{t}} \sum_{k = 1}^{t} \bomega_k \nonumber \\
& = \frac{1}{\sqrt{t}}\sum_{k=0}^{t-1}\prod_{j=0}^{k}\rbr{I - \beta_j(I-C^\star)}\bomega_0 + \frac{1}{\sqrt{t}}\sum_{k=0}^{t-1}\sum_{i=0}^{k}\prod_{j=i+1}^{k}\rbr{I - \beta_j(I-C^\star)}\beta_i\rbr{\btheta_i + \bdelta_i+\frac{\baralpha_i - \beta_i}{\beta_i}\bz_{i,\tau}} \nonumber\\
& \stackrel{\mathclap{\eqref{eq:A_B_recursion_matrices}}}{=} \frac{1}{\sqrt{t} \beta_0} B_{0}^{t-1} (I - \beta_0 (I - C^{\star})) \bomega_0 + \frac{1}{\sqrt{t}}\sum_{i=0}^{t-1}\sum_{k=i}^{t-1}\prod_{j=i+1}^{k}\rbr{I - \beta_j(I-C^\star)}\beta_i\rbr{\btheta_i + \bdelta_i+\frac{\baralpha_i - \beta_i}{\beta_i}\bz_{i,\tau}} \nonumber\\
& \stackrel{\mathclap{\eqref{eq:A_B_recursion_matrices}}}{=} \frac{1}{\sqrt{t} \beta_0} B_{0}^{t-1} (I - \beta_0 (I - C^{\star})) \bomega_0 + \frac{1}{\sqrt{t}}\sum_{i=0}^{t-1}B_{i}^{t-1}\rbr{\btheta_i + \bdelta_i+\frac{\baralpha_i - \beta_i}{\beta_i}\bz_{i,\tau}} \nonumber\\
& \stackrel{\mathclap{\eqref{eq:A_B_recursion_matrices}}}{=} \frac{1}{\sqrt{t} \beta_0} B_{0}^{t-1} (I - \beta_0 (I - C^{\star})) \bomega_0 + \frac{1}{\sqrt{t}}\sum_{i=0}^{t-1}B_{i}^{t-1}\rbr{\bdelta_i+\frac{\baralpha_i - \beta_i}{\beta_i}\bz_{i,\tau}} + \frac{1}{\sqrt{t}}\sum_{i=0}^{t-1}A_i^{t-1}\btheta_i  \nonumber \\
& \quad + \frac{1}{\sqrt{t}} \sum_{i = 0}^{t-1}(I - C^{\star})^{-1}\btheta_i.
\end{align}
In what follows, we analyze each right-hand-side term above and establish the following claims:
\begin{enumerate}[topsep=2pt,parsep=3pt,label=(\alph*):]
\setlength\itemsep{0.0em}
\item $\frac{1}{\sqrt{t}} \sum_{i = 0}^{t-1}(I - C^{\star})^{-1}\btheta_i \stackrel{d}{\longrightarrow} \mathcal{N}(0, \bar{\Xi}^{\star})$, where $\barXi^\star$ is defined in \eqref{equ:barXi}.
\item $\frac{1}{\sqrt{t}}\sum_{i=0}^{t-1}A_i^{t-1} \btheta_i = o_p(1)$.
\item $\frac{1}{\sqrt{t}}\sum_{i=0}^{t-1} B_i^{t-1} \bdelta_i = o(1)$.
\item $ \frac{1}{\sqrt{t}}\sum_{i=0}^{t-1} B_i^{t-1} \frac{\rbr{\baralpha_i - \beta_i}}{\beta_i}\bz_{i,\tau} = o(1)$.	
\item $\frac{1}{\sqrt{t} \beta_0} B_{0}^{t-1} [I - \beta_0 (I - C^{\star})] \bomega_0 = o(1)$.	
\end{enumerate}
Combining these claims together, we obtain
\begin{equation*}
\frac{1}{\sqrt{t}} \sum_{k = 1}^{t} \bomega_k \stackrel{d}{\longrightarrow} \mathcal{N}(0, \bar{\Xi}^{\star}).
\end{equation*}
This completes the proof by noting that $1/\sqrt{t}\sum_{k=0}^{t-1} \bomega_k = 1/\sqrt{t} \sum_{k = 1}^{t} \bomega_k - \bomega_t/\sqrt{t} + \bomega_{0}/\sqrt{t}$ and~$\bomega_t\rightarrow \0$ as $t\rightarrow 0$ almost surely.

\vskip4pt
\noindent $\bullet$ {\textbf{Proof of (a).}} By Lemma \ref{lem:4.1}, we know $\btheta_i$ is a martingale difference with the limiting covariance $\mE[(I-\tC^\star)\Xi^{\star}(I-\tC^\star)^{\top}]$. Thus, it suffices to verify the following Lindeberg condition.

\begin{lemma}\label{lem:B1}
Under the conditions of Theorem \ref{thm:4}, we have for any $\epsilon>0$, almost surely,
\begin{equation*}
\frac{1}{t} \sum_{i = 0}^{t-1} \mathbb{E}\left[ \|\boldsymbol{\theta}_i\|^2 \cdot \mathbf{1}_{\|\boldsymbol{\theta}_i\| > \epsilon \sqrt{t}}  \mid \mathcal{F}_{i-1} \right] \rightarrow 0 \quad \text{ as }\quad t\rightarrow\infty.
\end{equation*}
\end{lemma}

With the above Lindeberg condition, we apply the martingale central limit theorem in \cite[Proposition 2.1.9]{Duflo2013Random} and obtain the result in \textbf{(a)}.  

\vskip4pt
\noindent $\bullet$ \textbf{Proof of (b).} We need the following lemma to provide the uniform bounds for matrices $B_i^t$ and $A_i^t$ in \eqref{eq:A_B_recursion_matrices}.

\begin{lemma}[Adapted Lemma 1 in \cite{Polyak1992Acceleration}]\label{lem:B2}
Given $I - C^\star$ is positive~definite, there is a constant $\Upsilon_{AB}>0$ such that $\max\{\|B_i^t\|, \|A_i^t\|\}\leq \Upsilon_{AB}$, $\forall i, t\geq 0$. Furthermore, we have $\lim\limits_{t \rightarrow \infty} \frac{1}{t}\sum_{i=0}^{t} \norm{A_i^{t}} =0$.
\end{lemma}
  
With the above lemma and the fact that $\sup_{i \geq 0} \mE[\|\btheta_i\|^2]<\infty$ (see \eqref{eq: bounded_p_moment} in proving Lemma \ref{lem:B1}), we have
\begin{align*}
\mE \left[ \left\| \frac{1}{\sqrt{t}}\sum_{i=0}^{t-1}A_i^{t-1} \btheta_i \right\|^2 \right] = \frac{1}{t} \sum_{i = 0}^{t-1} \mE \left[ \|A_i^{t-1} \btheta_i \|^2 \right] \leq \Upsilon_{AB}\cdot \left(\sup_{i \geq 0} \mE[\|\btheta_i\|^2] \right)\cdot\frac{1}{t} \sum_{i = 0}^{t-1} \|A_i^{t-1}\|\rightarrow 0
\end{align*}
as $t\rightarrow\infty$. Thus, $\frac{1}{\sqrt{t}}\sum_{i=0}^{t-1}A_i^{t-1} \btheta_i = o_p(1)$ and this shows \textbf{(b)}.

\vskip4pt
\noindent$\bullet$ \textbf{Proof of (c).} By Lemma \ref{lem:B2}, we know it suffices to show that $1/\sqrt{t}\sum_{i=0}^{t-1} \left\| \bdelta_i \right\| \stackrel{a.s.}{\longrightarrow} 0$ as~$t\rightarrow\infty$.~We note that
\begin{align}\label{nequ:4}
\|\bdelta_i\| \; & \stackrel{\mathclap{\eqref{rec:def:c}}}{\leq} \; 2 \cbr{\|(K^\star)^{-1}\|\cdot\|\bpsi_i\| + \|K_i^{-1} - (K^\star)^{-1}\|\cdot\|\nabla\mL_i\|} + \|C_i - C^\star\|\cdot\|\bomega_i\| \quad (\|C_t\|\leq 1) \nonumber\\
& \leq 2\Upsilon_{K}\Upsilon_{L}\|\bomega_i\|^2 + 2\Upsilon_{K}^2\|K_i - K^\star\|\cdot\|\nabla\mL_i\| + \frac{2\tau\Upsilon_S}{\sigma_{\min}(K^\star)}\|K_i-K^\star\|\cdot\|\bomega_i\|,
\end{align}
where the second inequality is due to Assumption \ref{ass:1}, Lemma \ref{lem:K_tconvergence}, and Lemma \ref{lem:1}(c). We~present~the following lemma that studies the convergence of the terms on the right hand side.

\begin{lemma}\label{lem:B3}
Under the conditions of Theorem \ref{thm:4}, we have as $t\rightarrow\infty$,
\begin{equation*}
\frac{1}{\sqrt{t}}\sum_{i=0}^{t-1}  \|\bomega_i\|^2 \stackrel{a.s.}{\longrightarrow} 0 \qquad\text{ and }\qquad  \frac{1}{\sqrt{t}}\sum_{i=0}^{t-1}  \|K_i - K^{\star}\| \cdot (\|\bomega_i\| + \|\nabla\mL_i\|) \stackrel{a.s.}{\longrightarrow} 0.
\end{equation*}
\end{lemma}	

Combining Lemma \ref{lem:B3} with \eqref{nequ:4}, we prove the result in \textbf{(c)}.

\vskip4pt
\noindent$\bullet$ \textbf{Proof of (d).} By \eqref{equ:sandwich} and Lemma \ref{lem:B2}, it suffices to show that $1/\sqrt{t}\sum_{i=0}^{t-1}\chi_i\|\bz_{i,\tau}\|/\beta_i\stackrel{a.s.}{\longrightarrow} 0$ as $t\rightarrow\infty$. In fact, we know
\begin{equation*}
\mE \left[\sum_{i=0}^{\infty} \frac{1}{\sqrt{i+1}} \frac{\chi_i}{\beta_i}\|\bz_{i,\tau}\|\right] = \sum_{i=0}^{\infty} \frac{1}{\sqrt{i+1}} \frac{\chi_i}{\beta_i}\mE[\|\bz_{i,\tau}\|]\leq \frac{\Upsilon_zc_\chi}{c_\beta}\sum_{i=0}^{\infty}\frac{1}{(i+1)^{\chi-\beta+0.5}} <\infty,
\end{equation*}
where the last inequality is due to $\chi>\beta+0.5$. By Kronecker’s lemma in Lemma \ref{lem:2}, we obtain the result in \textbf{(d)}.

\vskip4pt
\noindent$\bullet$ \textbf{Proof of (e).} The result in \textbf{(e)} is trivial due to Lemma \ref{lem:B2}.

\subsection{Proof of Proposition \ref{prop:2}}

By the definitions of $\bar{\Xi}^{\star}$ in \eqref{equ:barXi} and $\Xi^{\star}$ in \eqref{equ:Omega}, we have 
\begin{align} \label{eq:18}
\bar{\Xi}^{\star} - \Xi^{\star}  &= (I - C^{\star})^{-1}\mathbb{E}\left[ (I - \widetilde{C}^{\star})\Xi^{\star}(I - \widetilde{C}^{\star})^{\top}\right](I - C^{\star})^{-1} - \Xi^{\star} \nonumber\\
&= (I - C^{\star})^{-1}\mE[(\widetilde{C}^{\star} - C^{\star}) \Xi^{\star} (\widetilde{C}^{\star} - C^{\star})^{\top}](I - C^{\star})^{-1}.
\end{align}
\noindent\textbf{(a) Without the sketching solver. } In this case, we have $C^{\star} = \Tilde{C}^{\star} = \0$, hence $\bar{\Xi}^{\star} = \Xi^{\star}$. 

\noindent\textbf{(b) With the sketching solver. } Since $\Xi^{\star} \succeq 0$, we know $(I - C^{\star})^{-1}\mE[(\widetilde{C}^{\star} - C^{\star}) \Xi^{\star} (\widetilde{C}^{\star} - C^{\star})^{\top}](I - C^{\star})^{-1} \succeq 0$, then we have $\bar{\Xi}^{\star} \succeq \Xi^{\star}$ by \eqref{eq:18}. Next, we recall that $\rho = 1-\gamma_S$ and have that
\begin{align}\label{eq:19}
\|\bar{\Xi}^{\star} - \Xi^{\star}\| &\; \stackrel{\mathclap{\eqref{eq:18}}}{=}\;  \|(I - C^{\star})^{-1}\mE[(\widetilde{C}^{\star} - C^{\star}) \Xi^{\star} (\widetilde{C}^{\star} - C^{\star})^{\top}](I - C^{\star})^{-1}\| \nonumber\\
&\; \leq\; \|(I - C^{\star})^{-1}\|^2 \cdot \|\mE[(\widetilde{C}^{\star} - C^{\star}) \Xi^{\star} (\widetilde{C}^{\star} - C^{\star})^{\top}]\| \nonumber\\
&\; \leq\; \frac{1}{(1 - \rho^{\tau})^2} \cdot \|\mE[\widetilde{C}^{\star} \Xi^{\star} (\widetilde{C}^{\star})^{\top}] - C^{\star} \Xi^{\star} C^{\star}\|  \nonumber\\
&\;\leq\; \frac{1}{(1 - \rho^{\tau})^2} \cdot \left(\|\mE[\widetilde{C}^{\star} \Xi^{\star} (\widetilde{C}^{\star})^{\top}]\| + \rho^{2\tau}\|\Xi^\star\|\right),
\end{align}
where the third and fourth inequalities use $\|C^{\star}\| \leq \rho^{\tau}$  in Lemma \ref{lem:1}(a) and Lemma \ref{lem:K_tconvergence}(a).~Furthermore, we have
\begin{align}\label{eq:20}
& \0 \preceq \mE[\widetilde{C}^{\star} \Xi^{\star} (\widetilde{C}^{\star})^{\top}] \preceq \|\Xi^{\star}\| \cdot \mE[\widetilde{C}^{\star}(\widetilde{C}^{\star})^{\top}] \nonumber\\
&\stackrel{\mathclap{\eqref{equ:tC}}}{=} \|\Xi^{\star}\| \cdot \mE \Bigg[ \left(\prod_{j=1}^{\tau}(I - K^\star S_j(S_j^{\top}(K^\star)^2S_j)^\dagger S_j^{\top}K^\star )\right) \left( \prod_{j=1}^{\tau}(I - K^\star S_j(S_j^{\top}(K^\star)^2S_j)^\dagger S_j^{\top}K^\star ) \right)^{\top} \Bigg] \nonumber\\
&= \|\Xi^{\star}\| \cdot \mE \Bigg[ \left(\prod_{j=2}^{\tau}(I - K^\star S_j(S_j^{\top}(K^\star)^2S_j)^\dagger S_j^{\top}K^\star )\right) \mE [(I - K^\star S_1(S_1^{\top}(K^\star)^2S_1)^\dagger S_1^{\top}K^\star )| S_2, S_3, ..., S_{\tau}] \nonumber\\
&\qquad \qquad \qquad \left( \prod_{j=2}^{\tau}(I - K^\star S_j(S_j^{\top}(K^\star)^2S_j)^\dagger S_j^{\top}K^\star ) \right)^{\top} \Bigg] \nonumber\\
&\preceq \|\Xi^{\star}\| \cdot \rho \cdot \mE \Bigg[ \left(\prod_{j=2}^{\tau}(I - K^\star S_j(S_j^{\top}(K^\star)^2S_j)^\dagger S_j^{\top}K^\star )\right) \left( \prod_{j=2}^{\tau}(I - K^\star S_j(S_j^{\top}(K^\star)^2S_j)^\dagger S_j^{\top}K^\star ) \right)^{\top} \Bigg],
\end{align}
where the fourth equality uses the fact that $(I - K^\star S_j(S_j^{\top}(K^\star)^2S_j)^\dagger S_j^{\top}K^\star )$ is a projection matrix for all $j \in [1, \tau]$, and the last inequality uses Assumption \ref{ass:3} and $K_t \rightarrow K^{\star}$ in Lemma \ref{lem:K_tconvergence}. We then repeat \eqref{eq:20} for $\tau$ times, and finally get that
\begin{equation}\label{eq:21}
\0  \preceq \mE[\widetilde{C}^{\star} \Xi^{\star} (\widetilde{C}^{\star})^{\top}] \preceq \rho^{\tau} \|\Xi^{\star}\| \cdot I.
\end{equation}
Then, we combine \eqref{eq:21} with \eqref{eq:19}, and obtain that
\begin{equation*}
\|\bar{\Xi}^{\star} - \Xi^{\star}\| \leq \frac{(\rho^{\tau} + \rho^{2\tau})}{(1 - \rho^{\tau})^2} \cdot \|\Xi^{\star}\| \leq \frac{(1 + \rho^{\tau})}{(1 - \rho^{\tau})^2} \rho^{\tau} \cdot \|\Xi^{\star}\|.
\end{equation*}
This completes the proof.

\subsection{Proof of Proposition \ref{prop:3}}

With $c_{\beta} = 1$ and  $\beta = 1$, the Lyapunov equation \eqref{equ:Lyapunov}  can be written as
\begin{equation*}
(0.5I - C^{\star})\Tilde{\Xi}^{\star}  + \Tilde{\Xi}^{\star} (0.5 I - C^{\star}) = \mathbb{E}[ (I - \widetilde{C}^{\star})\Xi^{\star}(I - \widetilde{C}^{\star})^{\top}].
\end{equation*}
Recall from \eqref{equ:barXi} that $\bar{\Xi}^\star = (I - C^{\star})^{-1}\mathbb{E}[(I - \widetilde{C}^{\star})\Xi^{\star}(I-\widetilde{C}^{\star})^{\top}](I- C^{\star})^{-1}$, then we obtain
\begin{align*}
(0.5I - C^{\star})(\Tilde{\Xi}^{\star} - \bar{\Xi}^{\star}) & + (\Tilde{\Xi}^{\star} - \bar{\Xi}^{\star})(0.5I - C^{\star}) \\
& = \mathbb{E}[ (I - \widetilde{C}^{\star})\Xi^{\star}(I - \widetilde{C}^{\star})^{\top}] - \bar{\Xi}^{\star} + C^{\star}\bar{\Xi}^{\star} + \bar{\Xi}^{\star}C^{\star} \\
& = \mathbb{E}[ (I - \widetilde{C}^{\star})\Xi^{\star}(I - \widetilde{C}^{\star})^{\top}] - (I - C^{\star})\bar{\Xi}^{\star} - \bar{\Xi}^{\star}(I - C^{\star}) + \bar{\Xi}^{\star} \\
& = (I - (I - C^{\star})^{-1}) \mathbb{E}[ (I - \widetilde{C}^{\star})\Xi^{\star}(I - \widetilde{C}^{\star})^{\top}] (I - (I - C^{\star})^{-1}) \\
& \succeq \0.
\end{align*}
By Lemma \ref{lem:1}, we have $\|C^{\star}\| \leq \rho^{\tau} < 1/2$, then the basic Lyapunov theorem \cite[Theorem 4.6]{Khalil2002Nonlinear} suggests that $\Tilde{\Xi}^{\star} \succeq \bar{\Xi}^{\star}$. This completes the proof.

\subsection{Proof of Theorem \ref{thm:FCLT}}

Recall that $\bomega_k = (\bx_k - \bx^{\star}, \blambda_k - \blambda^{\star})$. Following the same decomposition in \eqref{eq:polyakrecursion}, the partial sum process can be decomposed as 
\begin{align*}
\frac{1}{\sqrt{t}} \sum_{k = 1}^{\lfloor rt \rfloor} \bomega_k & \stackrel{\mathclap{\eqref{eq:polyakrecursion}}}{=} \frac{1}{\sqrt{t}} \sum_{i = 0}^{\lfloor rt \rfloor -1}(I - C^{\star})^{-1} \btheta_i + \frac{1}{\sqrt{t}}\sum_{i=0}^{\lfloor rt \rfloor -1}A_i^{\lfloor rt \rfloor-1} \btheta_i + \frac{1}{\sqrt{t}}\sum_{i=0}^{\lfloor rt \rfloor -1}B_i^{\lfloor rt \rfloor -1}  \left( \bdelta_i + \frac{\baralpha_i - \beta_i}{\beta_i}\bz_{i,\tau} \right) \\
& \quad + \frac{1}{\sqrt{t} \beta_0} B_{0}^{\lfloor rt \rfloor -1} [I - \beta_0 (I - C^{\star})] \bomega_0 \\
& = \frac{1}{\sqrt{t}} \sum_{i = 0}^{\lfloor rt \rfloor -1}(I - C^{\star})^{-1} \btheta_i + \frac{1}{\sqrt{t}}\sum_{i=0}^{\lfloor rt \rfloor -1} A_i^{t-1} \btheta_i + \frac{1}{\sqrt{t}}\sum_{i=0}^{\lfloor rt \rfloor -1}\left(A_i^{\lfloor rt \rfloor-1} - A_i^{t-1}\right) \btheta_i  \\
& \quad + \frac{1}{\sqrt{t}}\sum_{i=0}^{\lfloor rt \rfloor -1}B_i^{\lfloor rt \rfloor -1}  \left( \bdelta_i + \frac{\baralpha_i - \beta_i}{\beta_i}\bz_{i,\tau} \right) + \frac{1}{\sqrt{t} \beta_0} B_{0}^{\lfloor rt \rfloor -1} [I - \beta_0 (I - C^{\star})] \bomega_0,
\end{align*}
where $r \in [0, 1]$, and $A_i^{\lfloor rt \rfloor}$ and $B_i^{\lfloor rt \rfloor}$ are defined in \eqref{eq:A_B_recursion_matrices}. In what follows, we analyze each~right-hand-side term above and establish the following claims: 
\begin{enumerate}[topsep=2pt,parsep=3pt,label=(\alph*):]
\setlength\itemsep{0.0em}
\item $\frac{1}{\sqrt{t}} \sum_{i = 0}^{\lfloor rt \rfloor - 1} (I - C^{\star})^{-1} \btheta_i \Longrightarrow  (\bar{\Xi}^{\star})^{1/2} W_{d+m}(r)$, where $\bar{\Xi}^{\star}$ is defined in \eqref{equ:barXi}.
\item $\sup_{r \in [0, 1]} \| \frac{1}{\sqrt{t}}\sum_{i=0}^{\lfloor rt \rfloor - 1}A_i^{t - 1} \btheta_i \| = o_p(1)$.
\item $\sup_{r \in [0, 1]} \|\frac{1}{\sqrt{t}}\sum_{i=0}^{\lfloor rt \rfloor -1}(A_i^{\lfloor rt \rfloor-1} - A_i^{t-1}) \btheta_i\| = o_p(1)$.
\item $\sup_{r \in [0, 1]} \|\frac{1}{\sqrt{t}}\sum_{i=0}^{\lfloor rt \rfloor-1} B_i^{\lfloor rt \rfloor-1} \bdelta_i \| = o(1)$.
\item $\sup_{r \in [0, 1]} \|\frac{1}{\sqrt{t}}\sum_{i=0}^{\lfloor rt \rfloor-1} B_i^{\lfloor rt \rfloor-1} \frac{\rbr{\baralpha_i - \beta_i}}{\beta_i}\bz_{i,\tau} \| = o(1)$.	
\item $\sup_{r \in [0, 1]}\|\frac{1}{\sqrt{t} \beta_0} B_{0}^{\lfloor rt \rfloor-1} (I - \beta_0 (I - C^{\star})) \bomega_0\| = o(1)$.	
\end{enumerate}
Combining these claims together, we obtain
\begin{equation*}
\frac{1}{\sqrt{t}} \sum_{k = 1}^{\lfloor rt \rfloor } \bomega_k \Longrightarrow (\bar{\Xi}^{\star})^{1/2} W_{d+m}(r).
\end{equation*}
This completes the proof of the theorem by noting that $\frac{1}{\sqrt{t}}\sum_{k=0}^{\lfloor rt \rfloor -1} \bomega_k = \frac{1}{\sqrt{t}} \sum_{k = 1}^{\lfloor rt \rfloor } \bomega_k - \frac{1}{\sqrt{t}}\bomega_{\lfloor rt \rfloor } + \frac{1}{\sqrt{t}} \bomega_{0}$ and the fact that $\sup_{r \in [0,1]}\|\bomega_{\lfloor rt \rfloor }\|/\sqrt{t} = o(1)$.

\vskip4pt
\noindent $\bullet$ {\textbf{Proof of (a).}} By Lemma \ref{lem:4.1}, we know that $\btheta_i$ is a martingale difference and the conditional~covariance of $\btheta_i$ converges almost surely to $\mE[(I-\tC^\star)\Xi^{\star}(I-\tC^\star)^{\top}]$. By Lemma \ref{lem:B1}, we have verified~the Lindeberg condition. Then, we apply the martingale Functional Central Limit Theorem (FCLT)  \citep[Theorem 4.2]{Hall2014Martingale}, we have as $t \rightarrow \infty$
\begin{equation*}
\frac{1}{\sqrt{t}} \sum_{i = 0}^{\lfloor rt \rfloor - 1} (I - C^{\star})^{-1} \btheta_i \Longrightarrow  (\bar{\Xi}^{\star})^{1/2} W_{d+m}(r).
\end{equation*}
This completes the proof of \textbf{(a)}.

\vskip4pt
\noindent $\bullet$ \textbf{Proof of \textbf{(b)}.} Note that $\sum_{i=0}^{\lfloor rt \rfloor -1} A_i^{t-1} \btheta_i$ is a martingale indexed by $r \in [0, 1]$. By Doob's~inequality \cite[Theorem 2.2]{Hall2014Martingale}, we have
\begin{align*}
\mE \left[\sup_{r \in [0, 1]} \left\| \frac{1}{\sqrt{t}}\sum_{i=0}^{\lfloor rt \rfloor - 1}A_i^{t - 1} \btheta_i \right\|^2 \right] 
&\leq \frac{4}{t} \mE \left[ \left\|  \sum_{i=0}^{t-1}A_i^{t-1} \btheta_i \right\|^2 \right] = \frac{4}{t} \sum_{i = 0}^{t-1} \mE \left[ \left\|  A_i^{t-1} \btheta_i \right\|^2 \right] \\
&\leq \Upsilon_{AB} \cdot \left(\sup_{i \geq 0} \mE[\|\btheta_i\|^2] \right) \cdot \frac{4}{t} \sum_{i = 0}^{t-1} \|A_i^{t-1}\| \longrightarrow 0,
\end{align*}
where the last inequality is due to Lemma \ref{lem:B2} and the fact that $\btheta_{i}$ has bounded $(2+\delta)$-moment (see \eqref{eq: bounded_p_moment} in proving Lemma \ref{lem:B1}). Thus, we have \mbox{$\sup_{r \in [0, 1]} \| \frac{1}{\sqrt{t}}\sum_{i=0}^{\lfloor rt \rfloor - 1}A_i^{t - 1} \btheta_i \| = o_p(1)$}~and~\mbox{complete}~the proof of \textbf{(b)}.

\vskip4pt
\noindent $\bullet$ \textbf{Proof of (c).} We notice that for any integer $n \in \{0, 1, 2, ..., t-1\}$, 
\begin{align}\label{eq:17}
\sum_{i = 0}^{n} \left(A_i^{t-1} - A_i^n \right) \btheta_i & \; \stackrel{\mathclap{\eqref{eq:A_B_recursion_matrices}}}{=}\; \sum_{i = 0}^{n} \left(B_i^{t-1} - B_i^n \right) \btheta_i \stackrel{\eqref{eq:A_B_recursion_matrices}}{=} \sum_{i = 0}^{n} \sum_{k = n+1}^{t-1} \prod_{j= i+1}^{k}(I - \beta_j (I - C^{\star})) \beta_i \btheta_i  \nonumber\\
&\; =\;  \sum_{k = n+1}^{t-1}  \sum_{i = 0}^{n} \prod_{j= i+1}^{k}(I - \beta_j (I - C^{\star})) \beta_i \btheta_i \nonumber\\
& \; =\;  \frac{1}{\beta_{n+1}} \cdot \beta_{n+1} \sum_{k = n+1}^{t-1} \prod_{j= n+1}^{k}(I - \beta_j (I - C^{\star})) \left[  \sum_{i = 0}^n \prod_{j= i+1}^{n}(I - \beta_j (I - C^{\star})) \beta_i \btheta_i  \right] \nonumber\\
&\; \stackrel{\mathclap{\eqref{eq:A_B_recursion_matrices}}}{=}\; \frac{1}{\beta_{n+1}} \cdot B_{n+1}^{t-1} (I - \beta_{n+1} (I - C^{\star})) \left[  \sum_{i = 0}^n \prod_{j= i+1}^{n}(I - \beta_j (I - C^{\star})) \beta_i \btheta_i  \right].
\end{align}
By Lemma \ref{lem:B2}, we know that $\|B_{n+1}^{t-1} (I - \beta_{n+1} (I - C^{\star}))\| \leq \Upsilon_{AB} (1+2c_{\beta})$ for any $n, t \geq 0$. Thus, we obtain
\begin{multline}\label{eq:16}
\sup_{r \in [0, 1]} \left\|\frac{1}{\sqrt{t}}\sum_{i=0}^{\lfloor rt \rfloor -1} \left(A_i^{\lfloor rt \rfloor-1} - A_i^{t-1}\right) \btheta_i\right\| = \sup_{n \in [0, t-1]} \left\|\frac{1}{\sqrt{t}}\sum_{i=0}^{n}\left(A_i^{n} - A_i^{t-1}\right) \btheta_i\right\| \\
\stackrel{\mathclap{\eqref{eq:17}}}{\leq} \Upsilon_{AB} (1+2c_{\beta}) \sup_{n \in [0, t-1]} \left\|\frac{1}{\sqrt{t}} \frac{1}{\beta_{n+1}} \sum_{i = 0}^n \prod_{j= i+1}^{n}(I - \beta_j (I - C^{\star})) \beta_i \btheta_i \right\|.
\end{multline}
We need a technical lemma to proceed with the proof.

\begin{lemma}[Adapted Lemma B.2 in \cite{Li2023Online}]  \label{lem:phi_2_error}

Let $\{\btheta_n\}_{n \geq 0}$ be a martingale difference sequence. Given the matrix $I - C^{\star}$ is positive definite, we define an auxiliary sequence $\{\boldsymbol{y}_n\}_{n \geq 0}$ as follows: $\boldsymbol{y}_0 = \boldsymbol{0}$, and for $n \geq 0$,
\begin{equation*}
\boldsymbol{y}_{n+1} = \sum_{i = 0}^{n} \left( \prod_{j = i+1}^{n}(I - \beta_j (I - C^{\star}))\right) \beta_i \btheta_i.
\end{equation*}	
Let the step size $\beta_i = c_{\beta}/(i+1)^{\beta}$ for some $c_{\beta} > 0$ and $\beta \in (0.5, 1)$. If $\sup_{i \geq 0} \mE \|\btheta_i\|^{2+\delta} < \infty$ for some $\delta > 0$, then we have that as $t \rightarrow \infty$,
\begin{equation*}
\sup_{n \in [0, t-1]} \left\|\frac{1}{\sqrt{t}} \frac{1}{\beta_{n+1}} \sum_{i = 0}^n \prod_{j= i+1}^{n}(I - \beta_j (I - C^{\star})) \beta_i \btheta_i \right\| = o_p(1).
\end{equation*}	
\end{lemma}
We recall that $\btheta_i$ is a martingale difference with $(2+\delta)$-th bounded moment \eqref{eq: bounded_p_moment}. Thus, we directly apply Lemma \ref{lem:phi_2_error} to \eqref{eq:16} and complete the proof of \textbf{(c)}.

\vskip4pt
\noindent $\bullet$ \textbf{Proof of (d).} By Lemma \ref{lem:B2}, we know that $\|B_i^{t}\| \leq \Upsilon_{AB}$ for any $i, t \geq 0$, then as $t \to \infty$,
\begin{equation*}
\sup_{r \in [0, 1]} \left\|\frac{1}{\sqrt{t}} \sum_{i=0}^{\lfloor rt \rfloor - 1} B_i^{\lfloor rn \rfloor - 1}  \bdelta_i \right\| \leq \frac{\Upsilon_{AB}}{\sqrt{t}} \sum_{i=0}^{t-1} \|\bdelta_i\| \stackrel{a.s.}{\longrightarrow} 0,
\end{equation*}
where the convergence holds due to \eqref{nequ:4} and Lemma \ref{lem:B3}. This completes the proof of \textbf{(d)}.

\vskip4pt
\noindent $\bullet$ \textbf{Proof of (e).} By Lemma \ref{lem:B2}, we know that $\|B_i^{t}\| \leq \Upsilon_{AB}$ for any $i, t \geq 0$, then as $t \to \infty$,
\begin{equation*}
\sup_{r \in [0, 1]} \left\|\frac{1}{\sqrt{t}} \sum_{i=0}^{\lfloor rt \rfloor -1 } B_i^{\lfloor rt \rfloor - 1}  \frac{\rbr{\baralpha_i - \beta_i}}{\beta_i}\bz_{i,\tau} \right\| \leq \frac{\Upsilon_{AB}}{\sqrt{t}} \sum_{i=0}^{t - 1} \frac{\chi_i}{\beta_i} \|\bz_{i,\tau}\| \stackrel{a.s.}{\longrightarrow} 0.
\end{equation*}
where the convergence holds due to the proof of \textbf{(d)} in Section \ref{sec:B.2}. This completes the proof of \textbf{(e)}.

\vskip4pt
\noindent $\bullet$ \textbf{Proof of (f).} By Lemma \ref{lem:B2}, we know that $\|B_i^{t}\| \leq \Upsilon_{AB}$ for any $i, t \geq 0$, then as $t \to \infty$,
\begin{equation*}
\sup_{r \in [0, 1]} \left\|\frac{1}{\sqrt{t} \beta_0} B_{0}^{\lfloor rt \rfloor} (I - \beta_0 (I - C^{\star})) \bomega_0 \right\| \leq  \frac{\Upsilon_{AB}}{\sqrt{t} \beta_0} \|(I - \beta_0 (I - C^{\star})) \bomega_0\| \rightarrow 0.
\end{equation*}
This completes the proof of \textbf{(f)}.

\subsection{Proof of Theorem \ref{thm:Random_Scaling}}

We first state the following Lemma, which shows that the vector $\boldsymbol{w} \in \mR^{d+m}$ defined in Theorem~\ref{thm:Random_Scaling} satisfies $\boldsymbol{w}^{\top} \bar{\Xi}^{\star} \boldsymbol{w} > 0$, where $\bar{\Xi}^{\star}$ is defined in \eqref{equ:barXi}.

\begin{lemma}\label{lem:B4}

Let $G^{\star} = \nabla c(\bx^{\star})$ be the constraint Jacobian evaluated at $\bx^{\star}$. Suppose $\cov(\nabla F(\tx; \xi))\succ 0$, for~any~vector $\boldsymbol{w} = (\boldsymbol{w}_{\bx}, \boldsymbol{w}_{\blambda}) \in \mR^{d+m}$ with $\bw\notin \text{Span}((G^\star)^\top)\otimes\0_m$, we have $\boldsymbol{w}^{\top} \bar{\Xi}^{\star} \boldsymbol{w} > 0$.
\end{lemma}

Recall the definition that $\bomega_t = (\bx_t-\bx^{\star}, \blambda_t - \blambda^{\star}) $, we consider the following random function
\begin{equation*}
C_t(r) \coloneqq \frac{1}{\sqrt{t}} \sum_{i = 0}^{\lfloor rt \rfloor-1} \boldsymbol{w}^{\top} \bomega_i.
\end{equation*}
Based on Theorem \ref{thm:FCLT}, the one-dimensional random function $C_t(r)$ satisfies:
\begin{equation*}
C_t(r) \Longrightarrow (\boldsymbol{w}^{\top} \bar{\Xi}^{\star} \boldsymbol{w})^{1/2} W_{1}(r), \qquad r \in [0, 1],
\end{equation*}
where $W_{1}(\cdot)$ stands for the standard one-dimensional Brownian motion. Define $\bar {\bomega}_t = \frac{1}{t} \sum_{i = 0}^{t-1} \bomega_i$. Then, we have
\begin{multline*}
\frac{\sqrt{t} \ \boldsymbol{w}^{\top} \bar{\bomega}_t }{\sqrt{\boldsymbol{w}^{\top} V_t \boldsymbol{w}}}  = \frac{C_t(1)}{\sqrt{\frac{1}{t} \sum_{i = 1}^t  \cbr{\frac{i}{\sqrt{t}} \boldsymbol{w}^{\top} (\bar{\bomega}_i - \bar{\bomega}_t)}^2}}\\
= \frac{C_t(1)}{\sqrt{\frac{1}{t} \sum_{i = 1}^t  \cbr{\frac{i}{\sqrt{t}} \boldsymbol{w}^{\top} \left[\frac{1}{i}\sum_{k = 0}^{i-1}\bomega_k   - \frac{1}{t}\sum_{k = 0}^{t-1}\bomega_k \right] }^2}} = \frac{C_t(1)}{\sqrt{\frac{1}{t}\sum_{i = 1}^{t}\cbr{ C_t\left(\frac{i}{t}\right) - \frac{i}{t}C_t(1)}^2}}.
\end{multline*}
By the Riemann integral approximation, as $t \rightarrow \infty$, we have
\begin{equation*}
\frac{1}{t}\sum_{i = 1}^{t}\left( C_t\left(\frac{i}{t}\right) - \frac{i}{t}C_t(1)\right)^2 \longrightarrow \int_0^1 (C_t(r) - rC_t(1))^2 \text{d}r.
\end{equation*}
Since $\frac{C_t(1)}{\sqrt{\int_0^1 (C_t(r) - rC_t(1))^2 \text{d}r}}$ is a continuous functional of $C_t(\cdot)$, the continuous mapping theorem \cite[Theorem 3.4.3]{Whitt2002Stochastic} implies
\begin{align*}
\frac{C_t(1)}{\sqrt{\frac{1}{t}\sum_{i = 1}^{t}\left( C_t\left(\frac{i}{t}\right) - \frac{i}{t}C_t(1)\right)^2}} &  \stackrel{d}{\longrightarrow}   \frac{(\boldsymbol{w}^{\top} \bar{\Xi}^{\star} \boldsymbol{w})^{1/2} \cdot W_1(1)}{(\boldsymbol{w}^{\top} \bar{\Xi}^{\star} \boldsymbol{w})^{1/2} \cdot \sqrt{\int_{0}^1  \left(W_1(r) - rW_1(1) \right)^2 dr}} \\
& = \frac{W_1(1)}{\sqrt{\int_{0}^1  \left(W_1(r) - rW_1(1) \right)^2 dr}},
\end{align*}
where the last equality follows since $\boldsymbol{w}^{\top} \bar{\Xi}^{\star} \boldsymbol{w} > 0$ by Lemma \ref{lem:B4}. Therefore, we obtain 
\begin{equation*}
\frac{\sqrt{t} \ \boldsymbol{w}^{\top} \bar{\bomega}_t }{\sqrt{\boldsymbol{w}^{\top} V_t \boldsymbol{w}}}   \stackrel{d}{\longrightarrow}  \frac{W_1(1)}{\sqrt{\int_{0}^1  \left(W_1(r) - rW_1(1) \right)^2 dr}},
\end{equation*}
which completes the proof.

\section{Proofs of Technical Lemmas}

\subsection{Proof of Lemma \ref{lem:B1}}

By Assumption \ref{ass:1}, there exists a constant $\Upsilon_u>0$ such that
\begin{equation}\label{equ:1}
\|\nabla^2\mL(\bx, \blambda)\| \vee \|\nabla\mL(\bx, \blambda)\| \leq \Upsilon_u, \quad \forall (\bx, \blambda)\in \mX\times\Lambda.
\end{equation} 
For $\delta>0$ in Assumption \ref{ass:2}, we let $q = 2 + \delta$ and have
\begin{align} \label{eq: bounded_p_moment}
& \mE[\|\btheta_i\|^q \mid \mF_{i-1}]\;\; \stackrel{\mathclap{\eqref{rec:def:b}}}{\leq}\;\; 2^{q-1}\cbr{\mE[\|(I-C_i)K_i^{-1}(\bnabla\mL_i - \nabla\mL_i)\|^q \mid \mF_{i-1}] + \mE[\|\bz_{i,\tau} - (I-C_i)\tbz_i\|^q\mid \mF_{i-1}]} \nonumber\\
& \stackrel{\mathclap{\eqref{equ:z:recur}}} {\leq} 2^{q-1}\rbr{ \mE[\|(I-C_i)K_i^{-1}\|^q\|\barg_i - \nabla f_i\|^q\mid \mF_{i-1}] + \mE[\|(\tC_i - C_i)\tbz_i\|^q\mid \mF_{i-1}]}\quad  \nonumber\\
& \leq 2^{q-1}\rbr{2^q\Upsilon_K^{p}\mE[\|\barg_i - \nabla f_i\|^q\mid \mF_{i-1}] + 2^q\mE[\|\tbz_i\|^q\mid \mF_{i-1}]}\quad  (\|\tC_i\|\leq 1,  \|C_i\|\leq 1,  \|K_i^{-1}\|\leq \Upsilon_K) \nonumber\\
& \stackrel{\mathclap{\eqref{equ:BM}}}{\leq}\; 2^{q-1}\rbr{2^q\Upsilon_K^q\Upsilon_m+ 2^q\mE[\|\tbz_i\|^q\mid \mF_{i-1}]} \nonumber\\
& \stackrel{\mathclap{\eqref{equ:Newton}}}{\leq} 2^{q-1}\rbr{2^q\Upsilon_K^q\Upsilon_m + 2^q\Upsilon_K^q\mE[\|\bnabla\mL_i\|^q\mid \mF_{i-1}]} \quad (\|K_i^{-1}\|\leq \Upsilon_K) \nonumber\\
& \leq 2^{q-1}\rbr{2^q\Upsilon_K^q\Upsilon_m + 2^q\Upsilon_K^q\cbr{2^{q-1}\|\nabla\mL_i\|^q + 2^{q-1}\mE[\|\barg_i - \nabla f_i\|^q\mid \mF_{i-1}]}} \nonumber\\
& \stackrel{\mathclap{\eqref{equ:BM}}}{\leq} 2^{q-1}\rbr{2^q\Upsilon_K^q\Upsilon_m + 2^q\Upsilon_K^q\cbr{2^{q-1}\Upsilon_u^q + 2^{q-1}\Upsilon_m}} \eqqcolon C_{q},
\end{align}
where the first inequality is due to Jensen's inequality; $\|K_i^{-1}\|\leq \Upsilon_{K}$ is due to Lemma \ref{lem:K_tconvergence}(b); and the last inequality is also due to \eqref{equ:1}. With the above display, we have for any $\epsilon>0$,
\begin{equation}\label{eq:Lindeberg_CLT}
\frac{1}{t} \sum_{i = 0}^{t-1} \mathbb{E}\left[ \|\boldsymbol{\theta}_i\|^2 \cdot \mathbf{1}_{\|\boldsymbol{\theta}_i\| > \epsilon \sqrt{t}}  \mid \mathcal{F}_{i-1} \right] \leq \frac{1}{\epsilon^{\delta} t^{1+\delta/2}} \sum_{i = 0}^{t-1} \mathbb{E}\left[ \|\boldsymbol{\theta}_i\|^{2+\delta} | \mathcal{F}_{i-1} \right] \le \frac{C_{2+\delta}}{\epsilon^{\delta}t^{\delta/2}} \longrightarrow 0.
\end{equation}
This completes the proof.

\subsection{Proof of Lemma \ref{lem:B3}}

By \eqref{equ:1}, we have $\|\nabla\mL_i\|\leq \Upsilon_u\|\bomega_i\|$.~By the fact $\|K_i - K^{\star}\|\|\bomega_i\|\leq 0.5(\|K_i - K^{\star}\|^2+ \|\bomega_i\|^2)$,~we~know it suffices to show $1/\sqrt{t}\cdot\sum_{i=0}^{t-1} (\|\bomega_i\|^2 + \|K_i - K^\star\|^2)\rightarrow 0$ as $t\rightarrow \infty$ almost surely. Note that
\begin{align}\label{eq:K_t-K*_decomposition}
\hskip-0.6cm \|K_i - K^\star\| & \;\; \stackrel{\mathclap{\eqref{equ:Newton}, \eqref{equ:Omega}}}{=}\;\;
\left\|\begin{pmatrix}
B_i - \nabla_{\boldsymbol{x}}^2 \mathcal{L}^\star & G_i^{\top} - (G^\star)^{\top} \\
G_i - G^\star & 0 
\end{pmatrix}\right\| \nonumber\\
&\;\; \leq \;\; \left\|\begin{pmatrix}
\frac{1}{i} \sum_{j=0}^{i-1} \bar{\nabla}_{\boldsymbol{x}}^2 \mathcal{L}_j - \nabla_{\boldsymbol{x}}^2 \mathcal{L}^\star & G_i^{\top} - (G^\star)^{\top} \\
G_i - G^\star & 0 
\end{pmatrix}\right\| + \|\Delta_i\| \nonumber\\
&\;\; \leq \;\; \left\|\frac{1}{i} \sum_{j=0}^{i-1} \bar{\nabla}_{\boldsymbol{x}}^2 \mathcal{L}_j - \nabla_{\boldsymbol{x}}^2 \mathcal{L}^\star\right\| + \|G_i - G^\star\| + \|\Delta_i\| \nonumber\\
&\;\; \leq\;\; \left\|\frac{1}{i} \sum_{j=0}^{i-1} (\bar{H}_j - \nabla^2 f_j)\right\| + \frac{1}{i} \sum_{j=0}^{i-1}\|\nabla_{\boldsymbol{x}}^2 \mathcal{L}_j - \nabla_{\boldsymbol{x}}^2 \mathcal{L}^\star\| + \|G_i - G^\star\| + \|\Delta_i\| \nonumber\\
&\;\; \leq\;\; \left\|\frac{1}{i} \sum_{j=0}^{i-1} (\bar{H}_j - \nabla^2 f_j)\right\| + \frac{\Upsilon_L}{i}\sum_{j = 0}^{i-1} \|\bomega_j\| + \Upsilon_L \|\bomega_i\| + \|\Delta_i\| \quad (\text{Assumption \ref{ass:1}}).
\end{align}
By Lemma \ref{lem:K_tconvergence}(a), the regularization term  $\Delta_i = 0$ for all sufficiently large $i$ almost surely. Thus,~we~focus on the first two terms on the right hand side. Let us define
\begin{equation}\label{equ:2}
R_i \coloneqq \left\|\frac{1}{i} \sum_{j=0}^{i-1} (\bar{H}_j - \nabla^2 f_j)\right\| + \frac{1}{i}\sum_{j = 0}^{i-1} \|\bomega_j\|\quad \text{ for }\; i\geq 0,
\end{equation}
and for any fixed integer $t_0>0$ and a constant $\nu>0$, we define a stopping time
\begin{equation}\label{def:stopping_time}
\mu_{t_0, \nu} = \inf\{t\geq t_0: \|K_t - K^\star\|>\nu \;\text{ OR }\; \|\bomega_t\|>\nu\}.
\end{equation}
The next lemma provides the convergence rate of $\|\bomega_i\|^2$ and $\|R_i\|^2$ for $i$ large enough.

\begin{lemma}\label{lem:C1}
Under the conditions of Theorem \ref{thm:4} and supposing
\begin{equation}\label{eq:small_nu}
\nu \leq \frac{1 - \rho^{\tau}}{4\Upsilon_{K}(1+\Upsilon_{L})}
\end{equation}
with $\Upsilon_{K}$ from Lemma \ref{lem:K_tconvergence}(b) and $\Upsilon_L$ from Assumption \ref{ass:1}, there exists a constant $\Upsilon_{\mu}>0$ such~that
\begin{equation*}
\mE[\|\bomega_i\|^2 \mathbf{1}_{\{\mu_{t_0, \nu} > i\}}] \leq \Upsilon_{\mu} \beta_i\quad\quad \text{ and }\quad\quad \mathbb{E}[R_i^2\mathbf{1}_{\{\mu_{t_0, \nu} > i\}}] \leq \Upsilon_{\mu} \beta_i,\quad\quad \forall i\geq 0.
\end{equation*}
\end{lemma}

By Lemma \ref{lem:C1}, we let $\nu$ satisfy \eqref{eq:small_nu} and then obtain
\begin{equation*}
\mE \left[\sum_{i = 0}^{\infty} \frac{1}{\sqrt{i+1}} \|\bomega_i\|^2 \mathbf{1}_{\{\mu_{t_0, \nu} > i\}}\right] = \sum_{i = 0}^{\infty} \frac{1}{\sqrt{i+1}} \mE[\|\bomega_i\|^2 \mathbf{1}_{\{\mu_{t_0, \nu} > i\}}] \leq \sum_{i = 0}^{\infty} \frac{\Upsilon_{\mu}\beta_i}{\sqrt{i+1}} <\infty,
\end{equation*}
where the last inequality is due to $\beta_i = c_\beta/(i+1)^\beta$ with $\beta\in(0.5,1)$. This suggests that
\begin{equation*}
\sum_{i = 0}^{\infty} \frac{1}{\sqrt{i+1}} \|\bomega_i\|^2 \mathbf{1}_{\{\mu_{t_0, \nu} > i\}} <\infty \quad \quad \text{almost surely}.
\end{equation*} 
Since $\bomega_i\rightarrow\0$ and $K_i \rightarrow K^\star$ as $i\rightarrow\infty$ almost surely, for any run and for any $\nu>0$, there exists $t_0$~such that $\mu_{t_0, \nu} = \infty$, which implies that $\sum_{i = 0}^{\infty} \|\bomega_i\|^2/\sqrt{i+1}<\infty$ almost surely. By~Lemma~\ref{lem:2},~we~have $\sum_{i=0}^{t-1}\|\bomega_i\|^2/\sqrt{t}\rightarrow 0$ as $t\rightarrow\infty$ almost surely. Following the same reasoning from Lemma \ref{lem:C1},~we also obtain $\sum_{i=0}^{t-1}\|R_i\|^2/\sqrt{t}\rightarrow 0$, implying that $\sum_{i=0}^{t-1}\|K_i - K^\star\|^2/\sqrt{t}\rightarrow 0$ as $t\rightarrow\infty$ due to \eqref{eq:K_t-K*_decomposition}~and \eqref{equ:2}. This completes the proof.

\subsection{Proof of Lemma \ref{lem:B4}}

By Proposition \ref{prop:2},  $\bar{\Xi}^{\star} \succeq \Xi^{\star}$. Thus, it suffices to prove $\boldsymbol{w}^{\top} \Xi^{\star} \boldsymbol{w} > 0$. We recall the~definition of $\Xi^{\star}$ in \eqref{equ:Omega} and define a subspace of $\mR^{d+m}$ as $U \coloneqq \{ (\boldsymbol{0}, \boldsymbol{u})\in \mR^{d+m} \;|\; \boldsymbol{u} \in \mR^{m}\}$. Since $\cov(\nabla F(\tx; \xi))\succ 0$, we have $\boldsymbol{z} \in U  \Longleftrightarrow \boldsymbol{z}^{\top} \text{cov}(\nabla\mL(\tx,\tlambda;\xi)) \boldsymbol{z} = 0$. Therefore, we define two more subspaces:
\begin{align*}
U_1 & \coloneqq K^{\star} U = \left\{ \begin{pmatrix}
\nabla_{\bx}^2\mL^\star & (G^\star)^{\top}\\
G^\star & \0
\end{pmatrix} \begin{pmatrix}
\0 \\ \boldsymbol{u}
\end{pmatrix} \ | \ \boldsymbol{u} \in \mR^{m} \right\} = \left\{ \begin{pmatrix}
(G^{\star})^{\top} \boldsymbol{u} \\ \boldsymbol{0}
\end{pmatrix} \ | \ \boldsymbol{u} \in \mR^{m} \right\},\\
W_1 & \coloneqq \left\{  \begin{pmatrix}
\boldsymbol{u} \\ \boldsymbol{v}
\end{pmatrix} | \ \boldsymbol{u} \in \text{Kernel}(G^{\star}), \boldsymbol{v} \in \mR^{m}  \right\}.
\end{align*}
We have the direct sum decomposition $\mR^{d+m} = U_1 \oplus W_1$. Furthermore, for any $\boldsymbol{z} \in U_1$, there~exists~a vector $\boldsymbol{u} \in \mR^m$ such that
\begin{equation*}
\boldsymbol{z}^{\top} \Xi^{\star} \boldsymbol{z} = \begin{pmatrix}
\0 & \bu
\end{pmatrix} \  \text{cov}(\nabla\mL(\tx,\tlambda;\xi)) \begin{pmatrix}
\0 \\ \boldsymbol{u}
\end{pmatrix} = 0.
\end{equation*}
Conversely, if $\boldsymbol{z}^{\top} \Xi^{\star} \boldsymbol{z} = 0$ for some vector $\boldsymbol{z} \in \mR^{d+m}$, since $K^{\star}$ is invertible, there must exist a vector $\boldsymbol{v} \in \mR^{d+m}$ such that $\boldsymbol{z} = K^{\star} \boldsymbol{v}$. Thus, $\boldsymbol{z}^{\top} \Xi^{\star} \boldsymbol{z} = \boldsymbol{v}^{\top} \text{cov}(\nabla\mL(\tx,\tlambda;\xi)) \boldsymbol{v} = 0$. This~indicates~that~$\boldsymbol{v} \in U$ and $\bz\in U_1$. Thus, we obtain that $\boldsymbol{z}^{\top} \Xi^{\star} \boldsymbol{z} = 0 \Longleftrightarrow \boldsymbol{z} \in U_1$. We complete the proof.

\subsection{Proof of Lemma \ref{lem:C1}}

We first establish the convergence rate of $\bomega_t$. From \eqref{nequ:3} in the proof of Lemma \ref{lem:4.1}, we know
\begin{equation}\label{equ:3}
\bomega_{t+1} = \bomega_t - \beta_t(I - C_t)K_t^{-1}\nabla\mL_t  + \beta_t\btheta_t+ \rbr{\baralpha_t - \beta_t} \bz_{t,\tau}.
\end{equation}
Given $t_0, \nu>0$, we denote for notational brevity $\mu_t$ as the $\mF_{t-1}$-measurable event $\{\mu_{t_0, \nu}>t\}$. Then, 
\begin{align}\label{nequ:5}
& \mathbb{E}[\|\bomega_{t+1}\|^2 \mathbf{1}_{\mu_{t+1}} | \mathcal{F}_{t-1}] \leq \mE[\|\bomega_{t+1}\|^2 \mathbf{1}_{\mu_t} | \mathcal{F}_{t-1}] \nonumber\\
& \stackrel{\mathclap{\eqref{equ:3}}}{\leq}\;  \|\bomega_t\|^2\b1_{\mu_t} + \beta_t^2\|I - C_t\|^2\|K_t^{-1}\|^2\|\nabla\mL_t\|^2 + \beta_t^2\mE[\|\btheta_t\|^2\mid \mF_{t-1}] + \chi_t^2\mathbb{E}[\|\boldsymbol{z}_{t, \tau}\|^2 | \mathcal{F}_{t-1}] \nonumber\\
&\; \quad - 2 \beta_t \langle (I - C_t)K_t^{-1}\nabla \mathcal{L}_t,  \bomega_t \mathbf{1}_{\mu_t} \rangle  + 2\chi_t \|\bomega_t\|\mathbf{1}_{\mu_t}     \cdot \mathbb{E}[\|\boldsymbol{z}_{t, \tau}\| | \mathcal{F}_{t-1}] \nonumber\\
&\; \quad + 2\beta_t\chi_t\|I-C_t\|\cdot\|K_t^{-1}\|\cdot\|\nabla\mL_t\|\cdot \mathbb{E}[\|\boldsymbol{z}_{t, \tau}\| | \mathcal{F}_{t-1}] + 2\beta_t\chi_t \mathbb{E}[\|\boldsymbol{\theta}_t\| \cdot \|\boldsymbol{z}_{t, \tau}\| | \mathcal{F}_{t-1}] \nonumber\\
& \stackrel{\mathclap{\eqref{equ:1}}}{\leq}\;  \|\bomega_t\|^2\b1_{\mu_t} + 4\Upsilon_{K}\Upsilon_u\beta_t^2 + C_{2+\delta}^{2/(2+\delta)}\beta_t^2 + \Upsilon_z\chi_t^2 - 2 \beta_t \langle (I - C_t)K_t^{-1}\nabla \mathcal{L}_t,  \bomega_t \mathbf{1}_{\mu_t} \rangle + \frac{(1-\rho^\tau)}{2}\beta_t\|\bomega_t\|^2\b1_{\mu_t} \nonumber\\
& \;\quad  + \frac{2\Upsilon_z\chi_t^2}{(1-\rho^\tau)\beta_t} + 4\Upsilon_K\Upsilon_u\sqrt{\Upsilon_z}\beta_t\chi_t + C_{2+\delta}^{2/(2+\delta)}\beta_t\chi_t + \Upsilon_z\beta_t\chi_t \nonumber\\
& = \; \rbr{1+0.5(1-\rho^\tau)\beta_t}\|\bomega_t\|^2\b1_{\mu_t} - 2 \beta_t \langle (I - C_t)K_t^{-1}\nabla \mathcal{L}_t,  \bomega_t \mathbf{1}_{\mu_t} \rangle + (4\Upsilon_{K}\Upsilon_u + C_{2+\delta}^{2/(2+\delta)})\beta_t^2 + \Upsilon_z\chi_t^2 \nonumber\\
& \quad\; + \frac{2\Upsilon_z\chi_t^2}{(1-\rho^\tau)\beta_t}  + \rbr{4\Upsilon_K\Upsilon_u\sqrt{\Upsilon_z}+C_{2+\delta}^{2/(2+\delta)} + \Upsilon_z }\beta_t\chi_t,
\end{align}
where the second inequality is also due to \eqref{equ:sandwich} and the fact that $\btheta_t$ is the martingale difference; and~the third inequality is also due to Lemma \ref{lem:1}(b), Lemma \ref{lem:K_tconvergence}(b), \eqref{eq: bounded_p_moment}, and the fact~\mbox{$\|C_t\|\leq 1$}.~Since $\chi_t = c_\chi/(t+1)^\chi$, $\beta_t = c_\beta/(t+1)^\beta$ with $\chi>\beta+0.5$ and $\beta\in(0.5, 1)$, we know 
\begin{equation*}
\chi_t^2 = o(\beta_t^2),\quad\quad \frac{\chi_t^2}{\beta_t} = o(\beta_t^2),\quad\quad \beta_t\chi_t = o(\beta_t^2).
\end{equation*}
For the second term on the right hand side of \eqref{nequ:5}, we have for all $t_0\leq t < \mu_{t_0, \nu}$,
\begin{align*}
-\left\langle (I - C_t)K_t^{-1}\nabla \mathcal{L}_t, \bomega_t \right\rangle & = -\langle (I - C_t)\bomega_t, \bomega_t\rangle - \langle (I-C_t)(K_t^{-1}\nabla\mL_t - \bomega_t), \bomega_t\rangle\\
& \leq -(1-\rho^\tau)\|\bomega_t\|^2 + 2\|K_t^{-1}\nabla\mL_t-\bomega_t\|\|\bomega_t\| \quad (\text{Lemma \ref{lem:1}(c)})\\
& \leq -(1-\rho^\tau)\|\bomega_t\|^2 + 2\Upsilon_{K}\|\nabla\mL_t - K_t\bomega_t\|\|\bomega_t\| \quad (\text{Lemma \ref{lem:K_tconvergence}(b)})\\
& \leq -(1-\rho^\tau)\|\bomega_t\|^2 + 2\Upsilon_{K}\nu\|\bomega_t\|^2 + 2\Upsilon_{K}\|\nabla\mL_t-K^\star\bomega_t\|\|\bomega_t\| \quad (\text{by \eqref{def:stopping_time}})\\
& \leq -(1-\rho^\tau)\|\bomega_t\|^2 + 2\Upsilon_{K}\nu\|\bomega_t\|^2 + 2\Upsilon_{K}\Upsilon_{L}\|\bomega_t\|^3 \quad (\text{by Assumption \ref{ass:1}})\\
& \leq  -(1-\rho^\tau)\|\bomega_t\|^2 + 2\Upsilon_{K}(1+\Upsilon_{L})\nu\|\bomega_t\|^2 \quad (\text{by \eqref{def:stopping_time}})\\
& \leq -0.5(1-\rho^\tau)\|\bomega_t\|^2 \quad (\text{by \eqref{eq:small_nu}}).
\end{align*}
Combining the above two displays with \eqref{nequ:5}, and applying Lemma \ref{lem:phi_t_recursion_lemma}, we know there exists~$\Upsilon_{\mu}^{(1)}>0$ such that 
\begin{equation}\label{nequ:6}
\mE[\|\bomega_t\|^2 \mathbf{1}_{\{\mu_{t_0, \nu} > t\}}] \leq \Upsilon_{\mu}^{(1)} \beta_t,\quad\quad \forall t\geq 0.
\end{equation}
Now, we establish the convergence rate of $R_t$. By the definition of $R_t$ in \eqref{equ:2}, we have
\begin{align}\label{eq:R_t2_rate}
\mathbb{E}[R_t^2\mathbf{1}_{\{\mu_{t_0, \nu} > t\}}] & \leq 2\mE\left\|\frac{1}{t}\sum_{i = 0}^{t-1} (\bar{H}_i - \nabla^2 f_i) \right\|^2 + \frac{2}{t^2}\mathbb{E}\left[\left( \sum_{i = 0}^{t-1} \|\bomega_i\| \right)^2\mathbf{1}_{\{\mu_{t_0, \nu} > t\}}\right] \nonumber\\
& \leq 2\mE\left\|\frac{1}{t}\sum_{i = 0}^{t-1} (\bar{H}_i - \nabla^2 f_i) \right\|^2 + \frac{2}{t^2}\mathbb{E}\left[\left( \sum_{i = 0}^{t-1} \|\bomega_i\|\b1_{\{\mu_{t_0, \nu} > i\}}  \right)^2\right] \nonumber\\
& \leq 2\mE\left\|\frac{1}{t}\sum_{i = 0}^{t-1} (\bar{H}_i - \nabla^2 f_i) \right\|^2 + \frac{2}{t^2}\left( \sum_{i = 0}^{t-1} \mE[\|\bomega_i\|^2\b1_{\{\mu_{t_0, \nu} > i\}}]^{1/2}  \right)^2,
\end{align}
where the last inequality is from the Cauchy-Schwarz inequality. For the first term on the~right~hand~side,
\begin{align*}
&\mathbb{E}\left\|\frac{1}{t}\sum_{i = 0}^{t-1} (\bar{H}_i - \nabla^2 f_i) \right\|^2 \leq \mathbb{E}\left\|\frac{1}{t}\sum_{i = 0}^{t-1} (\bar{H}_i - \nabla^2 f_i) \right\|^2_{F} = \frac{1}{t^2} \text{Tr}\left( \mathbb{E}\left( \sum_{i = 0}^{t-1}(\bar{H}_i - \nabla^2 f_i) \right)^2 \right)\\
& = \frac{1}{t^2} \text{Tr} \left( \sum_{i = 0}^{t-1} \mathbb{E}\left[(\bar{H}_i - \nabla^2 f_i)^2\right] \right) = \frac{1}{t^2}\sum_{i = 0}^{t-1} \text{Tr}\rbr{\mE\sbr{\mE[\barH_i^2\mid \mF_{i-1}] - (\nabla^2 f_i)^2}} \leq \frac{1}{t^2}\sum_{i = 0}^{t-1}\mE[\|\barH_i\|^2_F] \leq \frac{d\Upsilon_m}{t},
\end{align*}
where the third equality is because $\bar{H}_i - \nabla^2 f_i$ forms a martingale difference sequence with~$\mE \left[\bar{H}_i - \nabla^2 f_i | \mF_{i-1} \right] = 0$, and the last inequality is due to Assumption \ref{ass:2}. For the second term on the right hand side of \eqref{eq:R_t2_rate}, we apply \eqref{nequ:6} and have
\begin{multline*}
\frac{1}{t^2}\left(\sum_{i = 0}^{t-1} \mathbb{E}[\|\bomega_i\|^2\mathbf{1}_{\{\mu_{t_0, \nu} > i\}}]^{1/2}\right)^2 \leq \frac{\Upsilon_{\mu}^{(1)}}{t^2}\rbr{\sum_{i=0}^{t-1}\sqrt{\beta_i}}^2 = \frac{\Upsilon_{\mu}^{(1)}c_\beta}{t^2}\rbr{\sum_{i=0}^{t-1}\frac{1}{(i+1)^{0.5\beta}}}^2\\
 \leq \frac{\Upsilon_{\mu}^{(1)}c_\beta}{t^2}\rbr{\int_{0}^{t}\frac{1}{(i+1)^{0.5\beta}}\;di }^2 \leq \frac{\Upsilon_{\mu}^{(1)}c_\beta}{(t+1)^\beta} = \Upsilon_{\mu}^{(1)}\beta_t.
\end{multline*}
Combining the above two displays and plugging into \eqref{eq:R_t2_rate}, we complete the proof.

\section{Additional Experimental Results}\label{appen:D}

In this section, we present the complete experimental results in Tables \ref{tab:3}-\ref{tab:8} for constrained~\mbox{linear}~and~logistic regression problems.~We vary the dimension $d \in \{5, 20, 40\}$ and study three different design covariances $\Sigma_a$: Identity, Toeplitz, and Equi-correlation. Following prior work in \cite{Chen2020Statistical, Na2025Statistical}, we vary the parameter $r \in \{0.4, 0.5, 0.6\}$ for Toeplitz $\Sigma_a$ and $r \in \{0.1, 0.2, 0.3\}$ for Equi-correlation $\Sigma_a$.
We also vary the sketching steps of AI-SSQP $\tau \in \{20, 40, \infty\}$, where $\tau = \infty$ corresponds to the exact SSQP scheme. For each combination of $d$, $\Sigma_a$, and $\tau$, we compare five~different online inference procedures across four evaluation metrics: mean absolute error,~averaged~coverage rate, averaged confidence interval length, and flops per iteration.

We observe from Tables \ref{tab:3}--\ref{tab:8} that for both linear and logistic regression, the exact SSQP method $(\tau = \infty)$ achieves a smaller mean absolute error than the inexact methods with $\tau = 20$ or $40$, between which $\tau = 40$ yields a smaller error than $\tau = 20$. In other words, the more accurate the Newton direction employed, the closer the method converges to the optimal solution.~For~the~same~setup~of~the method, the inference procedures based on the averaged iterate (\texttt{AvePlugIn}, \texttt{AveBM}, \texttt{AveRS}) consistently yield the same mean absolute error, which is generally~smaller, by half to one order~of~magnitude, than those based on the last iterate (\texttt{LastPlugIn}, \texttt{LastBF}). 
We notice that in some linear~regression scenarios, the inexact methods result in abnormally high errors (e.g., $d=5$, $\tau \in \{20, 40\}$, and $\Sigma_a$ is identity or Equi-correlation with $r=0.1$). This phenomenon is attributed to excessive~randomness in approximating the Newton direction. A closer inspection reveals that in approximately~5 out of 200 runs, the inexact methods fail to converge; for the runs that do converge, the errors~remain within a reasonable range.
Moreover, when $d=40$ and $\Sigma_a$ is Equi-correlation with $r=0.3$,~the~error for $\tau = 20$ is significantly higher than that for $\tau = 40$, likely due to insufficient accuracy in~the~Newton direction approximation. Overall, as suggested by the theory (cf. Remark \ref{rem:1}), setting~$\tau \gtrsim d$~appears to be a reasonable strategy for inexact methods to balance the trade-off between randomness and precision in the Newton direction approximation.

For the averaged coverage rate, we observe from Tables \ref{tab:3}--\ref{tab:8} that for a fixed dimension $d$,~\mbox{design}~covariance type $\Sigma_a$, and sketching step $\tau$, all five inference methods exhibit certain robustness to~variations in the parameter $r$. Among them, \texttt{AveRS} consistently achieves the nominal 95\% coverage~rate in most scenarios. \texttt{LastBF} performs comparably in terms of coverage but suffers from larger mean absolute errors and wider confidence intervals.
The inference methods based on the plug-in covariance estimator, \texttt{LastPlugIn} and \texttt{AvePlugIn}, perform reasonably well when $\tau = \infty$ across different scenarios, with \texttt{AvePlugIn} yielding shorter confidence intervals. However, both methods tend to fail for maintaining the nominal 95\% coverage rate when $\tau = 20, 40$ for both linear and logistic regression. There are exceptions: for instance, in linear regression with $d = 5$, \texttt{AvePlugIn} achieves coverage close to 95\% across all design covariances, while \texttt{LastPlugIn} exhibits extreme overcoverage, approaching 100\%. Conversely, in logistic regression with $d = 40$, \texttt{LastPlugIn} achieves near-nominal~coverage 
across all designs, while \texttt{AvePlugIn} exhibits extreme~undercoverage,~falling~below 80\%. Overall,~these failure modes stem from the inconsistency of the plug-in covariance estimator,~rendering~these~methods \textit{asymptotically invalid}.
Lastly, we note that \texttt{AveBM} performs the worst among the five methods, with coverage rates far below acceptable levels. The batch-means estimator was originally designed for first-order methods with strongly convex objectives and/or sub-Gaussian gradient noise \citep{Zhu2021Online, Jiang2025Online}, and its theoretical guarantees do not extend to our setting.

In terms of the length of the confidence interval, we note that for each setup and \mbox{inference}~method, the interval for logistic regression is consistently shorter than that for linear~\mbox{regression}.~Furthermore, \texttt{AveBM} and \texttt{AvePlugIn} yield the shortest intervals, reflecting the optimal efficiency of the averaged iterate due to its asymptotic normality guarantee, with \texttt{AveRS} performing nearly as well. In contrast, \texttt{LastBF} and \texttt{LastPlugIn} yield significantly wider intervals -- roughly an order of magnitude larger -- because the relative efficiency between the last and averaged iterates diverges under the stepsize schedule $\beta_t = 1/t^{0.501}$. This phenomenon demonstrates that our random scaling inference method retains much of the efficiency exhibited by covariance-estimation-based methods.

Finally, in terms of flops per iteration, it is clear that the inexact methods are matrix-free~and~incur lower computational costs compared to exact methods. The computational savings from the sketching solver become increasingly significant in higher dimensions. For example, in both linear and logistic regression with $d = 40$, \texttt{AveRS} with $\tau = 20$ achieves a 95\% coverage rate across different design covariances with fewer than $1700$ flops per iteration, while the exact method requires over $10^5$ flops per iteration to achieve a similar coverage rate. This online, matrix-free fashion makes our inference method especially attractive for large-scale constrained problems.

\include{tables_appendix}

%% file: tables_appendix.tex
\begin{table}[thbp!] 
\centering
\setlength{\tabcolsep}{6pt}         
\renewcommand{\arraystretch}{0.98}   
\resizebox{1.0\linewidth}{0.5\textheight}{
\begin{tabular}{|c|cc|cc|ccc|ccc|c|}
\hline
&&&&&&&&&&&\\[-2.4ex] 
\multirow{2}{*}{$d$}&\multicolumn{2}{c|}{\multirow{2}{*}{Design Cov}}&\multicolumn{2}{c|}{\multirow{2}{*}{Method}}&\multicolumn{3}{c|}{Constrained linear regression}& \multicolumn{3}{c|}{Constrained logistic regression} & \multirow{2}{*}{Flops/iter} \\
\cline{6-11}
&&&&&&&&&&& \\[-2.4ex]
& & & & & MAE $(10^{-2})$ &  \multicolumn{1}{|c|}{Ave Cov $(\%)$} & Ave Len $(10^{-2})$ & MAE $(10^{-2})$ &  \multicolumn{1}{|c|}{Ave Cov $(\%)$} & Ave Len $(10^{-2})$ & \\
\hline 
\multirow{60}{*}{5} 
& \parbox[t]{2mm}{\multirow{15}{*}{\;\;\;Identity}}& & \multirow{5}{*}{$\tau=\infty$} & 
\multicolumn{1}{|c|}{\texttt{LastPlugIn}} & 9.14 & \multicolumn{1}{|c|}{94.00} & 7.81  & 3.21 & \multicolumn{1}{|c|}{93.50} & 1.37 & \multirow{2}{*}{487.60}\\
\cline{5-11}
&&&& \multicolumn{1}{|c|}{\texttt{AvePlugIn}} & 1.63 & \multicolumn{1}{|c|}{95.00} & 1.40 &  0.58 & \multicolumn{1}{|c|}{90.50} & 0.25 & \\
\cline{5-12}
&&&& \multicolumn{1}{|c|}{\texttt{LastBF}} & 9.14 & \multicolumn{1}{|c|}{94.00} & 7.70 & 3.21 & \multicolumn{1}{|c|}{95.00} & 1.46 & \multirow{3}{*}{365.40} \\
\cline{5-11}
&&&& \multicolumn{1}{|c|}{\texttt{AveBM}} & \multirow{2}{*}{1.63} & \multicolumn{1}{|c|}{67.00} &0.74&  \multirow{2}{*}{0.58} & \multicolumn{1}{|c|}{60.50} & 0.15 & \\
\hline
\end{tabular}}
\vspace{-0.2cm}
\caption{\textit{Comparison results for different inference methods on constrained linear and logistic~regression problems under $d = 5$.}}\label{tab:3}
\vspace{-0.2cm}
\end{table}

\begin{table}[thbp!] 
\centering
\setlength{\tabcolsep}{6pt}         
\renewcommand{\arraystretch}{0.98}   
\resizebox{1.0\linewidth}{0.4\textheight}{\begin{tabular}{|c|cc|cc|ccc|ccc|c|}
\hline
&&&&&&&&&&&\\[-2.4ex]
\multirow{2}{*}{$d$}&\multicolumn{2}{c|}{\multirow{2}{*}{Design Cov}}&\multicolumn{2}{c|}{\multirow{2}{*}{Method}}&\multicolumn{3}{c|}{Constrained linear regression}& \multicolumn{3}{c|}{Constrained logistic regression} & \multirow{2}{*}{Flops/iter}\\
\cline{6-11}
&&&&&&&&&&&\\[-2.4ex]
& & & & & MAE $(10^{-2})$ &  \multicolumn{1}{|c|}{Ave Cov $(\%)$} & Ave Len $(10^{-2})$ & MAE $(10^{-2})$ &  \multicolumn{1}{|c|}{Ave Cov $(\%)$} & Ave Len $(10^{-2})$ & \\
\hline 
\multirow{45}{*}{5}
&\parbox[t]{2mm}{\multirow{15}{*}{\;\;\shortstack{Equi-Corr\\$r=0.1$}}}& & \multirow{5}{*}{$\tau=\infty$} & 
\multicolumn{1}{|c|}{\texttt{LastPlugIn}} & 8.15 & \multicolumn{1}{|c|}{94.50} & 7.38  & 3.12 & \multicolumn{1}{|c|}{94.00} & 1.44 & \multirow{2}{*}{487.60} \\
\cline{5-11}
&&&& \multicolumn{1}{|c|}{\texttt{AvePlugIn}} & 1.41 & \multicolumn{1}{|c|}{97.00} & 1.32 &  0.59 & \multicolumn{1}{|c|}{92.00} & 0.26 & \\
\cline{5-12}
&&&& \multicolumn{1}{|c|}{\texttt{LastBF}} & 8.15 & \multicolumn{1}{|c|}{94.50} & 7.29 & 3.12 & \multicolumn{1}{|c|}{96.50} & 1.55 & \multirow{3}{*}{365.40} \\
\cline{5-11}
&&&& \multicolumn{1}{|c|}{\texttt{AveBM}} & \multirow{2}{*}{1.41} & \multicolumn{1}{|c|}{68.50} &0.71& \multirow{2}{*}{0.59} & \multicolumn{1}{|c|}{74.00} & 0.16 & \\
\end{tabular}}
\vspace{-0.2cm}
\caption{\textit{(Cont') Comparison results for different inference methods on constrained linear and logistic regression problems under $d = 5$.}}\label{tab:4}
\vspace{-0.2cm}
\end{table}

\begin{table}[thbp!] 
\centering
\setlength{\tabcolsep}{6pt}         
\renewcommand{\arraystretch}{0.98}   
\resizebox{1.0\linewidth}{0.5\textheight}{\begin{tabular}{|c|cc|cc|ccc|ccc|c|}
\hline
&&&&&&&&&&&\\[-2.4ex] 
\multirow{2}{*}{$d$}&\multicolumn{2}{c|}{\multirow{2}{*}{Design Cov}}&\multicolumn{2}{c|}{\multirow{2}{*}{Method}}&\multicolumn{3}{c|}{Constrained linear regression}& \multicolumn{3}{c|}{Constrained logistic regression} & \multirow{2}{*}{Flops/iter} \\
\cline{6-11}
&&&&&&&&&&&\\[-2.4ex] 
& & & & & MAE $(10^{-2})$ &  \multicolumn{1}{|c|}{Ave Cov $(\%)$} & Ave Len $(10^{-2})$ & MAE $(10^{-2})$ &  \multicolumn{1}{|c|}{Ave Cov $(\%)$} & Ave Len $(10^{-2})$ & \\
\hline
\multirow{60}{*}{20}
&\parbox[t]{2mm}{\multirow{15}{*}{\;\;\;Identity}}& & \multirow{5}{*}{$\tau=\infty$} & 
\multicolumn{1}{|c|}{\texttt{LastPlugIn}} & 9.19 & \multicolumn{1}{|c|}{94.50} & 1.06  & 4.39 & \multicolumn{1}{|c|}{91.00} & 0.43 & \multirow{2}{*}{17632.74}\\
\cline{5-11}
&&&& \multicolumn{1}{|c|}{\texttt{AvePlugIn}} & 1.63 & \multicolumn{1}{|c|}{94.50} & 0.19 &  0.84 & \multicolumn{1}{|c|}{88.00} & 0.08 & \\
\cline{5-12}
&&&& \multicolumn{1}{|c|}{\texttt{LastBF}} & 9.19 & \multicolumn{1}{|c|}{94.50} & 1.05 & 4.39 & \multicolumn{1}{|c|}{93.50} & 0.46 & \multirow{3}{*}{14063.88} \\
\cline{5-11}
&&&& \multicolumn{1}{|c|}{\texttt{AveBM}} &  \multirow{2}{*}{1.63} & \multicolumn{1}{|c|}{60.00} &0.09&   \multirow{2}{*}{0.84} & \multicolumn{1}{|c|}{59.00} & 0.05 & \\
\end{tabular}}
\vspace{-0.2cm}
\caption{\textit{Comparison results for different inference methods on constrained linear and logistic~regression problems under $d = 20$.}}\label{tab:5}
\vspace{-0.2cm}
\end{table}

\begin{table}[thbp!] 
\centering
\setlength{\tabcolsep}{6pt}         
\renewcommand{\arraystretch}{0.98}   
\resizebox{1.0\linewidth}{0.4\textheight}{\begin{tabular}{|c|cc|cc|ccc|ccc|c|}
\hline
&&&&&&&&&&&\\[-2.4ex] 
\multirow{2}{*}{$d$}&\multicolumn{2}{c|}{\multirow{2}{*}{Design Cov}}&\multicolumn{2}{c|}{\multirow{2}{*}{Method}}&\multicolumn{3}{c|}{Constrained linear regression}& \multicolumn{3}{c|}{Constrained logistic regression} & \multirow{2}{*}{Flops/iter}\\
\cline{6-11}
&&&&&&&&&&&\\[-2.4ex] 
& & & & & MAE $(10^{-2})$ &  \multicolumn{1}{|c|}{Ave Cov $(\%)$} & Ave Len $(10^{-2})$ & MAE $(10^{-2})$ &  \multicolumn{1}{|c|}{Ave Cov $(\%)$} & Ave Len $(10^{-2})$ & \\
\hline 
\multirow{45}{*}{20}
&\parbox[t]{2mm}{\multirow{15}{*}{\;\;\shortstack{Equi-Corr\\$r=0.1$}}}& & \multirow{5}{*}{$\tau=\infty$} & 
\multicolumn{1}{|c|}{\texttt{LastPlugIn}} & 8.84 & \multicolumn{1}{|c|}{94.00} & 1.38  & 4.17 & \multicolumn{1}{|c|}{94.50} & 0.44 & \multirow{2}{*}{17632.74} \\
\cline{5-11}
&&&& \multicolumn{1}{|c|}{\texttt{AvePlugIn}} & 1.61 & \multicolumn{1}{|c|}{94.00} & 0.25 &  0.80 & \multicolumn{1}{|c|}{94.00} & 0.08 & \\
\cline{5-12}
&&&& \multicolumn{1}{|c|}{\texttt{LastBF}} & 8.84 & \multicolumn{1}{|c|}{93.50} & 1.37 & 4.17 & \multicolumn{1}{|c|}{94.50} & 0.46 & \multirow{3}{*}{14063.88} \\
\cline{5-11}
&&&& \multicolumn{1}{|c|}{\texttt{AveBM}} & \multirow{2}{*}{1.61} & \multicolumn{1}{|c|}{63.00} &0.13& \multirow{2}{*}{0.80} & \multicolumn{1}{|c|}{64.00} & 0.04 & \\
\end{tabular}}
\vspace{-0.2cm}
\caption{\textit{(Cont') Comparison results for different inference methods on constrained linear and logistic regression problems under $d = 20$.}}\label{tab:6}
\vspace{-0.2cm}
\end{table}

\begin{table}[thbp!] 
\centering
\setlength{\tabcolsep}{6pt}         
\renewcommand{\arraystretch}{0.98}   
\resizebox{1.0\linewidth}{0.5\textheight}{\begin{tabular}{|c|cc|cc|ccc|ccc|c|}
\hline
&&&&&&&&&&&\\[-2.4ex] 
\multirow{2}{*}{$d$}&\multicolumn{2}{c|}{\multirow{2}{*}{Design Cov}}&\multicolumn{2}{c|}{\multirow{2}{*}{Method}}&\multicolumn{3}{c|}{Constrained linear regression}& \multicolumn{3}{c|}{Constrained logistic regression} & \multirow{2}{*}{Flops/iter} \\
\cline{6-11}
&&&&&&&&&&&\\[-2.4ex]
& & & & & MAE $(10^{-2})$ &  \multicolumn{1}{|c|}{Ave Cov $(\%)$} & Ave Len $(10^{-2})$ & MAE $(10^{-2})$ &  \multicolumn{1}{|c|}{Ave Cov $(\%)$} & Ave Len $(10^{-2})$ & \\
\hline 
\multirow{60}{*}{40}
&\parbox[t]{2mm}{\multirow{15}{*}{\;\;\;Identity}}& & \multirow{5}{*}{$\tau=\infty$} & 
\multicolumn{1}{|c|}{\texttt{LastPlugIn}} & 10.68 & \multicolumn{1}{|c|}{97.00} & 0.62  & 5.63 & \multicolumn{1}{|c|}{95.00} & 0.28 & \multirow{2}{*}{106068.43}\\
\cline{5-11}
&&&& \multicolumn{1}{|c|}{\texttt{AvePlugIn}} & 1.93 & \multicolumn{1}{|c|}{94.00} & 0.11 &  1.07 & \multicolumn{1}{|c|}{90.00} & 0.05 & \\
\cline{5-12}
&&&& \multicolumn{1}{|c|}{\texttt{LastBF}} & 10.68 & \multicolumn{1}{|c|}{95.50} & 0.61 & 5.63 & \multicolumn{1}{|c|}{95.50} & 0.29 & \multirow{3}{*}{85975.21} \\
\cline{5-11}
&&&& \multicolumn{1}{|c|}{\texttt{AveBM}} & \multirow{2}{*}{1.93} & \multicolumn{1}{|c|}{60.00} &0.06&  \multirow{2}{*}{1.07} & \multicolumn{1}{|c|}{59.00} & 0.03 & \\
\end{tabular}}
\vspace{-0.2cm}
\caption{\textit{Comparison results for different inference methods on constrained linear and logistic~regression problems under $d = 40$.}}\label{tab:7}
\vspace{-0.2cm}
\end{table}

\begin{table}[thbp!] 
\centering
\setlength{\tabcolsep}{6pt}         
\renewcommand{\arraystretch}{0.98}   
\resizebox{1.0\linewidth}{0.4\textheight}{\begin{tabular}{|c|cc|cc|ccc|ccc|c|}
\hline
&&&&&&&&&&&\\[-2.4ex] 
\multirow{2}{*}{$d$}&\multicolumn{2}{c|}{\multirow{2}{*}{Design Cov}}&\multicolumn{2}{c|}{\multirow{2}{*}{Method}}&\multicolumn{3}{c|}{Constrained linear regression}& \multicolumn{3}{c|}{Constrained logistic regression} & \multirow{2}{*}{Flops/iter}\\
\cline{6-11}
&&&&&&&&&&&\\[-2.4ex] 
& & & & & MAE $(10^{-2})$ &  \multicolumn{1}{|c|}{Ave Cov $(\%)$} & Ave Len $(10^{-2})$ & MAE $(10^{-2})$ &  \multicolumn{1}{|c|}{Ave Cov $(\%)$} & Ave Len $(10^{-2})$ & \\
\hline 
\multirow{45}{*}{40}
&\parbox[t]{2mm}{\multirow{15}{*}{\;\;\shortstack{Equi-Corr\\$r=0.1$}}}& & \multirow{5}{*}{$\tau=\infty$} & 
\multicolumn{1}{|c|}{\texttt{LastPlugIn}} & 10.83 & \multicolumn{1}{|c|}{94.50} & 0.67  & 5.29 & \multicolumn{1}{|c|}{94.50} & 0.28 & \multirow{2}{*}{106068.43} \\
\cline{5-11}
&&&& \multicolumn{1}{|c|}{\texttt{AvePlugIn}} & 1.95 & \multicolumn{1}{|c|}{95.50} & 0.12 &  0.98 & \multicolumn{1}{|c|}{91.00} & 0.05 & \\
\cline{5-12}
&&&& \multicolumn{1}{|c|}{\texttt{LastBF}} & 10.83 & \multicolumn{1}{|c|}{94.50} & 0.65 & 5.29 & \multicolumn{1}{|c|}{97.00} & 0.29 & \multirow{3}{*}{85975.21} \\
\cline{5-11}
&&&& \multicolumn{1}{|c|}{\texttt{AveBM}} & \multirow{2}{*}{1.95} & \multicolumn{1}{|c|}{63.50} &0.06& \multirow{2}{*}{0.98} & \multicolumn{1}{|c|}{68.50} & 0.03 & \\
\end{tabular}}
\vspace{-0.2cm}
\caption{\textit{(Cont') Comparison results for different inference methods on constrained linear and logistic regression problems under $d = 40$.}}\label{tab:8}
\vspace{-0.2cm}
\end{table}